\documentclass{article}

\usepackage[english]{babel}
\usepackage[utf8]{inputenc}
\usepackage{enumitem}
\usepackage{microtype}
\usepackage{graphicx}
\usepackage{booktabs} 
\usepackage{amsmath,amsfonts,amssymb}
\usepackage{mathtools}
\usepackage[round]{natbib}
\usepackage{url}           
\usepackage{nicefrac}       
\usepackage[dvipsnames]{xcolor}
\usepackage{multirow}
\usepackage{wrapfig}
\usepackage{setspace}
\usepackage{pifont}

\usepackage{algorithm}
\usepackage{algpseudocode}
\usepackage{titletoc}
\usepackage{capt-of}

\usepackage[backref=page,colorlinks=true,linkcolor=RedOrange, citecolor=NavyBlue]{hyperref}

\usepackage{xurl}
\hypersetup{
pdftitle={Communication-Efficient Private Federated Reinforcement Learning with Post-Quantum Authentication for Vehicular Networks},
pdfauthor={Hojat Allah Salehi, Devrim Aras, Md Jueal Mia, M. Hadi Amini, Imtiaz Karim, Mashrur Chowdhury, Bhavani Thuraisingham}
}

\newtheorem{theorem}{Theorem}
\newtheorem{corollary}{Corollary}
\newtheorem{proposition}{Proposition}

\newtheorem{definition}{Definition}

\newcommand{\qed}{\hfill$\blacksquare$}

\usepackage{array}
\usepackage{textcomp}
\usepackage{verbatim}
\usepackage[caption=false,font=normalsize,labelfont=sf,textfont=sf]{subfig}

\newenvironment{proof}[1][Proof]
{\par\smallskip\noindent\textit{#1.}\enspace}
{\qed\par\smallskip}

\makeatletter
\long\def\@makecaption#1#2{
  \vskip 0.8ex
  \setbox\@tempboxa\hbox{\small {\bf #1:} #2}
  \parindent 1.5em  
  \dimen0=\hsize
  \advance\dimen0 by -3em
  \ifdim \wd\@tempboxa >\dimen0
  \hbox to \hsize{
    \parindent 0em
    \hfil 
    \parbox{\dimen0}{\def\baselinestretch{0.96}\small
      {\bf #1.} #2
    } 
    \hfil}
  \else \hbox to \hsize{\hfil \box\@tempboxa \hfil}
  \fi
}
\makeatother

\title{\bfseries\LARGE
Communication-Efficient Private Federated Reinforcement \\ Learning with Post-Quantum Authentication for\\ 
Vehicular Networks
}

\author{
\normalsize
Hojat Allah Salehi$^{1,2}$, Devrim Aras$^{1,2}$,
Md Jueal Mia$^{1,2}$, M.~Hadi Amini$^{1,2}$\\[0.25em]
\normalsize
Imtiaz Karim$^{3}$, Mashrur Chowdhury$^{4}$,
Bhavani Thuraisingham$^{3}$\\[0.65em]
\parbox{0.95\textwidth}{\centering\small
$^{1}$Knight Foundation School of Computing and Information Sciences,
Florida International University, \\Miami, FL 33199, USA\\[0.15em]
$^{2}$Security, Optimization, and Learning for Interdependent Networks Lab (solid lab),
Miami, FL 33199, USA\\[0.15em]
$^{3}$Department of Computer Science, The University of Texas at Dallas,
Richardson, TX 75080, USA\\[0.15em]
$^{4}$Glenn Department of Civil Engineering, Clemson University, Clemson, SC 29634, USA\\[0.4em]
\texttt{\footnotesize hsalehi@fiu.edu, amini@cs.fiu.edu}
}
}

\date{}
\begin{document}
\maketitle

{\renewcommand{\thefootnote}{}\begin{NoHyper}\footnotetext{This work has been submitted to the IEEE for possible publication. Copyright may be transferred without notice, after which this version may no longer be accessible.}\end{NoHyper}}

\begin{abstract}
Federated reinforcement learning supports cooperative vehicular routing without centralizing trajectory data, but exchanging model updates may reveal private information and incur substantial communication overhead over intermittent vehicle-to-infrastructure links, particularly when post-quantum authentication is required. We study correlated private quantization, which uses complementary shared randomness between paired vehicles to reduce aggregation mean squared error while preserving the marginal parameter-level local differential privacy guarantee and one-bit communication cost of independent quantization. A tunable extension mixes shared and independent randomness to bound privacy leakage under paired-output observations and shared-randomness side information. For this extension, R\'enyi differential privacy accounting establishes vector-level local differential privacy for one model upload even when pairings are revealed. Under matched vector-level privacy and identical local and global parameters, the extension yields 16.5\% lower analytical aggregation mean squared error than independent quantization. We evaluate cooperative routing by coupling microscopic traffic simulation with packet-level simulation of post-quantum-authenticated vehicle-to-infrastructure communication. Compared with full-precision updates, one-bit uploads reduce authenticated uplink size by 92.4\% and increase the mean number of verified updates aggregated per round from 19.6 to 26.5. Under mobility-constrained communication at a parameter-level privacy budget of 1, correlated private quantization reduces mean travel time by 14.3\% relative to independent quantization at the same budget and communication cost. At a parameter-level privacy budget of 0.3, a vehicle-origin inference attack trained on independent simulations reaches 30\% accuracy with correlated private quantization and 22\% with independent quantization after 100 rounds, versus 91\% without privacy and 10\% for random guessing.
\par\smallskip
\noindent\textbf{Code:}
\url{https://github.com/solidlabnetwork/cpq-vehicular-frl}
\end{abstract}

\section{Introduction}
Connected and autonomous vehicles increasingly rely on sensing and vehicle-to-everything (V2X) communication to support cooperative transportation services. Mobility data are distributed across vehicles and roadside infrastructure, making federated learning (FL) a suitable approach for collaborative model training without centralizing raw sensor and trajectory data. However, repeated exchanges of model updates over bandwidth-constrained vehicular links impose substantial communication overhead.

Federated reinforcement learning (FRL) extends FL to sequential decision-making, allowing distributed agents to learn policies through local interactions and periodically aggregate model updates~\citep{fan2021fault, gao2024federated, jin2022federated, yuan2023federated}. In vehicular networks, FRL has been investigated for traffic signal control, emergency vehicle routing, travel-time minimization, and communication resource management~\citep{bao2023scalable, ye2021fedlight, su2023emvlight, mamond2025fldqn, zhao2024federated, wang2024optimizing, zhang2025federated}. These applications motivate examining how privacy, authentication, and communication constraints affect collaborative policy learning in vehicular networks.

Although keeping raw data local limits direct data exposure, federated learning does not by itself prevent information leakage through transmitted model updates. Prior work has demonstrated that shared gradients and model parameters can reveal information about local training data through attacks such as membership inference and model inversion~\citep{shokri2017membership, nasr2019comprehensive, fredrikson2015model}. Differential privacy provides a principled approach for limiting such leakage by randomizing the information released during training~\citep{dwork2014algorithmic, abadi2016deep, sun2020ldp}. In vehicular FRL, however, privacy mechanisms must be considered together with the communication cost of repeatedly exchanging model updates over bandwidth-constrained wireless links. This motivates mechanisms that provide formal privacy guarantees while reducing the size of transmitted model updates.

Model updates must also be authenticated before aggregation so that the server can verify their origin and integrity. Conventional public-key signature schemes based on integer factorization or discrete logarithms are vulnerable to sufficiently capable quantum computers through Shor's algorithm~\citep{shor1999polynomial}, motivating the transition toward post-quantum authentication. NIST has standardized ML-DSA and SLH-DSA~\citep{nist_fips204,nist_fips205} and selected Falcon for standardization as an additional post-quantum signature scheme~\citep{alagic2022status}. These schemes provide authentication and integrity protection for exchanged federated model updates, but differ in signature size and computational cost, which can affect packet fragmentation, retransmissions, and delivery latency over V2X links. Post-quantum authentication is therefore both a security and a communication-design consideration in vehicular FRL.

This work investigates privacy-preserving and communication-efficient vehicular FRL with post-quantum-authenticated model exchange. We investigate correlated private quantization (CPQ) for cooperative routing and compare it with independent stochastic binary quantization~\citep{sun2020ldp} and full-precision Gaussian and Laplace mechanisms. We also introduce CPQ-$\lambda$, which combines shared and independent local randomness to establish privacy guarantees under stronger server observations, including a vector-level guarantee for one model upload. We first evaluate routing performance across privacy budgets under ideal model exchange. We then couple Simulation of Urban Mobility (SUMO) trajectories with packet-level ns-3 simulations of IEEE~802.11p vehicle-to-infrastructure (V2I) communication~\citep{riley2010ns3}. The primary mobility-coupled evaluation uses ML-DSA-44 authentication, with successful model delivery and verification determining participation in federated aggregation. Additional experiments compare ML-DSA-44, Falcon-512, and SLH-DSA-SHA2-128f with full-precision and CPQ updates, supported by cryptographic microbenchmarks. We also evaluate whether an honest-but-curious roadside unit (RSU) can infer vehicles' starting road segments from received model updates across multiple rounds. Figure~\ref{fig:system_architecture} illustrates the integrated framework.

\begin{figure}[t]
\centering
\includegraphics[width=0.98\textwidth]{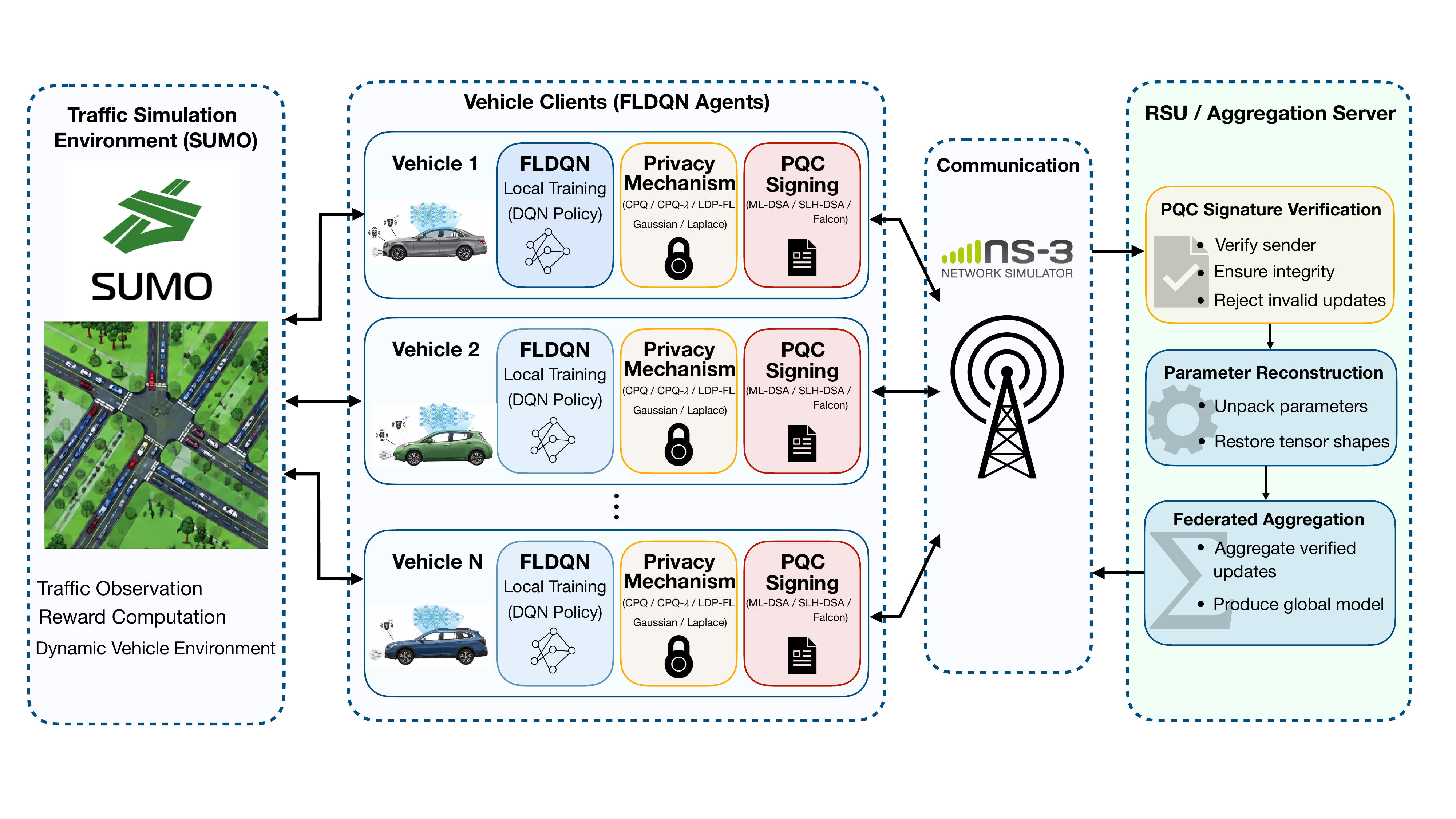}
\caption{Architecture of the proposed privacy-preserving and post-quantum-authenticated vehicular FRL framework. SUMO trajectories drive ns-3 V2I communication, with signed full-precision global-model downlinks and privacy-protected, signed model-update uplinks.}
\label{fig:system_architecture}
\end{figure}

\noindent Our main contributions are:

\begin{itemize}

    \item \textbf{Mobility-Coupled Evaluation of Authenticated Vehicular FRL.} We develop an integrated SUMO and ns-3 evaluation framework that couples cooperative vehicular FRL with post-quantum-authenticated IEEE~802.11p V2I communication. By incorporating model delivery and verification outcomes into federated aggregation, the framework enables joint evaluation of communication reliability, effective participation, and routing performance under intermittent connectivity.

    \item \textbf{Communication-Efficient Private Quantization and Privacy Guarantees.} We characterize CPQ for vehicular FRL, demonstrating reduced aggregation mean squared error (MSE) while preserving one-bit communication and the same $\epsilon_p$-PLDP guarantee as independent stochastic binary quantization. We further introduce CPQ-$\lambda$ and establish privacy guarantees under stronger observation models, including $(\epsilon_v,\delta_v)$ vector-level local differential privacy for one complete model upload. Under matched vector-level privacy, CPQ-$\lambda$ also demonstrates a theoretical aggregation-MSE advantage over independent quantization.

    \item \textbf{Post-Quantum Authentication and V2I Communication Trade-Offs.} We quantify how model-update representation and post-quantum signature selection affect authenticated V2I communication through cryptographic benchmarks and mobility-coupled IEEE~802.11p simulations. The results show that one-bit uploads reduce message size and transfer latency versus full-precision updates, while signature overhead influences delivery reliability when model payloads are highly compressed.
    
    \item \textbf{Vehicle-Origin Inference Attack Evaluation.} We evaluate an inference attack in which an adversary attempts to identify vehicles' starting road segments from observed model updates across federated rounds. Under two attack settings with mobility-constrained V2I communication, CPQ substantially reduces attack accuracy compared with unprotected model exchange.

\end{itemize}

The paper is organized as follows. Section~\ref{sec:relatedwork} reviews related work. Section~\ref{sec:dp} presents the private quantization mechanisms and their privacy and aggregation-error analyses. Section~\ref{sec:pqc} describes the post-quantum signature schemes, while Section~\ref{sec:frl} introduces the vehicular FRL formulation. Section~\ref{sec:system} details system design and authenticated model exchange. Section~\ref{sec:experiments} presents the experimental results, and Section~\ref{sec:conclusion} concludes the paper.

\section{Related Work}
\label{sec:relatedwork}
\noindent\textbf{Communication-efficient and privacy-preserving FL.}
Prior research has investigated communication-efficient and privacy-aware mechanisms for distributed estimation and learning, including locally private estimation and telemetry collection~\citep{duchi2018minimax,ding2017collecting}, binomial randomization with secure aggregation~\citep{chen2022poisson}, randomized quantization with configurable output levels~\citep{youn2023randomized}, and minimum-variance unbiased mechanisms supporting different communication budgets~\citep{chaudhuri2022privacy}. Related sign-based methods address private and Byzantine-resilient gradient compression~\citep{xiang2023distributed} and communication-efficient federated averaging~\citep{tang2024zsignfedavg}. LDP-FL~\citep{sun2020ldp} provides independent stochastic binary quantization with one bit per parameter and an $\epsilon_p$-PLDP guarantee. CorBin-FL~\citep{salehi2024corbin} uses complementary shared randomness across clients to reduce aggregation error while preserving the same marginal PLDP guarantee.

\noindent\textbf{Privacy, security, and FRL in vehicular networks.}
Vehicular FL research has investigated privacy-aware knowledge trading using personalized differential privacy and contract theory~\citep{zhang2024differential}, personalized learning for autonomous driving using correlated differential privacy~\citep{tian2024personalized}, and authenticated secure aggregation~\citep{byun2024secure}.  For sequential decision-making, federated deep reinforcement learning has been studied for task offloading and resource allocation in vehicular edge computing~\citep{zhao2024federated}, while federated graph-based multi-agent reinforcement learning has been applied to age-of-information optimization~\citep{wang2024optimizing}. More directly related to cooperative routing, FLDQN~\citep{mamond2025fldqn} combines federated deep Q-learning with SUMO-based traffic simulation to minimize vehicle travel time and provides the routing and learning foundation adopted in this study.

\noindent\textbf{Vehicular traffic and communication simulation.}
Prior vehicular studies have coupled microscopic traffic and wireless network simulations to capture mobility-dependent communication. Veins integrates SUMO and OMNeT++ for bidirectionally coupled traffic and network simulation~\citep{sommer2011bidirectionally}, while iTETRIS combines SUMO and ns-3 to evaluate cooperative intelligent transportation applications~\citep{rondinone2013itetris}.

\noindent\textbf{Research gap and positioning.}
Prior studies have established important foundations for privacy-preserving quantization, secure vehicular FL, cooperative FRL routing, and mobility-aware traffic--network simulation. Building on these foundations, two complementary research directions motivate this work. The first concerns improving aggregation accuracy through correlated one-bit quantization in vehicular FRL while maintaining formal parameter-level privacy and low communication cost. An additional theoretical question is how such mechanisms can provide privacy guarantees under stronger server observations while retaining an aggregation-accuracy advantage under matched vector-level privacy. We investigate CPQ for the first objective and introduce CPQ-$\lambda$ to extend the privacy analysis, including a vector-level guarantee for one complete model upload. The second direction concerns the interaction between privacy-preserving model exchange and post-quantum-authenticated V2I communication. Under intermittent connectivity, model-update size and authentication overhead influence successful communication and verified participation in federated aggregation, which can ultimately affect routing performance. We examine these interactions through mobility-coupled, packet-level evaluation of vehicular FRL and complement the formal privacy analysis with vehicle-origin inference attack experiments.

\section{Differential Privacy for Federated Learning}
\label{sec:dp}

This section presents the privacy mechanisms used for local model exchange in vehicular FRL and analyzes their privacy and aggregation-error properties. We first establish a common parameter-level local differential privacy (PLDP) formulation and describe the independent one-bit LDP-FL~\citep{sun2020ldp} and Gaussian and Laplace additive-noise baselines~\citep{dwork2014algorithmic}. We then present CPQ and establish its privacy and aggregation-error properties. Finally, we introduce CPQ-$\lambda$ and analyze its guarantees under stronger observation models, including vector-level privacy for one complete model upload and aggregation-MSE comparisons under matched vector-level privacy.

\noindent\textbf{Common setup.}
At a given communication round, client $\mathcal{C}_i$ holds a locally trained parameter vector $\boldsymbol{\theta}_i=(\theta_{i,j})_{j\in[m]}\in\mathbb{R}^m$, where $m$ is the number of model parameters and $j$ indexes their coordinates. For each coordinate $j$, all clients use a common public clipping interval $\mathcal{I}_j=[c_j-r_j,c_j+r_j]$, where $c_j$ and $r_j>0$ denote the clipping center and radius, respectively. Clipping $\theta_{i,j}$ to $\mathcal{I}_j$ yields $w_{i,j}$, and $\mathbf{w}_i=(w_{i,j})_{j\in[m]}$ denotes the clipped parameter vector. In our implementation, $c_j$ corresponds to the current global-model parameter, and all coordinates share a common radius $r_j=r$. For a clipped input $w\in\mathcal{I}_j$, let $\overline{W}_{i,j}(w)$ denote the randomized parameter output, after reconstruction where applicable. We use the following parameter-level local differential privacy definition to characterize these randomized outputs.

\begin{definition}[Parameter-Level Local Differential Privacy (PLDP)]
\label{def:pldp}
For $\epsilon_p>0$ and $0\leq\delta_p<1$, a randomized parameter mechanism satisfies $(\epsilon_p,\delta_p)$-parameter-level local differential privacy (PLDP) if, for every client $i$, coordinate $j$, inputs $w,w'\in\mathcal{I}_j$, and measurable output set $\mathcal{T}\subseteq\mathbb{R}$,
\begin{equation}
\begin{aligned}
\Pr\!\left(\overline{W}_{i,j}(w)\in\mathcal{T}\right)
&\leq e^{\epsilon_p}
\Pr\!\left(\overline{W}_{i,j}(w')\in\mathcal{T}\right) +\delta_p.
\end{aligned}
\label{eq:pldp}
\end{equation}
Pure $\epsilon_p$-PLDP corresponds to $\delta_p=0$.
\end{definition}

\subsection{Independent Stochastic Binary Quantization (LDP-FL)}
\label{subsec:ldpfl}

LDP-FL~\citep{sun2020ldp} provides the independent one-bit quantization baseline for CPQ and satisfies pure $\epsilon_p$-PLDP for each clipped parameter. To simplify the scalar quantization rule, we fix a client $i$ and coordinate $j$ and write $(w,c,r)=(w_{i,j},c_j,r_j)$. For $\epsilon_p>0$, define
\begin{equation}
  \alpha(\epsilon_p)
  = \frac{e^{\epsilon_p}+1}{e^{\epsilon_p}-1},
  \qquad
  q(w)
  = \frac{1}{2}+\frac{w-c}{2r\,\alpha(\epsilon_p)}.
  \label{eq:alpha}
\end{equation}
Using independent local randomness, the client draws $B\sim\mathrm{Bernoulli}(q(w))$ and transmits the resulting one-bit symbol. With $U=2B-1\in\{-1,+1\}$, the server reconstructs
\begin{equation}
  \overline{W}
  =c+Ur\,\alpha(\epsilon_p),
  \qquad
  \overline{W}\in
  \bigl\{c-r\alpha(\epsilon_p),\,c+r\alpha(\epsilon_p)\bigr\}.
  \label{eq:ldpfl_output}
\end{equation}
The reconstruction is unbiased since $\mathbb{E}[\overline{W}]=w$.

For a fixed coordinate $j$, consider $n$ clients with clipped parameters $w_{i,j}\in\mathcal{I}_j$ and independent quantization randomness. Let $\overline{W}_{i,j}$ denote the reconstructed parameter of client $i$, and define the reconstructed aggregate $\overline{W}_{g,j}=\frac{1}{n}\sum_{i=1}^{n}\overline{W}_{i,j}$ and clipped-parameter mean $\bar{w}_j=\frac{1}{n}\sum_{i=1}^{n}w_{i,j}$. The aggregate is unbiased, and its mean squared error is
\begin{equation}
\begin{aligned}
  \mathbb{E}\!\left[
    (\overline{W}_{g,j}-\bar{w}_j)^2
  \right]
  &=\frac{1}{n^2}\sum_{i=1}^{n}
    \left[
      r_j^2\alpha^2(\epsilon_p)
      -(w_{i,j}-c_j)^2
    \right]\leq\frac{r_j^2\alpha^2(\epsilon_p)}{n}.
\end{aligned}
\label{eq:ldpfl_mse}
\end{equation}
The equality follows from the independence of client randomization, which eliminates cross-client error covariance. Moreover, among unbiased binary quantizers with input-independent reconstruction levels satisfying $\epsilon_p$-PLDP on the clipping interval, LDP-FL minimizes scalar MSE for every admissible input~\citep[Proposition~1]{salehi2024corbin}. Thus, its aggregation-error bound provides the reference for evaluating the benefits of correlated quantization.

Certain quantization and sign-based mechanisms admit one-bit scalar configurations with the same response distribution and reconstruction levels as LDP-FL. The corresponding parameter mappings and conditions for equivalent aggregation behavior are detailed in~Appendix~\ref{app:one_bit_equivalence}.

\subsection{Gaussian and Laplace Mechanisms}
\label{subsec:gauss_lap}

Gaussian and Laplace mechanisms provide additive-noise baselines for comparison with one-bit private quantization. Both perturb parameters after they have been clipped to the public intervals defined in the common setup.

\noindent\textbf{Gaussian mechanism.}
Each client adds independent Gaussian noise to its clipped parameters according to
\begin{equation*}
  \overline{W}_{i,j}=w_{i,j}+Z_{i,j},
  \qquad Z_{i,j}\sim\mathcal{N}(0,\sigma_j^2).
\end{equation*}
For coordinate $j$, the sensitivity is $\Delta_j=2r_j$. We calibrate $\sigma_j$ using the analytic Gaussian mechanism~\citep{balle2018improving}, selecting the smallest positive value satisfying
\begin{equation*}
\begin{aligned}
  &\Phi\!\left(
    \frac{\Delta_j}{2\sigma_j}
    -\frac{\epsilon_p\sigma_j}{\Delta_j}
  \right)-e^{\epsilon_p}\Phi\!\left(
    -\frac{\Delta_j}{2\sigma_j}
    -\frac{\epsilon_p\sigma_j}{\Delta_j}
  \right)
  \leq\delta_p,
\end{aligned}
\end{equation*}
where $\Phi$ denotes the standard normal cumulative distribution function. This calibration ensures $(\epsilon_p,\delta_p)$-PLDP for each parameter.

\noindent\textbf{Laplace mechanism.}
For pure $\epsilon_p$-PLDP, each client adds independent Laplace noise with scale $\Delta_j/\epsilon_p$, where $\Delta_j=2r_j$~\citep{dwork2014algorithmic}.
\begin{equation*}
  \overline{W}_{i,j}
  =w_{i,j}+\mathrm{Lap}(0,2r_j/\epsilon_p).
\end{equation*}

\subsection{Correlated Private Quantizer (CPQ)}
\label{subsec:cpq}

CPQ coordinates the one-bit quantization of paired vehicle clients using complementary shared randomness. It preserves the marginal quantization distribution of LDP-FL, maintaining unbiasedness and the same $\epsilon_p$-PLDP guarantee for each client, while reducing pairwise aggregation mean squared error (MSE) when both members of a pair contribute. We first describe the pairing and quantization procedures, then analyze their privacy and aggregation-error properties at the pair and population levels.

To realize this coupling, the vehicles establish a shared group secret and agree on an ordered participant list before training. For each coordinate $j$ and communication round $t$, a secret-keyed permutation of this list determines the client pairs. Each pair derives a common pseudorandom value $V$, with one client using $V$ and the other using $1-V$ for quantization. Pair assignments and random values are refreshed across coordinates and rounds without exchanging pairing information during training. In the implemented protocol, the RSU does not possess the group secret or the pairing assignments. The theoretical analysis models the shared values as independent $\mathrm{Uniform}[0,1]$ variables across pairs, while the group-establishment procedure and finite-precision implementation are detailed in Appendix~\ref{app:hidden-pairing} and Appendix~\ref{app:finite_precision}.

For each clipped parameter $w_{i,j}$, client $\mathcal{C}_i$ uses its assigned complementary random value $V_{i,j,t}$ in communication round $t$. Using the same quantization probability as LDP-FL, $q_{i,j}=\frac{1}{2}+\frac{w_{i,j}-c_j}{2r_j\alpha(\epsilon_p)}$, the client generates the one-bit output
\begin{equation}
  b_{i,j} =\mathbf{1}_{\{V_{i,j,t}<q_{i,j}\}}\in\{0,1\}.
  \label{eq:cpq_bit}
\end{equation}
The server reconstructs each parameter using the same two output levels as LDP-FL, $c_j\pm r_j\alpha(\epsilon_p)$. For $n$ client updates included in aggregation, indexed by $i\in[n]$, the reconstructed mean is
\begin{equation}
  \overline{W}_{g,j} =c_j+\frac{r_j\alpha(\epsilon_p)}{n} \sum_{i=1}^{n}(2b_{i,j}-1).
  \label{eq:cpq_agg}
\end{equation}
This reconstruction requires only the received bits and public parameters $(c_j,r_j,\epsilon_p)$, without revealing the pairing assignments. Algorithm~\ref{alg:cpq} summarizes the procedure. The complementary coupling preserves each client's quantization distribution while changing the dependence between paired outputs, which determines their aggregation error.

\begin{algorithm}[t]
\caption{Correlated Private Quantizer (CPQ)}
\label{alg:cpq}

\textbf{Inputs:} $n$ participating clients, $m$ model parameters, communication round $t$, privacy budget $\epsilon_p$ ($\alpha = \frac{e^{\epsilon_p}+1}{e^{\epsilon_p}-1}$), clipping parameters $(\mathbf{c},\mathbf{r})$, and global model $\mathbf{W}_{\mathrm{global}}^{(t)}$.

\textbf{Setup:} Each vehicle holds
$(\mathcal R,e,K_G)$, established before training
using Algorithm~\ref{alg:hidden-setup}.

\begin{algorithmic}[1]

\For{each client $\mathcal C_i$, in parallel}
    \State Receive $\mathbf{W}_{\mathrm{global}}^{(t)},\mathbf{c},\mathbf{r}$ from the server
    \State Perform local reinforcement learning to obtain $\boldsymbol{\theta}_i$

    \For{each coordinate $j\in[m]$}

        \State $w_{i,j}\gets \min\{c_j+r_j,\max\{c_j-r_j,\theta_{i,j}\}\}$

        \State $q_{i,j}\gets \frac{1}{2}+ \frac{w_{i,j}-c_j}{2r_j\alpha}$

        \State $V_{i,j,t}\gets \textsc{HiddenPairValue} (K_G,\mathcal R,e,i,j,t)$

        \State $b_{i,j}\gets \mathbf{1}_{\{V_{i,j,t}<q_{i,j}\}}$

    \EndFor

    \State Transmit $\mathbf{b}_i=(b_{i,j})_{j\in[m]}$ to the server

\EndFor

\State \textbf{Server aggregation}

\For{each coordinate $j\in[m]$}

    \State $\overline{W}_{g,j}\gets c_j+\dfrac{r_j\alpha}{n} \displaystyle\sum_{i=1}^{n}(2b_{i,j}-1)$

\EndFor

\end{algorithmic}
\end{algorithm}

\begin{theorem}[PLDP, unbiasedness, and pairwise MSE optimality of CPQ]
\label{thm:cpq}
Let $c\in\mathbb{R}$, $r>0$, and $0<\epsilon_p<\infty$. Consider two clipped inputs $w_1,w_2\in[c-r,c+r]$ and their CPQ reconstructions $Q_1,Q_2$, obtained using complementary random values $V$ and $1-V$, where $V\sim\mathrm{Uniform}[0,1]$. Write $\alpha=\alpha(\epsilon_p)$ and define $x_k=(w_k-c)/(r\alpha)$, $A=(w_1+w_2)/2$, and $\widehat{A}=(Q_1+Q_2)/2$.

\begin{enumerate}[label=(\roman*)]
\item \textbf{PLDP and unbiasedness.}
Each $Q_k$ follows the same output distribution as the LDP-FL quantizer and satisfies pure $\epsilon_p$-PLDP according to Definition~\ref{def:pldp}. Moreover, $\mathbb{E}[Q_k]=w_k$, so $\mathbb{E}[\widehat{A}]=A$. Applied coordinatewise, CPQ also provides an unbiased reconstructed aggregate for any fixed set of clients.

\item \textbf{Pairwise MSE optimality.}
Among all couplings of individually unbiased binary quantizers with reconstruction levels $c\pm r\alpha$ satisfying $\epsilon_p$-PLDP, CPQ minimizes the pair-average MSE for every admissible input pair, attaining
\begin{equation}
\begin{aligned}
m^*(w_1,w_2)
=\frac{r^2\alpha^2}{2}
\left[
|x_1+x_2|
-\frac{(x_1+x_2)^2}{2}
\right].
\end{aligned}
\label{eq:cpq_mse_opt}
\end{equation}
This MSE is strictly smaller than that of independent quantization with the same output distributions. 
\end{enumerate}
\end{theorem}

\begin{proof}[Proof sketch]
Since both $V$ and $1-V$ are uniformly distributed, each $Q_k$ follows the LDP-FL output distribution and therefore satisfies $\epsilon_p$-PLDP. Furthermore,
$\mathbb{E}[Q_k]=c+r\alpha(2q(w_k)-1)=w_k$,
establishing unbiasedness.

For the pairwise MSE, let $U_k=(Q_k-c)/(r\alpha)\in\{-1,+1\}$, with $\mathbb{E}[U_k]=x_k$. Expanding the pair-average error gives
\begin{equation*}
\begin{aligned}
\mathbb{E}[(\widehat{A}-A)^2]
&=\frac{r^2\alpha^2}{4}
\bigl[2+2\mathbb{E}[U_1U_2] -(x_1+x_2)^2\bigr].
\end{aligned}
\end{equation*}
The individual output distributions fix all terms except $\mathbb{E}[U_1U_2]$. For any coupling of these binary outputs, the Fr\'echet lower bound gives
$\mathbb{E}[U_1U_2]\geq |x_1+x_2|-1$.
The complementary random values used by CPQ attain this bound, yielding~\eqref{eq:cpq_mse_opt}.

Under independent quantization, $\mathbb{E}[U_1U_2]=x_1x_2$, giving
\begin{equation*}
\begin{aligned}
m_{\mathrm{ind}}-m^*
=\frac{r^2\alpha^2}{2}
\left[1+x_1x_2-|x_1+x_2|\right]>0.
\end{aligned}
\end{equation*}
The inequality is strict because $|x_k|\leq1/\alpha<1$. The complete derivation is provided in Appendix~\ref{app:proof_cpq}.
\end{proof}

The pairwise MSE result extends to population-level aggregation when all included clients form complete pairs and distinct pairs use independent shared randomness. Under these conditions, the aggregation error can be bounded using the pairwise result.

\begin{corollary}[Population aggregation MSE of CPQ]
\label{cor:cpq_pop}
For a fixed coordinate $j$, suppose the $n$ client updates included in aggregation form $n/2$ complete pairs, with independent shared randomness across distinct pairs. All clients use the same public clipping parameters $(c_j,r_j)$. Let $\bar{w}_j=\frac{1}{n}\sum_{i=1}^{n}w_{i,j}$ denote the mean clipped parameter. Then
\begin{equation}
\begin{aligned}
  \mathbb{E}\!\left[(\overline{W}_{g,j}-\bar{w}_j)^2\right]
  &\leq \frac{r_j^2}{n}\,\kappa(\epsilon_p),\\
  \kappa(\epsilon_p)
  &=
  \begin{cases}
    \alpha^2(\epsilon_p)/2,
      & \alpha(\epsilon_p)\leq 2,\\
    2\bigl(\alpha(\epsilon_p)-1\bigr),
      & \alpha(\epsilon_p)>2.
  \end{cases}
\end{aligned}
\label{eq:cpq_pop}
\end{equation}
For every finite $\epsilon_p>0$, $\kappa(\epsilon_p)<\alpha^2(\epsilon_p)$, yielding a strictly tighter aggregation-MSE bound than independent LDP-FL in~\eqref{eq:ldpfl_mse}. As $\epsilon_p\to0$, the CPQ bound grows linearly in $\alpha(\epsilon_p)$, compared with quadratic growth for LDP-FL.
\end{corollary}

\begin{proof}
Fix coordinate $j$ and write $c=c_j$, $r=r_j$, and $\alpha=\alpha(\epsilon_p)$. Let $S_\ell$ and $s_\ell$ denote the reconstructed and clipped-parameter sums of pair $\ell$, respectively. Since the pair-sum errors are unbiased and independent across distinct pairs,
\begin{equation*}
\mathbb{E}\!\left[
  (\overline{W}_{g,j}-\bar{w}_j)^2
\right]
=\frac{1}{n^2}\sum_{\ell=1}^{n/2}
\mathbb{E}\!\left[(S_\ell-s_\ell)^2\right].
\end{equation*}
For $t_\ell=|s_\ell-2c|\in[0,2r]$, the pair-sum MSE is four times the pair-average MSE in~\eqref{eq:cpq_mse_opt}, giving
\begin{equation*}
\mathbb{E}\!\left[(S_\ell-s_\ell)^2\right]
=2r\alpha t_\ell-t_\ell^2.
\end{equation*}
This concave expression is maximized over $[0,2r]$ at $t_\ell=\min\{r\alpha,2r\}$. The resulting upper bound is $r^2\alpha^2$ when $\alpha\leq2$ and $4r^2(\alpha-1)$ otherwise. Summing over the $n/2$ pairs yields~\eqref{eq:cpq_pop}. Finally, $\kappa(\epsilon_p)<\alpha^2(\epsilon_p)$ for $\alpha(\epsilon_p)>1$, establishing the comparison with LDP-FL.
\end{proof}

CPQ also supports partial participation, allowing each received update to contribute to aggregation even when its paired client's update is unavailable. Complete pairs retain the pairwise MSE advantage of CPQ, while unmatched clients contribute the same quantization-error variance as independent LDP-FL. Consequently, for a fixed set of accepted clients selected independently of the quantization randomness, CPQ preserves unbiased aggregation and achieves MSE no greater than independent quantization of the same clients, with a strict improvement whenever at least one complete pair contributes. The population bound in~\eqref{eq:cpq_pop} characterizes the complete-pair case.

\subsection{CPQ-$\lambda$ under Stronger Observations}
\label{subsec:cpqlambda}

CPQ-$\lambda$ extends CPQ by combining complementary shared randomness with independent local randomness in the one-bit quantization process. The mixing probabilities control the dependence between paired outputs, allowing us to characterize aggregation-error behavior and privacy under stronger observation models, including access to both outputs of a pair or shared-randomness side information. The one-bit representation and server reconstruction rule remain unchanged. We first describe the mixing mechanism and then establish its privacy and aggregation-error properties.

For a leader--follower pair and a fixed coordinate, let $V\sim\mathrm{Uniform}[0,1]$ denote the shared random value. Each client $k\in\{1,2\}$ independently draws a mixing indicator $C_k\sim\mathrm{Bernoulli}(\lambda_k)$ and a private random value $R_k\sim\mathrm{Uniform}[0,1]$, where $\lambda_k\in[0,1]$. All these random variables are mutually independent. When $C_k=1$, the client uses its assigned shared randomness, while $C_k=0$ selects independent local randomness. The resulting quantization thresholds are
\begin{equation}
  T_1=C_1V+(1-C_1)R_1,\qquad
  T_2=C_2(1-V)+(1-C_2)R_2.
  \label{eq:cpqlambda_thresholds}
\end{equation}
Using the quantization probability $q(w_k)$ defined in~\eqref{eq:alpha}, each client transmits $b_k=\mathbf{1}_{\{T_k<q(w_k)\}}$. Since both thresholds are uniformly distributed, $\Pr(b_k=1)=q(w_k)$, and the reconstruction rule in~\eqref{eq:cpq_agg} remains unchanged.

Let $\mu=\lambda_1\lambda_2$ denote the probability that both clients select the shared-randomness branch. For fixed clipped inputs, the joint output distribution is
\begin{equation}
  P_{\lambda}
  =\mu P_{\mathrm{CPQ}}+(1-\mu)P_{\mathrm{IND}},
  \label{eq:cpqlambda_mixture}
\end{equation}
where $P_{\mathrm{CPQ}}$ and $P_{\mathrm{IND}}$ denote the joint distributions under complementary CPQ and independent LDP-FL quantization, respectively. When at least one client selects independent local randomness, the two outputs are independent. Thus, $\mu=0$ yields independent quantization, while $\lambda_1=\lambda_2=1$ recovers CPQ. The dependence between paired outputs motivates two privacy criteria beyond the individual-output PLDP guarantee. The first considers an observer with access to both bits of a known pair, while the second allows additional side information about the shared randomness and paired client.

\begin{definition}[Pair-view PLDP]
\label{def:pairview_pldp}
For a fixed pair and coordinate, suppose an observer knows the pairing and observes $(b_1,b_2)\in\{0,1\}^2$. The mechanism satisfies $\epsilon$-pair-view PLDP for client~1 if, for every $w_1,w_1',w_2\in[c-r,c+r]$ and every set $\mathcal{T}\subseteq\{0,1\}^2$,
\begin{equation}
  \Pr\!\left((b_1,b_2)\in\mathcal{T}\mid w_1,w_2\right)
  \leq e^{\epsilon}
  \Pr\!\left((b_1,b_2)\in\mathcal{T}\mid w_1',w_2\right).
  \label{eq:pairview_pldp}
\end{equation}
The definition for client~2 is symmetric, with $w_1$ held fixed.
\end{definition}

\begin{definition}[PLDP under side information]
\label{def:sideinfo_pldp}
For a fixed pair and coordinate, let $Z$ denote side information generated without access to the protected input $w_1$ and independently of client~1's private random variables $(C_1,R_1)$. Non-protected inputs included in $Z$ are held fixed in the privacy comparison. The mechanism satisfies $\epsilon$-PLDP for client~1 conditioned on $Z$ if, for every $w_1,w_1'\in[c-r,c+r]$, every set $\mathcal{T}\subseteq\{0,1\}$, and almost every realization $z$ of $Z$,
\begin{equation}
  \Pr(b_1\in\mathcal{T}\mid w_1,Z=z)
  \leq e^{\epsilon}
  \Pr(b_1\in\mathcal{T}\mid w_1',Z=z).
  \label{eq:sideinfo_pldp}
\end{equation}
The definition for client~2 is symmetric.
\end{definition}

For CPQ-$\lambda$, we consider the side information $Z_1=(V,C_2,R_2,w_2)$ for client~1, with the corresponding observation defined symmetrically for client~2. This observation includes the shared random value and the paired client's input and local randomization, but excludes the protected client's private randomness. We now characterize the privacy guarantees under these observation models and the associated aggregation-MSE properties.

\begin{theorem}[Privacy and aggregation MSE of CPQ-$\lambda$]
\label{thm:cpqlambda}
Consider CPQ-$\lambda$ with $0<\epsilon_p<\infty$ and
$\lambda_1,\lambda_2\in[0,1)$. For a fixed pair and coordinate,
let $w_1,w_2\in[c-r,c+r]$, set $\mu=\lambda_1\lambda_2$,
and write $\alpha=\alpha(\epsilon_p)$. The following properties hold.

\begin{enumerate}[label=(\alph*)]
\item \textbf{PLDP and unbiasedness.}
Each client's one-bit output follows the LDP-FL distribution
and satisfies pure $\epsilon_p$-PLDP. The reconstructed
parameters and their arithmetic mean are unbiased.

\item \textbf{Privacy under stronger observations.}
Under the pair-view observation of Definition~\ref{def:pairview_pldp},
CPQ-$\lambda$ satisfies $\epsilon_S$-pair-view PLDP, where
\begin{equation}
\epsilon_S(\epsilon_p,\mu)
=\ln\!\left(
e^{\epsilon_p}
+\frac{\mu}{1-\mu}
\left(e^{\epsilon_p}-e^{-\epsilon_p}\right)
\right).
\label{eq:eps_server}
\end{equation}
Under the side-information observation
$Z_1=(V,C_2,R_2,w_2)$ of
Definition~\ref{def:sideinfo_pldp}, client~1 satisfies
$\epsilon_C$-PLDP, where
\begin{equation}
\epsilon_C(\epsilon_p,\lambda_1)
=\ln\!\left(
\frac{e^{\epsilon_p}+\lambda_1}{1-\lambda_1}
\right).
\label{eq:eps_collusion}
\end{equation}
Both bounds are tight. The corresponding side-information
bound for client~2 follows by symmetry.

\item \textbf{Pairwise MSE.}
Let $x_k=(w_k-c)/(r\alpha)$, and let
$m_{\mathrm{IND}}(w_1,w_2)
=\frac{r^2\alpha^2}{4}(2-x_1^2-x_2^2)$
denote the pair-average MSE of independent LDP-FL
quantization. With $m^*(w_1,w_2)$ defined
in~\eqref{eq:cpq_mse_opt}, the CPQ-$\lambda$ pair-average MSE is
\begin{equation}
\begin{aligned}
m_{\mu}(w_1,w_2)
&=(1-\mu)m_{\mathrm{IND}}(w_1,w_2)+\mu m^*(w_1,w_2).
\end{aligned}
\label{eq:cpqlambda_mse}
\end{equation}
For every $\mu>0$, this MSE is strictly smaller than that
of independent quantization for all admissible input pairs.

\item \textbf{Population MSE.}
For a fixed coordinate $j$, suppose the $n$ updates included
in aggregation form $n/2$ complete pairs, with independent
randomness across distinct pairs. Let $\mu_\ell$ denote the
mixing probability of pair $\ell$, and define
$\bar{\mu}=\frac{2}{n}\sum_{\ell=1}^{n/2}\mu_\ell$.
Then
\begin{equation}
\begin{aligned}
\mathbb{E}\!\left[
(\overline{W}_{g,j}-\bar{w}_j)^2
\right]
&\leq\frac{r_j^2}{n}
\left[
(1-\bar{\mu})\alpha^2(\epsilon_p)
+\bar{\mu}\,\kappa(\epsilon_p)
\right],
\end{aligned}
\label{eq:cpqlambda_pop}
\end{equation}
where $\kappa$ is defined in~\eqref{eq:cpq_pop}.
\end{enumerate}
\end{theorem}

\begin{proof}[Proof sketch]
Write $q_k=q(w_k)$, $q_-=(1+e^{\epsilon_p})^{-1}$, and
$q_+=1-q_-$. Since each threshold $T_k$ is uniformly
distributed, the output $b_k$ follows the LDP-FL distribution,
preserving $\epsilon_p$-PLDP and unbiased reconstruction.

For pair-view privacy, the finite output alphabet allows the
privacy loss to be bounded through the likelihood ratios of
the four possible bit pairs. Using the mixture
in~\eqref{eq:cpqlambda_mixture}, the largest ratio is attained
for the output $(1,1)$ when comparing $q_1=q_+$ with
$q_1'=q_-$ at $q_2=q_+$. This ratio equals
$e^{\epsilon_S}$ in~\eqref{eq:eps_server}, establishing
the bound and its tightness.

For the side-information model, conditioning on
$Z_1=(V,C_2,R_2,w_2)$ leaves the protected client's
private randomness $(C_1,R_1)$ undisclosed. Hence
\begin{equation*}
\Pr(b_1=1\mid w_1,Z_1)
=\lambda_1\mathbf{1}_{\{V<q_1\}}
 +(1-\lambda_1)q_1.
\end{equation*}
Maximizing the conditional likelihood ratios over both
possible bit values yields $e^{\epsilon_C}$
in~\eqref{eq:eps_collusion}. The bound is attained when
$V\in[q_-,q_+)$, proving tightness. The result for client~2
follows symmetrically.

The mixture representation directly gives the pairwise
MSE identity in~\eqref{eq:cpqlambda_mse}. Its strict
improvement for $\mu>0$ follows from the pairwise
optimality result of Theorem~\ref{thm:cpq}. For population
aggregation, the independent, zero-mean errors of complete
pairs allow their MSE contributions to be summed. Applying
the independent-quantization bound and
Corollary~\ref{cor:cpq_pop} to each pair yields
\eqref{eq:cpqlambda_pop}. Complete derivations are
provided in Appendix~\ref{app:proof_cpqlambda}.
\end{proof}

\subsection{Vector-Level Privacy under Server Observation}
\label{subsec:vector_privacy}

We extend the privacy analysis of CPQ-$\lambda$ to a complete model upload, complementing the parameter-level and stronger-observation guarantees established above. For one communication round, let $\mathcal{B}=\prod_{j=1}^{m}\mathcal{I}_j$ denote the set of admissible clipped model vectors. We consider any two vectors $\mathbf{w}_i,\mathbf{w}'_i\in\mathcal{B}$ for a target vehicle $i$, while holding the public round state and all other vehicles' clipped inputs fixed.

\begin{definition}[Vector-Level Local Differential Privacy]
\label{def:vector_ldp}
For $\epsilon_v\geq0$ and $0\leq\delta_v<1$, a randomized observation mechanism $\mathcal{M}_{\mathrm{obs}}$ satisfies $(\epsilon_v,\delta_v)$ vector-level local differential privacy if, for any target client $i$, any two clipped model vectors $\mathbf{w}_i,\mathbf{w}'_i\in\mathcal{B}$, and every measurable output set $\mathcal{S}$,
\begin{equation}
\begin{aligned}
&\Pr\!\left(
\mathcal{M}_{\mathrm{obs}}(\mathbf{w}_i)\in\mathcal{S}
\right)\leq e^{\epsilon_v}
\Pr\!\left(
\mathcal{M}_{\mathrm{obs}}(\mathbf{w}'_i)\in\mathcal{S}
\right)+\delta_v.
\end{aligned}
\label{eq:vector_ldp}
\end{equation}
The public round state and all other clients' clipped inputs are held fixed in this comparison.
\end{definition}

We consider three observation models. An \emph{upload-only observer} receives one client's quantized model vector, while the \emph{RSU observer} receives the quantized updates of all participating clients without knowing the secret pairing assignments or shared randomness. We further consider a stronger \emph{pairing-revealed observer} that additionally knows the client pairings but not the shared random values. Establishing privacy under this stronger observation also provides a guarantee for the RSU when the pairing assignments remain hidden. 

To establish vector-level privacy for one complete CPQ-$\lambda$ model upload, we use R\'enyi differential privacy (RDP)~\citep{mironov2017renyi}. For an order $\nu>1$ and discrete distributions $P$ and $Q$, the R\'enyi divergence is defined as
\begin{equation}
D_\nu(P\|Q)
=\frac{1}{\nu-1}
\ln\sum_x P(x)^\nu Q(x)^{1-\nu}.
\label{eq:renyi_divergence}
\end{equation}
A mechanism satisfies $(\nu,\rho)$-RDP if this divergence is at most $\rho$ for every admissible input comparison. We apply this criterion to the joint CPQ-$\lambda$ output distributions at each coordinate. Under independent randomization across coordinates, the RDP bounds compose additively, yielding an $(\epsilon_v,\delta_v)$ vector-level privacy guarantee for any $0<\delta_v<1$, where
$\epsilon_v=\rho+\ln(1/\delta_v)/(\nu-1)$.

For symmetric CPQ-$\lambda$, define $q_-=(1+e^{\epsilon_p})^{-1}$, $q_+=1-q_-$, and $\mu=\lambda^2$. For a fixed coordinate, let $P_{u,z}$ denote the joint distribution of the paired output bits $(B_1,B_2)$, with $\Pr(B_1=1)=u$ and $\Pr(B_2=1)=z$, where $u,z\in[q_-,q_+]$. By~\eqref{eq:cpqlambda_mixture},
\begin{equation*}
P_{u,z}(1,1) =(1-\mu)uz+\mu(u+z-1)_+,
\end{equation*}
where $(x)_+=\max\{x,0\}$. The remaining joint probabilities are determined by the two Bernoulli marginals.

\begin{theorem}[Vector-level privacy of CPQ-$\lambda$]
\label{thm:cpqlambda_vector}
Consider one model upload using symmetric CPQ-$\lambda$,
with $\epsilon_p>0$ and $0\leq\lambda<1$.
Assume independent ideal randomization across
coordinates and distinct client pairs, with
pair assignments independent of the protected inputs.
The public round state and all other clients'
clipped inputs are held fixed.

Under the pairing-revealed observation model,
the observer sees all quantized uploads and knows
the client pairings but not the shared random values.
For any R\'enyi order $\nu>1$, the exact worst-case
per-coordinate divergence is
\begin{equation}
\begin{aligned}
d_\nu
&=\frac{1}{\nu-1}
\ln\max_{\substack{
(u,v)\in\{(q_-,q_+),(q_+,q_-)\}\\
z\in\{q_-,q_+\}}}\\
&\quad\sum_{x\in\{0,1\}^2}
P_{u,z}(x)^\nu
P_{v,z}(x)^{1-\nu}.
\end{aligned}
\label{eq:cpqlambda_rdp}
\end{equation}

For a model with $m$ parameters, CPQ-$\lambda$
satisfies $(\epsilon_v,\delta_v)$ vector-level
local differential privacy for one complete upload,
where, for any $0<\delta_v<1$, an RDP-based
privacy guarantee is obtained with
\begin{equation}
\epsilon_v
=\inf_{\nu>1}
\left[
m\,d_\nu+
\frac{\ln(1/\delta_v)}{\nu-1}
\right].
\label{eq:cpqlambda_vector_privacy}
\end{equation}
This guarantee holds with pairings revealed and
therefore also applies to the RSU observation
model in which pairings remain hidden.
\end{theorem}

\begin{proof}[Proof sketch]
For a fixed partner probability $z$, the partner's
output distribution is unchanged when the target
input varies. Conditioning on the partner's bit
yields Bernoulli distributions whose success
probabilities are nondecreasing in the target
probability $u$. The largest R\'enyi divergence
is therefore attained when the alternative
target probabilities are $q_-$ and $q_+$,
in either order.

For these target endpoints, the joint output
probabilities are affine in $z$. The joint
convexity of the R\'enyi power sum for $\nu>1$
then places the maximum at $z\in\{q_-,q_+\}$.
This establishes the exact per-coordinate
expression in~\eqref{eq:cpqlambda_rdp}.

Under independent coordinate randomization,
the joint observation distribution factorizes
across the $m$ model coordinates. Outputs of
pairs not containing the target client have
identical distributions under the two protected
inputs and contribute zero divergence. Hence
the total RDP divergence is bounded by $m\,d_\nu$.
Converting this bound to approximate differential
privacy and optimizing over $\nu>1$ yields
\eqref{eq:cpqlambda_vector_privacy}. Hiding the
pairing assignments is a post-processing of the
pairing-revealed observation and preserves the
guarantee. The complete proof is provided
in Appendix~\ref{app:cpqlambda_vector_proof}.
\end{proof}

We compare the analytical per-coordinate FedAvg MSE under a common vector-level privacy target of $(\epsilon_v,\delta_v)=(1,10^{-5})$ for one complete model upload. All mechanisms are calibrated using the same RDP accounting framework, with $m=29{,}653$ parameters, $K=50$ vehicles, and clipping radius $r=0.02$. To isolate the error introduced by privacy randomization, we consider identical local and global model parameters, such that $w_{i,j}=c_j$ for all clients and coordinates. For binary quantization, define $A(s)=r(e^s+1)/(e^s-1)$, where $s$ denotes the mechanism-specific parameter-level privacy budget. CPQ-$\lambda$ uses symmetric mixing with $\lambda_1=\lambda_2=\lambda=0.627$ and is evaluated under the pairing-revealed observation model.

\begin{table}[!t]
\centering
\caption{Analytical FedAvg MSE under matched vector-level privacy.}
\label{tab:vector_privacy_comparison}
\scriptsize
\setlength{\tabcolsep}{1.5pt}
\renewcommand{\arraystretch}{1.15}
\begin{tabular*}{\columnwidth}{@{\extracolsep{\fill}}lcccc@{}}
\toprule
Mechanism & Calibration & MSE formula & MSE $\downarrow$ &
\shortstack{Bits/\\weight} \\
\midrule
LDP-FL & $s_0=0.001185$ & $A(s_0)^2/K$ & 22.79 & 1 \\
CPQ-$\lambda$ & $s=0.001010$ & $(1-\lambda^2)A(s)^2/K$ & \textbf{19.03} & 1 \\
Gaussian & $\sigma=33.76$ & $\sigma^2/K$ & 22.79 & 32 \\
Laplace & $\beta=33.75$ & $2\beta^2/K$ & 45.55 & 32 \\
\bottomrule
\end{tabular*}
\end{table}

As shown in Table~\ref{tab:vector_privacy_comparison}, CPQ-$\lambda$ achieves 16.5\% lower aggregation MSE than independent LDP-FL at the same RDP-certified vector-level privacy target, while retaining one-bit communication. The analysis in Appendix~\ref{app:cpqlambda_vector_proof} further establishes an MSE advantage over LDP-FL for all admissible clipped inputs under the considered calibration. These comparisons use sufficient RDP-based privacy certificates and characterize theoretical aggregation error rather than routing performance under the stated vector-level privacy settings.

\section{Post-Quantum Cryptography}
\label{sec:pqc}

In vehicular FRL, digital signatures authenticate global-model distributions and local-update uploads, allowing recipients to verify message origin and integrity~\citep{twardokus2024cryptography}. We consider post-quantum signatures to address long-term quantum security concerns while examining how their communication and computational overhead affects authenticated model exchange over mobility-constrained V2I links.

We evaluate three post-quantum signature schemes with distinct cryptographic constructions and authentication overheads. ML-DSA, derived from CRYSTALS-Dilithium~\citep{ducas2018crystals}, is a module-lattice-based scheme standardized under FIPS~204~\citep{nist_fips204}, while SLH-DSA, derived from SPHINCS+~\citep{bernstein2019sphincs}, is a stateless hash-based scheme standardized under FIPS~205~\citep{nist_fips205}. Falcon is an NTRU-lattice-based signature scheme~\citep{prest2020falcon} selected by NIST as the basis for FN-DSA~\citep{alagic2022status}. The primary mobility-coupled V2I evaluation uses ML-DSA-44 to hold authentication constant, while controlled benchmarks and an additional signature-sensitivity experiment compare ML-DSA-44, Falcon-512, and SLH-DSA-SHA2-128f.

The evaluated parameter sets differ substantially in authentication overhead. ML-DSA-44 produces fixed-size 2,420-byte signatures~\citep{nist_fips204}, while Falcon-512 generates compact, variable-length signatures averaging 655.1 bytes in our controlled microbenchmark~\citep{prest2020falcon}. In contrast, SLH-DSA-SHA2-128f produces 17,088-byte signatures~\citep{nist_fips205}. These differences are particularly relevant for bit-packed one-bit model updates, where authentication data can constitute a substantial fraction of the transmitted payload. Detailed measurements of public-key sizes, signature sizes, and cryptographic operation times are provided in Appendix~\ref{app:signature_microbenchmark}.

\section{Federated Deep Q-Learning for Vehicular Routing}
\label{sec:frl}

This section describes the adopted vehicular routing formulation, local DQN training, and federated model aggregation.

\subsection{Vehicular Routing Formulation}
\label{subsec:routing_formulation}

We adopt FLDQN~\citep{mamond2025fldqn}, a federated routing framework based on deep Q-networks (DQNs)~\citep{mnih2015human}. Each learning vehicle selects feasible outgoing road edges at successive intersections in response to evolving traffic conditions. In the SUMO-based simulation~\citep{lopez2018microscopic}, learning vehicles obtain network-wide traffic observations through the TraCI interface to construct their routing states, while background vehicles contribute to congestion without participating in federated training. For a road network with $M$ directed edges, the state of a learning vehicle at decision step $t$ is defined as
\begin{equation}
\mathbf{S}_t =
\left[
\mathbf{n}_t,\,
\bar{\mathbf{v}}_t,\,
\mathbf{l},\,
\mathbf{p}_t,\,
\mathbf{p}_{\mathrm{dest}}
\right],
\label{eq:state}
\end{equation}

where $\mathbf{n}_t,\bar{\mathbf{v}}_t,\mathbf{l}\in\mathbb{R}^{M}$ represent the edgewise vehicle counts, mean speeds, and lengths, respectively. The vectors $\mathbf{p}_t,\mathbf{p}_{\mathrm{dest}}\in\mathbb{R}^{2}$ specify the coordinates of the starting point of the vehicle's current lane and its destination. The resulting state dimension is $3M+4$.

At each routing decision, the DQN produces Q-values for a fixed set of action indices $\mathcal{A}_{\max}$, whose size corresponds to the maximum number of turning choices at any intersection. Since feasible turns vary by location, unavailable actions are masked during action selection, ensuring that each selected action corresponds to a reachable outgoing edge.

Routing rewards are based on the travel time experienced under the prevailing traffic conditions. Let $u_t$ denote the simulation time at which a vehicle enters the edge associated with routing decision $t$. Upon entering the subsequent edge, the base reward is
\begin{equation}
r_t^{\mathrm{base}}
=-(u_{t+1}-u_t).
\label{eq:reward}
\end{equation}
To discourage cyclic routes, an additional penalty is imposed for repeated edge visits. Let $N_t(e_t)$ denote the number of visits to the selected edge $e_t$ within the current trip, including the current visit. The complete reward is
\begin{equation}
r_t=r_t^{\mathrm{base}}+
\begin{cases}
0, & N_t(e_t)=1,\\
-\kappa_2, & N_t(e_t)=2,\\
-\kappa_{3+}, & N_t(e_t)\geq 3,
\end{cases}
\label{eq:loop_penalty}
\end{equation}
where $\kappa_{3+}>\kappa_2>0$ penalize the second and subsequent repeated visits. The reward thus favors shorter experienced travel times while discouraging repeated edge traversal.

A vehicle's trip ends when it reaches its destination, encounters a dead end, or exceeds the maximum trip duration $T_{\mathrm{trip}}$. Time-limit truncations are stored as nonterminal transitions in the local experience replay buffer, as detailed in Appendix~\ref{app:exp_config}.
\subsection{Local DQN Training and Federated Aggregation}
\label{subsec:fldqn_training}

Each vehicle maintains an online Q-network $Q(s,a;\theta_i)$ and a target network $Q(s,a;\theta_i^-)$, whose parameters are periodically synchronized with the online network. The networks use fully connected layers with nonlinear hidden activations and a linear output layer. Each agent collects transitions in a local experience replay buffer $\mathcal{B}_i$ and trains its online network using sampled mini-batches. Routing decisions follow an $\epsilon$-greedy exploration policy.

For a sampled transition $j$, let $R_j$ denote the observed reward. The Bellman target is
\begin{equation}
y_j =
R_j+\gamma(1-d_j)
\max_{a'\in\mathcal{A}_{\max}}
Q(s_j',a';\theta_i^-),
\end{equation}
where $\gamma$ is the discount factor and $d_j\in\{0,1\}$ indicates whether the transition is terminal. The target maximization uses the full network output action set, whereas infeasible turns are masked during routing action selection.

Defining the temporal-difference error as $e_j=y_j-Q(s_j,a_j;\theta_i)$, the online network minimizes the Huber loss
\begin{equation}
\begin{aligned}
\mathcal{L}_{\mathrm{Huber}}(\theta_i)
&=\frac{1}{B}\sum_{j=1}^{B}
\ell_{\delta_{\mathrm{H}}}(e_j),\\
\ell_{\delta_{\mathrm{H}}}(e)
&=
\begin{cases}
\frac{1}{2}e^2,
& |e|\leq\delta_{\mathrm{H}},\\
\delta_{\mathrm{H}}
\left(|e|-\frac{\delta_{\mathrm{H}}}{2}\right),
& |e|>\delta_{\mathrm{H}},
\end{cases}
\end{aligned}
\label{eq:huber_loss}
\end{equation}
where $B$ is the mini-batch size and $\delta_{\mathrm{H}}$ is the Huber-loss threshold. The network architecture, exploration schedule, and optimization parameters are specified in Appendix~\ref{app:exp_config}.

At federated round $t$, each vehicle performs $E$ local training episodes using its available model. Vehicles that successfully receive the authenticated current-round global model $\mathbf{W}_{\mathrm{global}}^{(t)}$ are eligible to upload their locally trained online-network parameters after privacy processing. Target networks and experience replay buffers remain local.

Let $\mathcal{V}_t$ denote the set of clients whose current-round uploads are successfully delivered and verified. For each $i\in\mathcal{V}_t$, let $\overline{\mathbf{w}}_i^{(t+1)}$ denote the model parameters reconstructed at the RSU using the selected privacy mechanism. The global model is updated through uniform federated averaging
\begin{equation}
\mathbf{W}_{\mathrm{global}}^{(t+1)}
=
\frac{1}{|\mathcal{V}_t|}
\sum_{i\in\mathcal{V}_t}
\overline{\mathbf{w}}_i^{(t+1)},
\label{eq:uniform_agg}
\end{equation}
when $\mathcal{V}_t\neq\emptyset$. Otherwise, the global model remains unchanged. The authenticated, contact-dependent model-exchange procedure is detailed in Section~\ref{sec:system}.

\section{System Design and Implementation}
\label{sec:system}

The implementation couples a shared SUMO traffic environment with ns-3 through synchronized vehicle mobility updates. Learning vehicles train local DQN policies using SUMO traffic observations, while a roadside unit (RSU) coordinates post-quantum-authenticated model exchange over simulated IEEE~802.11p V2I links. Global models are distributed in full precision, whereas eligible vehicles return privacy-protected updates. Mobility-dependent delivery and authentication determine which updates are accepted for federated aggregation. Figure~\ref{fig:system_architecture} illustrates the integrated architecture.

\subsection{Privacy-Protected and Authenticated Model Exchange}
\label{subsec:privacy_integration}

Before privacy processing, each eligible vehicle clips its locally trained model parameters relative to the authenticated current-round global model $\mathbf{W}_{\mathrm{global}}^{(t)}$. For coordinate $j$, the clipping center is the corresponding global parameter $W_{\mathrm{global},j}^{(t)}$, with a common radius $r>0$ across all coordinates. The clipped parameter is
\begin{equation}
w_{i,j}^{(t+1)}
=
W_{\mathrm{global},j}^{(t)}
+
\operatorname{clip}\!\left(
\theta_{i,j}^{(t+1)}
-W_{\mathrm{global},j}^{(t)},
-r,\;r
\right),
\label{eq:clip_cr}
\end{equation}
where $\theta_{i,j}^{(t+1)}$ denotes the locally trained online-network parameter. This instantiates the public clipping parameters of Section~\ref{sec:dp} as $c_j=W_{\mathrm{global},j}^{(t)}$ and $r_j=r$, which are also used for server-side quantized-parameter reconstruction.

For CPQ and CPQ-$\lambda$, vehicles establish a shared group secret and agree on an ordered participant list $\mathcal R$ before federated training. Secret-derived ChaCha20 streams~\citep{chacha20} determine coordinate- and round-dependent pair assignments and complementary pseudorandom values for paired clients. These values are generated locally without exchanging pairing information during training. The RSU relays authenticated setup messages but does not obtain the group secret or pairing assignments, which are not required for model reconstruction. Pairings are formed over the entire group $\mathcal R$, independently of successful wireless delivery. Consequently, a received update remains usable when its partner's update is unavailable, although the corresponding pairwise error cancellation is not realized. The complete setup and randomness-derivation procedures are provided in Appendix~\ref{app:hidden-pairing}.

CPQ, CPQ-$\lambda$, and LDP-FL represent each model parameter with one bit. The binary outputs are packed separately for each parameter tensor using \texttt{numpy.packbits}, with the original tensor shapes stored as metadata. The packed arrays and metadata are serialized into compressed NPZ payloads. At the RSU, padding bits are removed according to the original tensor sizes, and the binary values are restored to their tensor shapes before parameter reconstruction. In contrast, Gaussian and Laplace updates retain float32 representations, as do global-model downlinks. These payloads are subsequently authenticated and transmitted through the V2I communication framework.

Before federated training, the RSU and learning vehicles generate post-quantum signing key pairs and register their public keys through an assumed authenticated bootstrap channel, which may be supported by V2X credential-management infrastructure~\citep{twardokus2024cryptography}. The primary mobility-coupled evaluation uses ML-DSA-44, while the signature-sensitivity experiment considers Falcon-512 and SLH-DSA-SHA2-128f. In each configuration, the selected signature scheme authenticates global-model downlinks, local-update uplinks, and CPQ/CPQ-$\lambda$ group-establishment messages. For CPQ and CPQ-$\lambda$, vehicles use ML-KEM-512 key pairs to establish a shared group secret through authenticated, ML-KEM-protected messages relayed by the RSU~\citep{fips203}. The RSU does not obtain the group secret, and the establishment protocol is provided in Appendix~\ref{app:hidden-pairing}.

Each authenticated model object contains a compressed NPZ payload $P$, a serialized JSON header $H_{\rm bytes}$, and a post-quantum signature $\sigma$, with length fields delimiting the header and signature. The header identifies the sender, vehicle, federated round, transmission direction, and signature scheme, together with payload metadata. For both uplink and downlink exchanges, the sender signs the serialized header and payload jointly
\begin{equation}
\sigma
=\mathrm{Sign}_{\mathcal A}\!\left(
\mathit{sk}_{\rm sender},
H_{\rm bytes}\,\|\,P
\right),
\label{eq:sign_actual}
\end{equation}
where $\mathcal A$ denotes the selected signature scheme. After complete SR-ARQ reassembly, the receiver verifies the signature using the sender's registered public key and validates the signed sender, vehicle, round, and direction against the expected transfer context. These checks bind the model content to its authenticated communication context. Only verified current-round downlinks establish uplink eligibility, while only successfully delivered and verified uplinks contribute to aggregation.

\subsection{Mobility-Coupled V2I Communication}
\label{subsec:ns3}

We use ns-3~\citep{riley2010ns3} to model IEEE~802.11p V2I communication between learning vehicles and a single RSU located at junction J3 of the SUMO road network. The wireless channel combines log-distance path loss, Nakagami-$m$ fading, and constant-speed propagation delay, with fixed-rate OFDM transmission and an 802.11p ad hoc MAC. The evaluated road-network topology is shown in Appendix~\ref{app:simulation_details}, while the radio and communication parameters are detailed in Appendix~\ref{app:exp_config}.

At each SUMO simulation step $\Delta t$, the positions and velocities of active learning vehicles are supplied to ns-3, which advances synchronously with the SUMO simulation clock. Between updates, vehicle positions evolve according to a constant-velocity mobility model. Only the $K$ learning vehicles are represented as wireless nodes, while background traffic affects vehicle mobility without generating wireless interference. Vehicles that have not departed or are no longer active are kept outside communication range. Transfer outcomes and delays are measured in simulation time, independently of training execution time.

The RSU periodically broadcasts discovery beacons over the simulated IEEE~802.11p channel. A vehicle becomes eligible for model exchange after successfully receiving at least $N_{\mathrm{beacon}}$ beacons within a sliding observation window of duration $T_{\mathrm{obs}}$. This reception-based criterion accounts for wireless channel conditions rather than relying solely on geometric proximity. Contact eligibility, however, does not guarantee successful delivery of a complete model object.

Pending transfers from eligible vehicles are managed through a first-come, first-served (FCFS) queue, with only one active uplink or downlink transfer permitted at a time. The RSU periodically updates contact eligibility and schedules the next eligible request after processing completed attempts. Waiting vehicles that lose contact are removed from the queue, while unfinished transfers may be retried during subsequent contact opportunities before the communication phase ends. Requests with additional transfer work rejoin at the end of the queue after each completed attempt. Packet-level reliability is handled separately by the SR-ARQ mechanism described below. The beacon and scheduling parameters are reported in Appendix~\ref{app:exp_config}.

\begin{algorithm}[t]
\caption{One contact-dependent federated round}
\label{alg:system}
\renewcommand{\algorithmicrequire}{\textbf{Input:}}
\renewcommand{\algorithmicensure}{\textbf{Output:}}
\begin{algorithmic}[1]
\Require Global model $\mathbf W_{\rm global}^{(t)}$,
clients $\mathcal C$, local episodes $E$,
privacy mechanism $\mathcal M$,
signature scheme $\mathcal A$,
network model $\mathcal N$
\Ensure Updated global model
$\mathbf W_{\rm global}^{(t+1)}$

\State \textbf{Authenticated downlink}
\State $\mathcal E_t\gets
\operatorname{AuthenticatedDL}_{\mathcal A,\mathcal N}
(\mathbf W_{\rm global}^{(t)},\mathcal C,t)$

\State \textbf{Local training}
\For{each $i\in\mathcal C$}
  \State $\mathbf W_i^{\rm start}\gets
  \operatorname{SelectModel}(i,\mathcal E_t,
  \mathbf W_{\rm global}^{(t)})$
\EndFor
\State $\{\boldsymbol\theta_i^{(t+1)}\}
\gets\operatorname{TrainSharedSUMO}
(\{\mathbf W_i^{\rm start}\},E)$

\State \textbf{Protected uplink preparation}
\For{each $i\in\mathcal E_t$}
  \State $\mathbf w_i\gets
  \operatorname{Clip}
  (\boldsymbol\theta_i^{(t+1)};
  \mathbf W_{\rm global}^{(t)},r)$
  \State $P_i\gets
  \operatorname{Encode}_{\mathcal M}
  (\mathcal M(\mathbf w_i))$
  \State $U_i\gets
  \operatorname{SignObject}_{\mathcal A}
  (i,t,\mathrm{UL},P_i)$
\EndFor

\State \textbf{Contact-dependent uplink}
\State $\mathcal P_t^{\rm UL}\gets
\operatorname{UplinkDrive}_{\mathcal N}
(\{U_i\}_{i\in\mathcal E_t})$

\State \textbf{Verified aggregation}
\State $\mathcal V_t\gets\emptyset$
\For{each complete $U_i\in\mathcal P_t^{\rm UL}$}
  \If{$i\in\mathcal E_t$ and
  $\operatorname{VerifyContext}_{\mathcal A}
  (U_i,i,t,\mathrm{UL})$}
    \State $\overline{\mathbf w}_i^{(t+1)}
    \gets\operatorname{Reconstruct}_{\mathcal M}
    (P_i;\mathbf W_{\rm global}^{(t)},r)$
    \State $\mathcal V_t\gets\mathcal V_t\cup\{i\}$
  \EndIf
\EndFor
\If{$\mathcal V_t\neq\emptyset$}
  \State Update $\mathbf W_{\rm global}^{(t+1)}$
  using~\eqref{eq:uniform_agg}
\Else
  \State $\mathbf W_{\rm global}^{(t+1)}
  \gets\mathbf W_{\rm global}^{(t)}$
\EndIf
\end{algorithmic}
\end{algorithm}

Authenticated model objects may exceed the payload capacity of a single IEEE~802.11p frame. We therefore use a custom Selective Repeat ARQ (SR-ARQ) layer over UDP to fragment each object into indexed chunks and reassemble it at the receiver. Each chunk carries message identification, sequencing, and CRC-32 integrity metadata. The receiver acknowledges valid chunks, while the sender retransmits unacknowledged chunks using a sliding window and a bounded retry count. An attempt fails when a chunk exhausts its retransmission limit, after which the RSU may schedule another attempt during a subsequent contact opportunity before the communication drive ends. Complete objects undergo signature and context verification after reassembly. The transport configuration and chunk-header format are detailed in Appendix~\ref{app:exp_config} (Table~\ref{tab:sr_params}).

A persistent TCP control connection coordinates the Python co-simulation driver and the ns-3 process. It supports mobility updates, contact queries, asynchronous transfer requests, and synchronized simulation advancement, while reporting transfer outcomes and completion times. Model data is transmitted through the simulated UDP/IEEE~802.11p wireless channel rather than the TCP control connection. The end of each communication drive defines the deadline for pending model exchanges.

The communication mechanisms described above are integrated into the contact-dependent federated round summarized in Algorithm~\ref{alg:system}. Let $\mathcal C$ denote the learning vehicles, $\mathcal E_t\subseteq\mathcal C$ those that successfully receive and authenticate the round-$t$ global model, and $\mathcal V_t\subseteq\mathcal E_t$ those whose current-round updates are delivered and verified. Downlink and uplink exchanges occur in separate SUMO communication drives, during which vehicles follow their routing policies without learning. Between these drives, all vehicles train in a shared SUMO environment using either the authenticated current-round model or their previously held model. Only vehicles in $\mathcal E_t$ are eligible to upload for aggregation.

\section{Experimental Results}
\label{sec:experiments}

\subsection{Experimental Setup and Evaluation Protocol}
\label{subsec:exp_setup}

\begin{figure}[t]
\centering
\includegraphics[width=0.98\textwidth]
{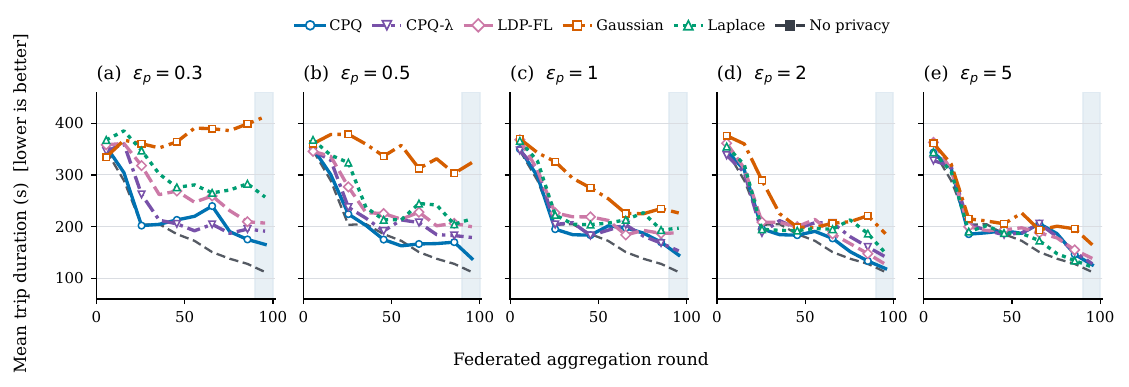}
\caption{Evolution of mean evaluation travel time across
privacy budgets during federated training.}
\label{fig:learning_dynamics}
\end{figure}

The primary experiments use SUMO~1.26 on a road network with 30 directed edges, 50 learning vehicles distributed equally across five starting edges, and a common destination. Background traffic follows randomly generated trips with an average departure interval of 1.75~s. After a 200~s simulation warm-up, learning vehicles depart over a 100~s window, with individual trips limited to 600~s. The road topology and detailed traffic configuration are provided in Appendix~\ref{app:simulation_details} and Appendix~\ref{app:exp_config}.

The primary ideal-communication and mobility-coupled evaluations run for $T=100$ federated rounds with $E=10$ local training episodes per round. Under ideal communication, all learning vehicles receive the current global model and contribute to aggregation in every round. In the mobility-coupled setting, participation depends on successful authenticated model delivery over the simulated V2I links, following Algorithm~\ref{alg:system}.

We compare CPQ, CPQ-$\lambda$ ($\lambda=0.75$), LDP-FL, Gaussian, and Laplace mechanisms, using full-precision FLDQN without privacy perturbation as a reference. The ideal-communication experiments consider $\epsilon_p\in\{0.3,0.5,1,2,5\}$, while the primary mobility-coupled evaluation uses $\epsilon_p=1$. All private mechanisms use a clipping radius of $r=0.02$. The privacy budgets apply to individual model parameters within one communication round. CPQ, CPQ-$\lambda$, LDP-FL, and Laplace satisfy pure $\epsilon_p$-PLDP, whereas Gaussian is calibrated to $(\epsilon_p,\delta_p)$-PLDP with $\delta_p=10^{-5}$.

After each federated round, the global model is evaluated using greedy routing without exploration, further training, or model exchange. The primary evaluations use three fixed SUMO traffic scenarios with 50 learning vehicles each. Routing performance is measured by mean travel time over all evaluated trips, including elapsed time until termination for unsuccessful trips, and arrival rate, defined as the percentage of vehicles reaching their destinations within the trip limit. Late-round results are averaged over rounds 91--100 for each of five independent seeds. We report the mean and sample standard deviation across these run-level averages, using identical evaluation scenarios across the compared mechanisms and runs.

\subsection{Privacy--Utility Under Ideal Communication}
\label{subsec:exp_privacy_utility}

Figure~\ref{fig:learning_dynamics} shows the evolution of routing performance across the evaluated privacy budgets. Mean travel times are computed over nonoverlapping 10-round intervals and averaged across five independent runs, with full-precision FLDQN shown as a no-privacy reference. Under tighter privacy constraints, CPQ achieves lower late-round travel times than the other private mechanisms, while Gaussian perturbation exhibits particularly poor performance at $\epsilon_p=0.3$. Performance differences generally narrow as $\epsilon_p$ increases, although the learning trajectories remain variable and do not follow a consistent ordering across rounds. To avoid selecting favorable training checkpoints, Table~\ref{tab:privacy_utility} reports results over the fixed rounds 91--100.

CPQ achieves lower mean travel time than independent LDP-FL at all five evaluated privacy budgets and ranks first among the private mechanisms at four of them. It also attains the highest mean arrival rate at three budgets, while CPQ-$\lambda$ ($\lambda=0.75$) leads at $\epsilon_p=1$; at $\epsilon_p=5$, Laplace has the lowest travel time and LDP-FL the highest arrival rate. These results indicate favorable average routing utility for CPQ, particularly under tighter privacy constraints. However, the observed differences vary across seeds. The seed-paired analysis in Appendix~\ref{app:detailed_privacy_results} finds that the nominal 95\% confidence interval for the CPQ--LDP-FL travel-time difference excludes zero at $\epsilon_p=1$, but not at the other evaluated budgets.

\begin{table}[t]
\centering
\caption{Privacy--utility performance of federated routing across privacy budgets.}
\label{tab:privacy_utility}
\renewcommand{\arraystretch}{1.16}
\begin{tabular*}{0.99\textwidth}{@{\extracolsep{\fill}}lccccc@{}}
\toprule
\textbf{Mechanism} & $\epsilon_p = 0.3$ & $\epsilon_p = 0.5$ & $\epsilon_p = 1$ & $\epsilon_p = 2$ & $\epsilon_p = 5$ \\
\midrule
\multicolumn{6}{c}{\textbf{A. Mean evaluation travel time (s) $\downarrow$}} \\
\cmidrule(lr){1-6}
CPQ      & \textbf{165.3 $\pm$ 21.4} & \textbf{137.9 $\pm$ 59.7} & \textbf{144.7 $\pm$ 26.5} & \textbf{118.6 $\pm$ 23.1} & \underline{125.7 $\pm$ 24.5} \\
CPQ-$\lambda$ ($\lambda=0.75$) & \underline{191.0 $\pm$ 36.4} & \underline{178.8 $\pm$ 42.7} & \underline{152.8 $\pm$ 17.5} & 141.1 $\pm$ 41.7 & 127.8 $\pm$ 11.9 \\
LDP-FL   & 206.6 $\pm$ 33.2 & 199.5 $\pm$ 21.8 & 187.7 $\pm$ 25.3 & \underline{126.7 $\pm$ 45.5} & 138.3 $\pm$ 38.3 \\
Gaussian & 412.8 $\pm$ 30.5 & 323.8 $\pm$ 56.4 & 226.2 $\pm$ 41.8 & 186.1 $\pm$ 30.9 & 164.6 $\pm$ 43.1 \\
Laplace  & 257.1 $\pm$ 19.7 & 215.1 $\pm$ 19.6 & 196.9 $\pm$ 37.8 & 148.4 $\pm$ 15.9 & \textbf{121.5 $\pm$ 20.0} \\
\cmidrule(lr){1-6}
No privacy & \multicolumn{5}{c}{111.8 $\pm$ 36.7} \\
\midrule
\multicolumn{6}{c}{\textbf{B. Arrival rate (\%) $\uparrow$}} \\
\cmidrule(lr){1-6}
CPQ      & \textbf{78.7 $\pm$ 8.9} & \textbf{84.9 $\pm$ 16.1} & \underline{87.1 $\pm$ 8.5} & \textbf{92.1 $\pm$ 6.7} & 87.8 $\pm$ 10.6 \\
CPQ-$\lambda$ ($\lambda=0.75$) & \underline{75.1 $\pm$ 7.3} & \underline{82.5 $\pm$ 7.8} & \textbf{89.0 $\pm$ 8.1} & 90.3 $\pm$ 5.0 & \underline{90.1 $\pm$ 7.5} \\
LDP-FL   & 71.1 $\pm$ 4.7 & 71.3 $\pm$ 8.0 & 78.3 $\pm$ 6.6 & \underline{91.9 $\pm$ 8.7} & \textbf{90.8 $\pm$ 8.2} \\
Gaussian & 22.8 $\pm$ 7.1 & 45.7 $\pm$ 13.6 & 67.3 $\pm$ 10.8 & 78.5 $\pm$ 5.1 & 82.4 $\pm$ 9.5 \\
Laplace  & 60.7 $\pm$ 6.3 & 68.0 $\pm$ 4.5 & 74.8 $\pm$ 10.5 & 83.6 $\pm$ 5.4 & 89.2 $\pm$ 6.5 \\
\cmidrule(lr){1-6}
No privacy & \multicolumn{5}{c}{89.8 $\pm$ 9.4} \\
\bottomrule
\end{tabular*}
\end{table}

We evaluate symmetric CPQ-$\lambda$ under ideal communication at $\epsilon_p=1$, with $\lambda\in\{0.25,0.5,0.75,0.9,1\}$ and independent LDP-FL. The fully correlated CPQ mechanism corresponds to $\lambda=1$. Empirical aggregation MSE is computed each round as the mean squared difference, across model parameters, between the reconstructed FedAvg model and the average of the clients' clipped models, excluding clipping error. Table~\ref{tab:cpqlambda_exp} reports MSE averaged over 100 rounds, together with routing metrics averaged over rounds 91--100 and ten seeds. Reported uncertainties are 95\% Student-$t$ confidence half-widths. The stronger-observation privacy bounds follow~\eqref{eq:eps_server} and~\eqref{eq:eps_collusion}.

\begin{table}[!t]
\centering
\caption{CPQ-$\lambda$ ablation at $\epsilon_p=1$.}
\label{tab:cpqlambda_exp}
\renewcommand{\arraystretch}{1.15}
\begin{tabular*}{\columnwidth}{@{\extracolsep{\fill}}lccccc@{}}
\toprule
$\lambda$ & $\epsilon_S$ & $\epsilon_C$ &
\shortstack{MSE\\($10^{-5}$)} &
\shortstack{Travel\\time (s)} &
\shortstack{Arrival\\(\%)} \\
\midrule
LDP-FL & 1.00 & 1.00 & $3.717\pm0.003$ & $184\pm27$ & $79\pm6$ \\
$0.25$ & 1.06 & 1.38 & $3.497\pm0.002$ & $149\pm28$ & $87\pm6$ \\
$0.50$ & 1.25 & 1.86 & $2.825\pm0.005$ & $145\pm28$ & $87\pm5$ \\
$0.75$ & 1.75 & 2.63 & $1.708\pm0.007$ & $142\pm16$ & $90\pm5$ \\
$0.90$ & 2.54 & 3.59 & $0.828\pm0.011$ & $153\pm19$ & $84\pm5$ \\
$1$ & $\infty$ & $\infty$ & $0.145\pm0.015$ & $138\pm21$ & $89\pm6$ \\
\bottomrule
\end{tabular*}
\end{table}

Empirical aggregation MSE decreases monotonically with $\lambda$, consistent with the correlation-mixture relationship in~\eqref{eq:cpqlambda_mse}. At $\lambda=0.75$, CPQ-$\lambda$ reduces aggregation MSE by 54.0\% relative to independent LDP-FL while retaining finite pair-view and shared-randomness side-information privacy bounds of $\epsilon_S=1.75$ and $\epsilon_C=2.63$. All intermediate mixing settings achieve lower mean travel times than LDP-FL, although routing performance does not vary monotonically with $\lambda$. The $\lambda=0.75$ configuration provides routing performance close to fully correlated CPQ, with variability across seeds limiting finer comparisons.

We further examine sensitivity to the learning-vehicle population by comparing CPQ and LDP-FL at $\epsilon_p=1$ for $K\in\{10,25,50\}$ under ideal communication. CPQ achieves positive mean travel-time and arrival-rate gains at all three population sizes, with lower travel times in all five matched seeds at $K=25$ and $K=50$. The travel-time confidence intervals exclude zero for these two populations, while greater variability is observed at $K=10$. Because the learning vehicles share the SUMO environment, varying $K$ also changes traffic demand. Detailed results are provided in Appendix~\ref{app:detailed_privacy_results}.

\subsection{Authenticated V2I Communication}
\label{subsec:exp_pqc}

We evaluate the impact of model-update representation on authenticated IEEE~802.11p V2I communication under mobility-dependent connectivity. All mechanisms receive full-precision global-model downlinks and use ML-DSA-44 authentication. CPQ, CPQ-$\lambda$, and LDP-FL transmit bit-packed one-bit updates, whereas Gaussian, Laplace, and the no-privacy baseline use full-precision uplinks. Figure~\ref{fig:communication} compares transmitted traffic, uplink transfer success, and delivery latency.

\begin{figure}[!t]
\centering
\includegraphics[width=0.98\textwidth]{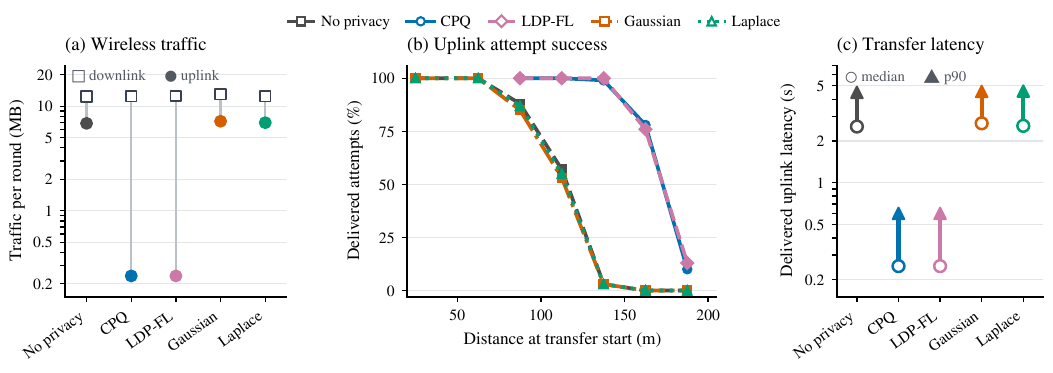}
\caption{Authenticated V2I communication under mobility-dependent connectivity using ML-DSA-44. Panels show wireless traffic, distance-dependent uplink attempt success, and successful-uplink latency.}
\label{fig:communication}
\end{figure}

As shown in Fig.~\ref{fig:communication}, one-bit quantization reduces authenticated uplink size from $113.76$ to $8.65$~KB, a $92.4\%$ reduction. Including retransmissions and unsuccessful attempts, mean uplink traffic decreases by $96.5\%$, while the use of full-precision global-model downlinks limits the overall traffic reduction to approximately $34\%$. Compact updates also improve delivery under intermittent contact. CPQ and LDP-FL achieve $96.3\%$ uplink transfer-attempt success, compared with $32.5\%$ for the full-precision baseline, while median successful-uplink latency decreases from $2.54$ to $0.25$~s. The distance-dependent results further demonstrate the reliability advantage of shorter model transfers as vehicles move away from the RSU.

To examine the effect of post-quantum signature selection on authenticated V2I communication, we compare ML-DSA-44, Falcon-512, and SLH-DSA-SHA2-128f with full-precision no-privacy and one-bit CPQ updates at $\epsilon_p=1$. The experiment uses the same mobility-coupled IEEE~802.11p framework over 30 federated rounds, with one simulation seed per configuration. Table~\ref{tab:e3_signature_sensitivity} reports authenticated uplink size, transfer-attempt success, and overall delivery of initiated uplinks across all rounds.

\begin{table}[!t]
\centering
\caption{Signature-scheme sensitivity of V2I uplinks. FP denotes full-precision model updates.}
\label{tab:e3_signature_sensitivity}
\renewcommand{\arraystretch}{1.15}
\begin{tabular*}{\columnwidth}{@{\extracolsep{\fill}}llccc@{}}
\toprule
Scheme & Update &
\shortstack{UL size\\(KB)} &
\shortstack{Attempt\\success (\%)} &
\shortstack{UL delivery\\(\%)} \\
\midrule
ML-DSA-44 & FP  & 113.60 & 32.2 & 74.7 \\
          & CPQ & 8.65   & 97.0 & 99.8 \\
\midrule
Falcon-512 & FP  & 111.90 & 33.3 & 78.1 \\
           & CPQ & 6.89   & 99.4 & 99.8 \\
\midrule
SLH-DSA-SHA2-128f & FP  & 128.40 & 29.9 & 74.8 \\
                   & CPQ & 23.34  & 48.2 & 93.1 \\
\bottomrule
\end{tabular*}
\end{table}

The effect of signature selection is most pronounced for compact CPQ updates. Falcon-512 produces the smallest authenticated uplink at $6.89$~KB, achieving $99.4\%$ transfer-attempt success and $99.8\%$ overall delivery. In contrast, SLH-DSA-SHA2-128f increases the CPQ uplink to $23.34$~KB, with corresponding rates of $48.2\%$ and $93.1\%$. For full-precision uploads, attempt success remains near $30\%$ across all three schemes, indicating that the larger model payload dominates transmission requirements. These results illustrate the growing importance of authentication overhead when model updates are highly compressed. Additional distance-dependent results are provided in Appendix~\ref{app:e3_signature_network}.

\subsection{Mobility-Constrained Federated Learning}
\label{subsec:exp_network}

We next examine how authenticated, mobility-dependent V2I communication affects federated participation and routing performance at $\epsilon_p=1$. Unlike ideal communication, only vehicles receiving the authenticated current-round global model can contribute updates, and aggregation includes only successfully delivered and verified uploads. Table~\ref{tab:network_results} compares the six mechanisms across matched simulation seeds. Routing metrics are averaged over rounds 91--100, while communication and participation metrics cover all 100 rounds. Values are reported as mean $\pm$ sample standard deviation across five seeds.

\begin{figure}[!t]
\centering
\begin{minipage}[t]{0.66\textwidth}
    \centering
    \includegraphics[width=\linewidth]
    {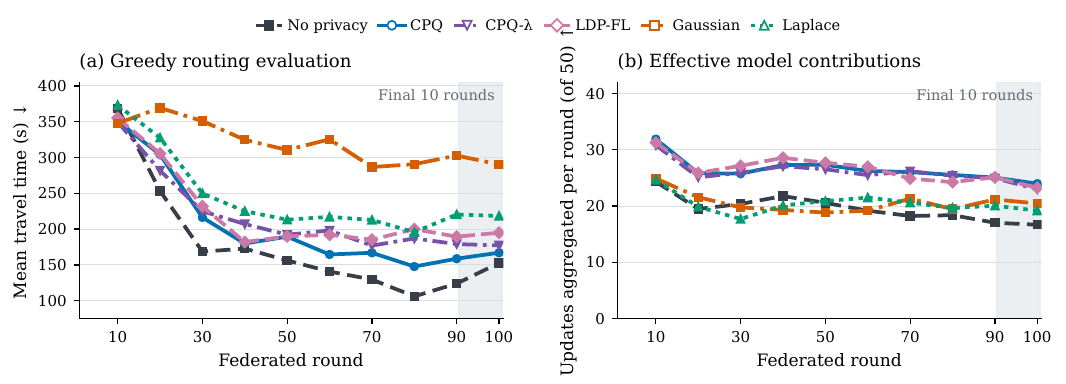}
    \caption{Routing performance and verified aggregation under mobility-constrained V2I communication.}
    \label{fig:mobility_learning}
\end{minipage}
\hfill
\begin{minipage}[t]{0.30\textwidth}
    \centering
    \includegraphics[width=\linewidth]
    {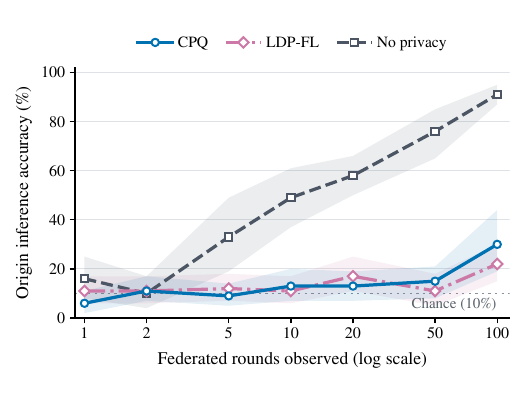}
    \vspace{-10pt}
    \caption{Shadow-run origin inference at $\epsilon_p\!=\!0.3$.}
    \label{fig:origin_inference_shadow}
\end{minipage}
\end{figure}

\begin{table}[!t]
\centering
\caption{Routing performance and participation under mobility-constrained V2I communication.}
\label{tab:network_results}
\renewcommand{\arraystretch}{1.15}
\begin{tabular*}{\textwidth}{@{\extracolsep{\fill}}lcccccc@{}}
\toprule
Method &
\shortstack{Travel\\(s) $\downarrow$} &
\shortstack{Arrival\\(\%) $\uparrow$} &
\shortstack{DL reach\\(\%) $\uparrow$} &
\shortstack{UL delivery\\(\%) $\uparrow$} &
\shortstack{Updates\\per round $\uparrow$} &
\shortstack{$\Delta T$\\(s)} \\
\midrule
No privacy & $152.6\pm20.2$ & $80.4\pm8.1$ & $51.8\pm2.0$ & $75.7\pm1.3$ & $19.6\pm0.9$ & $+40.8\pm56.4$ \\
CPQ & $167.0\pm33.5$ & $79.8\pm8.8$ & $53.3\pm1.1$ & $99.4\pm0.2$ & $26.5\pm0.6$ & $+22.4\pm35.8$ \\
CPQ-$\lambda$ ($\lambda=0.75$) & $177.2\pm41.5$ & $75.8\pm9.2$ & $52.4\pm0.3$ & $99.5\pm0.3$ & $26.1\pm0.2$ & $+24.4\pm54.9$ \\
LDP-FL & $194.9\pm53.0$ & $72.8\pm11.0$ & $53.4\pm1.2$ & $99.4\pm0.3$ & $26.5\pm0.5$ & $+7.2\pm64.7$ \\
Gaussian & $290.3\pm28.3$ & $51.8\pm6.9$ & $55.6\pm1.8$ & $74.0\pm1.1$ & $20.6\pm0.9$ & $+64.1\pm48.7$ \\
Laplace & $218.4\pm33.3$ & $66.4\pm9.0$ & $53.8\pm2.2$ & $75.9\pm2.3$ & $20.4\pm1.4$ & $+21.5\pm50.1$ \\
\bottomrule
\end{tabular*}

\vspace{2pt}
\begin{minipage}{\textwidth}
\footnotesize
\emph{Notes:} UL delivery is conditioned on initiated uploads. $\Delta T$ is the seed-paired V2I minus ideal-communication travel time, with positive values indicating an increase.
\end{minipage}
\end{table}

Mobility-dependent connectivity substantially restricts federated participation, with authenticated downlink reach remaining between $51.8\%$ and $55.6\%$ across the evaluated methods. Among initiated uplinks, CPQ, CPQ-$\lambda$ ($\lambda=0.75$), and LDP-FL achieve $99.4$--$99.5\%$ delivery, compared with $75.7\%$ for the full-precision no-privacy baseline. Consequently, the one-bit mechanisms contribute $26.1$--$26.5$ verified updates per round, approximately $33$--$35\%$ more than the baseline's $19.6$. Figure~\ref{fig:mobility_learning}(b) shows that this participation advantage persists throughout training. These results demonstrate the practical importance of compact uplinks, although full-precision downlink delivery continues to limit aggregation participation.

The resulting routing performance is shown in Fig.~\ref{fig:mobility_learning}(a). Despite nearly identical communication characteristics and aggregation participation, CPQ achieves a $14.3\%$ lower mean travel time than independent LDP-FL ($167.0$ versus $194.9$~s) and a higher arrival rate ($79.8\%$ versus $72.8\%$). CPQ-$\lambda$ ($\lambda=0.75$) achieves intermediate mean travel time ($177.2$~s) and arrival rate ($75.8\%$) while retaining finite stronger-observation privacy bounds. This favorable mean routing trend is consistent with the aggregation-accuracy benefits of correlated quantization, although variability across seeds limits conclusions about the magnitude of the utility improvement. The seed-paired comparisons with ideal communication further show variable increases in travel time across mechanisms, highlighting the combined influence of wireless delivery constraints and privacy perturbation on federated routing.

\subsection{Empirical Vehicle-Origin Inference}
\label{subsec:exp_origin_attack}

We evaluate vehicle-origin inference by an honest-but-curious RSU under mobility-constrained communication. The experiment uses 50 learning vehicles distributed across ten possible starting edges, giving a random-guessing accuracy of $10\%$. At $\epsilon_p=0.3$, the attacker observes successfully received model updates, which can be associated with the same vehicle across rounds, but does not know private trajectories or the CPQ group secret. We consider a known-origin attacker with one labeled vehicle from each class and a shadow-run attacker trained on independent simulations without labeled vehicles in the target deployment. Both attacks accumulate evidence across federated rounds and are evaluated using three independent seeds. Implementation details are provided in Appendix~\ref{app:origin_attack}.

As shown in Fig.~\ref{fig:origin_inference_shadow}, shadow-run inference accuracy reaches $91\%$ after 100 rounds without privacy, compared with $30\%$ for CPQ and $22\%$ for LDP-FL. In the known-origin setting, the corresponding accuracies are $69\%$, $13\%$, and $12\%$. Both private mechanisms substantially reduce the effectiveness of the evaluated attacks, although the shadow-run results remain above chance after prolonged observation. These findings provide application-level evidence of reduced vehicle-origin leakage and complement the single-upload privacy guarantees established in Section~\ref{sec:dp}.

\section{Conclusion}
\label{sec:conclusion}

This paper presented a privacy-preserving and post-quantum-authenticated framework for vehicular federated reinforcement learning under mobility-constrained V2I communication. CPQ reduces one-bit aggregation error through correlated randomness while preserving parameter-level local differential privacy, and CPQ-$\lambda$ extends the framework with stronger-observation and vector-level privacy guarantees. Theoretical analysis and experiments demonstrate the aggregation-accuracy benefits of controlled correlation, while mobility-coupled evaluations show that compact authenticated updates improve reliable delivery and federated participation while supporting favorable routing performance. Vehicle-origin inference experiments further demonstrate reduced empirical privacy leakage. These findings highlight the importance of jointly considering privacy, authentication, and communication efficiency in vehicular learning. Future work will explore more efficient global-model distribution and coordination across multiple RSUs.

\section*{Code and Data Availability}
The source code is publicly available at
\url{https://github.com/solidlabnetwork/cpq-vehicular-frl}.
The experimental data are archived on Zenodo at
\url{https://doi.org/10.5281/zenodo.23693204}.

\section*{Acknowledgments}
This work is based upon the work partly supported by the National Center for Transportation Cybersecurity and Resiliency (TraCR) (a U.S. Department of Transportation National University Transportation Center) headquartered at Clemson University, Clemson, South Carolina, USA. Any opinions, findings, conclusions, and recommendations expressed in this material are those of the author(s) and do not necessarily reflect the views of TraCR, and the U.S. Government assumes no liability for the contents or use thereof.

\bibliography{references}

\clearpage
\appendix
\section*{Appendices}

\counterwithin{equation}{section}
\counterwithin{figure}{section}
\counterwithin{table}{section}
\counterwithin{algorithm}{section}

\section{One-Bit Specializations of Related Quantization Mechanisms}
\label{app:one_bit_equivalence}

This appendix examines the relationship between LDP-FL
and several established one-bit randomization mechanisms.
We identify the conditions under which these mechanisms
induce the same scalar output distribution and reconstruction
levels as LDP-FL, and distinguish these cases from
mechanisms that share only the transmitted bit distribution
or employ different randomization rules.

Consider a scalar parameter $w\in[c-r,c+r]$,
with public clipping center $c$ and radius $r>0$.
Let $x=w-c\in[-r,r]$ denote the centered input,
and define $\alpha_p=\alpha(\epsilon_p)$ as
in~\eqref{eq:alpha}. For mechanisms originally
defined on a zero-centered interval, we apply
the input transformation $x=w-c$ and shift the
reconstructed output by $c$.

The reference LDP-FL quantizer generates
$B\in\{0,1\}$ according to
\begin{equation}
\Pr(B=1\mid w)
=
\frac12+\frac{x}{2r\alpha_p},
\label{eq:equiv_reference}
\end{equation}
and reconstructs
$W=c+r\alpha_p(2B-1)$.

Two mechanisms are estimator-equivalent when,
after the corresponding input transformation,
they have the same binary response probabilities
and reconstruction levels for every admissible
input. Channel equivalence requires only the
same conditional distribution of transmitted
bits.

Table~\ref{tab:one_bit_equivalence} summarizes
the conditions under which established one-bit
mechanisms coincide with LDP-FL and identifies
related mechanisms that are not generally
estimator-equivalent.

\begin{table}[!b]
\centering
\caption{One-bit mechanisms equivalent to LDP-FL under specified configurations.}
\label{tab:one_bit_equivalence}
\renewcommand{\arraystretch}{1.17}

\begin{tabular*}{\textwidth}
{@{\extracolsep{\fill}}
p{0.25\textwidth}
p{0.70\textwidth}
@{}}

\toprule
\textbf{Mechanism}
& \textbf{One-bit configuration and equivalence conditions} \\
\midrule

\citet{duchi2018minimax}
& Scalar case $d=1$, with the input mapped
  to the common clipping interval $[c-r,c+r]$. \\

1BitMean~\citep{ding2017collecting}
& Counter range $[0,m]$, with $m=2r$
  and input mapping $u=x+r$. \\

PBM~\citep{chen2022poisson}
& One Bernoulli trial ($m=1$), $C=r$,
  and $\theta=\frac12\tanh(\epsilon_p/2)$.
  The original constraint $\theta\leq1/4$
  requires $\epsilon_p\leq\ln 3$. \\

RQM~\citep{youn2023randomized}
& Two output levels ($m=2$), $C=r$,
  and $\Delta=2r/(e^{\epsilon_p}-1)$,
  using exact local-privacy calibration. \\

MVU~\citep{chaudhuri2022privacy}
& One output bit, pure-$\epsilon_p$-LDP,
  hard unbiasedness, and exact optimization
  on a grid containing the clipping endpoints.
  The unique optimum yields the LDP-FL
  reconstruction levels. \\

$\beta$-StoSign~\citep{xiang2023distributed}
& Centered input $x=w-c$, clipping bound $r$,
  and $\beta=2r/(e^{\epsilon_p}-1)$.
  Reconstruction uses $c\pm(r+\beta)$,
  followed by arithmetic-mean aggregation. \\

z-SignFedAvg~\citep{tang2024zsignfedavg}
& Uniform-noise limit $z\to\infty$,
  centered input $x=w-c$, noise half-width
  $\sigma=r\alpha(\epsilon_p)$,
  and reconstruction $c\pm\sigma$. \\

\bottomrule
\end{tabular*}
\end{table}


\noindent\textbf{Fixed-level binary characterization.}
The following result recalls the scalar optimality property
used in Section~\ref{subsec:ldpfl} and makes the assumptions
needed for the equivalence mappings explicit; see also
\citep[Proposition~1]{salehi2024corbin}.

\begin{proposition}[Unbiased fixed-level one-bit quantization]
\label{prop:onebit_fixed_level}
Let $w\in[c-r,c+r]$, $r>0$, and $\epsilon_p>0$. Consider a
binary mechanism with $B\in\{0,1\}$ and fixed public
reconstruction levels $a_0=c-R_0$ and $a_1=c+R_1$,
where $a_0<a_1$. If $\mathbb{E}[a_B\mid w]=w$ for every
admissible $w$, then
\begin{equation}
\Pr(B=1\mid w)=\frac{w-c+R_0}{R_0+R_1}.
\label{eq:onebit_affine}
\end{equation}
This mechanism satisfies pure $\epsilon_p$-PLDP if and only if
$R_0,R_1\ge r\alpha(\epsilon_p)$. Its mean squared error is
\begin{equation}
\mathbb{E}[(a_B-w)^2\mid w]
=\bigl(R_0+w-c\bigr)\bigl(R_1-w+c\bigr).
\label{eq:onebit_general_mse}
\end{equation}
The unique pointwise MSE-minimizing choice is
$R_0=R_1=r\alpha(\epsilon_p)$, which yields the LDP-FL
response and reconstruction levels in
\eqref{eq:alpha}--\eqref{eq:ldpfl_output}.
\end{proposition}

\begin{proof}
Writing $p(w)=\Pr(B=1\mid w)$, unbiasedness gives
$c-R_0+(R_0+R_1)p(w)=w$, establishing
\eqref{eq:onebit_affine}. Since $p(w)\in[0,1]$ over
the clipping interval, $R_0,R_1\ge r$.
The probability is affine and increasing, so the two
binary-output privacy constraints attain their extreme
ratios at the clipping endpoints:
\begin{equation}
\epsilon_M=\max\!\left\{
\log\frac{R_0+r}{R_0-r},\,
\log\frac{R_1+r}{R_1-r}
\right\}.
\label{eq:onebit_exact_privacy}
\end{equation}
For finite $\epsilon_p$, the denominators are positive.
The requirement $\epsilon_M\le\epsilon_p$ is equivalent to
$R_0,R_1\ge r(e^{\epsilon_p}+1)/(e^{\epsilon_p}-1)$.

As the estimator is unbiased, its MSE equals its variance,
$(R_0+R_1)^2p(w)(1-p(w))$, which simplifies to
\eqref{eq:onebit_general_mse}. For
$x=w-c\in[-r,r]$, this expression is
$(R_0+x)(R_1-x)$. Both factors are strictly
positive and strictly increasing in their respective
reconstruction distances throughout the feasible set.
Its unique minimum is therefore at
$R_0=R_1=r\alpha(\epsilon_p)$, giving the LDP-FL
channel and levels. 
\end{proof}


\noindent\textbf{Mechanism-specific equivalences.}
We now verify the rows of Table~\ref{tab:one_bit_equivalence}.
In each case let $x=w-c\in[-r,r]$,
$E=e^{\epsilon_p}$, and $\alpha_p=(E+1)/(E-1)$.
An equivalent scalar estimator must satisfy
\begin{equation}
\Pr(B=1\mid x)=\frac12+\frac{x}{2r\alpha_p},
\qquad a_B=c+r\alpha_p(2B-1).
\label{eq:onebit_matching_test}
\end{equation}
All matching statements below refer to these conditional
output distributions; aggregation equivalence additionally
assumes independent client randomization and identical
arithmetic-mean reconstruction.

\noindent\emph{1) Duchi et al.}
In the scalar $d=1$ case, first stochastically round $x$ to
$S\in\{0,1\}$ with $\Pr(S=1)=(x+r)/(2r)$, then apply
binary randomized response, retaining $S$ with probability
$E/(E+1)$. Thus
\begin{equation}
\Pr(B=1\mid x)
=\frac{1}{E+1}+\frac{E-1}{E+1}\frac{x+r}{2r}
=\frac12+\frac{x}{2r\alpha_p}.
\end{equation}
The scalar output levels $c\pm r\alpha_p$ establish
estimator equivalence~\citep{duchi2018minimax}.

\noindent\emph{2) 1BitMean.}
For the telemetry interval $[0,m]$, set $m=2r$ and
$u=x+r$. Its response probability is
\begin{equation}
\Pr(B=1\mid u)=\frac{1}{E+1}
+\frac{u}{m}\frac{E-1}{E+1}
=\frac12+\frac{x}{2r\alpha_p}.
\end{equation}
Shifting its decoded estimate by $c-r$ gives
$c-r+2r((E+1)B-1)/(E-1)=c+r\alpha_p(2B-1)$,
proving estimator equivalence~\citep{ding2017collecting}.

\noindent\emph{3) PBM.}
For a single Bernoulli trial ($m=1$), clipping bound $C=r$,
and $\theta=\tfrac12\tanh(\epsilon_p/2)=1/(2\alpha_p)$,
PBM produces
\begin{equation}
\Pr(B=1\mid x)=\frac12+\frac{\theta x}{r}
=\frac12+\frac{x}{2r\alpha_p}.
\end{equation}
The reconstructed value $c+(r/\theta)(B-1/2)$ equals
$c+r\alpha_p(2B-1)$. Under the original constraint
$\theta\le 1/4$, this calibration is admissible for
$\epsilon_p\le\log 3$. This identity concerns the one-trial
scalar randomizer observed locally; the original PBM also
studies secure aggregation~\citep{chen2022poisson}.

\noindent\emph{4) RQM.}
With two quantization levels ($m=2$), the range-extension
parameter $\Delta=2r/(E-1)$ gives
$X_{\max}=r+\Delta=r\alpha_p$. The upper-level
probability and decoded value satisfy
\begin{equation}
\begin{aligned}
\Pr(B=1\mid x)&=\frac{x+X_{\max}}{2X_{\max}}
=\frac12+\frac{x}{2r\alpha_p},\\
a_B&=c+X_{\max}(2B-1).
\end{aligned}
\end{equation}
Thus two-level RQM with this exact local-privacy calibration
is estimator-equivalent to LDP-FL~\citep{youn2023randomized}.

\noindent\emph{5) MVU.}
Consider the exactly solved one-bit MVU optimization with
pure $\epsilon_p$-LDP and hard unbiasedness constraints on
an input grid containing both clipping endpoints. Since
there are only two fixed output levels, unbiasedness forces
the affine probabilities in~\eqref{eq:onebit_affine} at every
grid point. The privacy constraints on the grid are tightest
at its endpoints, so Proposition~\ref{prop:onebit_fixed_level}
applies.

Every grid-point MSE is uniquely minimized at
$R_0=R_1=r\alpha_p$, which is also the unique optimizer
of their sum. Unbiased stochastic interpolation between grid
points preserves the affine response, establishing the scalar
LDP-FL estimator over the full interval
\citep{chaudhuri2022privacy}.

\noindent\emph{6) $\beta$-StoSign.}
For centered input $x$, clipping bound $r$, and
$\beta=2r/(E-1)$, the probability of a positive sign is
\begin{equation}
\Pr(S=+1\mid x)=\frac{r+\beta+x}{2(r+\beta)}
=\frac12+\frac{x}{2r\alpha_p},
\end{equation}
because $r+\beta=r\alpha_p$. With matched scalar decoding
$a_S=c+(r+\beta)S$ and arithmetic-mean aggregation,
this yields the LDP-FL estimator. The original sign-based
aggregation rule is a separate algorithmic choice
\citep{xiang2023distributed}.

\noindent\emph{7) Uniform-noise sign quantization.}
In the uniform-noise limit of z-SignFedAvg, let
$U\sim\mathrm{Uniform}[-1,1]$ be locally generated and set
$\sigma=r\alpha_p$. Since $|x|\le r<\sigma$,
\begin{equation}
\Pr(B=1\mid x)
=\Pr(x+\sigma U>0)
=\frac12+\frac{x}{2\sigma}
=\frac12+\frac{x}{2r\alpha_p}.
\end{equation}
Reconstruction as $a_B=c+\sigma(2B-1)$ gives precisely
the LDP-FL scalar estimator~\citep{tang2024zsignfedavg}.


\noindent\textbf{Implication for the independent baseline.}
For any fixed collection of clipped client values, the seven
specified scalar configurations generate the same conditional
distribution of reconstructed values as LDP-FL. If all clients
use independent private randomness and the same arithmetic-mean
aggregation, their aggregate estimators have the same
conditional distribution and MSE as~\eqref{eq:ldpfl_mse}.
The CPQ comparison therefore extends to these matched
independent one-bit scalar configurations. This conclusion
uses the stated scalar channel, decoding, clipping, and
aggregation assumptions.

\section{Proof of Theorem~\ref{thm:cpq}}
\label{app:proof_cpq}

\noindent\textbf{(i) PLDP.}
Since $1-V\sim\mathrm{Uniform}[0,1]$, both $Q_k$ satisfy
$P(Q_k=c+r\alpha)=q(w_k)$, matching the LDP-FL output distribution. The PLDP ratio for the upper
output is
\[
  \frac{q(w)}{q(w')} \leq \frac{q(c+r)}{q(c-r)}
  = \frac{\tfrac{1}{2}+\tfrac{1}{2\alpha}}{\tfrac{1}{2}-\tfrac{1}{2\alpha}}
  = \frac{\alpha+1}{\alpha-1} = e^{\epsilon_p}.
\]
By symmetry, the same bound holds for the lower (complementary) output:
\[
  \frac{1-q(w)}{1-q(w')} \leq \frac{1-q(c-r)}{1-q(c+r)}
  = \frac{\tfrac{1}{2}+\tfrac{1}{2\alpha}}{\tfrac{1}{2}-\tfrac{1}{2\alpha}}
  = \frac{\alpha+1}{\alpha-1} = e^{\epsilon_p}.
\]
Since the output alphabet is binary, bounding the ratio at both atoms bounds it for every
measurable output set $\mathcal{T}\subseteq\mathbb{R}$, establishing $\epsilon_p$-PLDP.

\noindent\textbf{(i) Unbiasedness.}
$\mathbb{E}[Q_k(w_k,V)] = c + r\alpha(2q(w_k)-1) = w_k$, so
$\mathbb{E}[\overline{W}_{g,j}]
= c_j + \frac{r_j\alpha}{n}\sum_i(2q(w_{i,j})-1)
= \frac{1}{n}\sum_i w_{i,j}$.

\noindent\textbf{(ii) Pairwise MSE optimality.}
Let $U_k=(Q_k-c)/(r\alpha)\in\{-1,+1\}$ and $x_k=(w_k-c)/(r\alpha)$.
Expanding the MSE of the pair-average estimator
$\tfrac{1}{2}(Q_1+Q_2)$ (hence the factor $\tfrac{r^2\alpha^2}{4}$),

\begin{align}   \label{eq:mse_expanded}
      m(Q_1,Q_2,w_1,w_2)
  =&\frac{r^2\alpha^2}{4}\bigl[
    (1-x_1^2)+(1-x_2^2)+2(\mathbb{E}[U_1U_2]-x_1x_2)\bigr].
\end{align}

The terms $(1-x_k^2)$ are fixed by the marginals, so minimizing
\eqref{eq:mse_expanded} is equivalent to minimising $\mathbb{E}[U_1U_2]$.

\noindent\textit{Fr\'echet--Hoeffding lower bound.}
For binary $U_k\in\{-1,+1\}$ with $P(U_k=+1)=q_k=(1+x_k)/2$,
$\mathbb{E}[U_1U_2]=1-2P(U_1\neq U_2)$.
Maximizing $P(U_1\neq U_2)$ over all couplings gives
\begin{align}
    P^*(U_1\neq U_2) &= \min(q_1,1-q_2)+\min(1-q_1,q_2)\nonumber = 1-|q_1+q_2-1|,
\end{align}

so the minimum cross-term is
$\mathbb{E}_{\min}[U_1U_2] = 2|q_1+q_2-1|-1 = |x_1+x_2|-1$.

\noindent\textit{CPQ achieves this bound.}
Under CPQ, $U_1=2\cdot\mathbf{1}_{[V<q_1]}-1$ and
$U_2=2\cdot\mathbf{1}_{[1-V<q_2]}-1$. Computing the joint probabilities:
\begin{align*}
  P(U_1=+1,U_2=+1) &= P(V<q_1,\;V>1-q_2) = \max(q_1+q_2-1,\,0),\\
  P(U_1=-1,U_2=-1) &= P(V\geq q_1,\;V\leq 1-q_2) = \max(1-q_1-q_2,\,0).
\end{align*}
Therefore $P(U_1\neq U_2) = 1-|q_1+q_2-1|$, so $\mathbb{E}[U_1U_2]=|x_1+x_2|-1$,
exactly matching the Fr\'echet lower bound. Substituting into
\eqref{eq:mse_expanded},
\begin{align*}
  m^* &= \frac{r^2\alpha^2}{4}
    \bigl[2-x_1^2-x_2^2 + 2(|x_1+x_2|-1) - 2x_1x_2\bigr] = \frac{r^2\alpha^2}{2}\!\left[|x_1+x_2|-\frac{(x_1+x_2)^2}{2}\right],
\end{align*}
which is \eqref{eq:cpq_mse_opt}.

\noindent\textit{Comparison with LDP-FL.}
For independent quantizers, $\mathbb{E}[U_1U_2]=x_1x_2$, so
\[
  m_{\mathrm{LDP}} - m^*
  = \frac{r^2\alpha^2}{2}\bigl(1+x_1x_2-|x_1+x_2|\bigr) \geq 0,
\]
where the inequality follows from $1+x_1x_2-|x_1+x_2|=(1-x_1)(1-x_2)\ge 0$
when $x_1+x_2\ge 0$ and $=(1+x_1)(1+x_2)\ge 0$ when $x_1+x_2<0$. Equality therefore holds
if and only if one marginal is deterministic ($x_1\in\{-1,+1\}$ or $x_2\in\{-1,+1\}$).
Since $|x_k|=|w_k-c|/(r\alpha)\le 1/\alpha<1$ for every $w_k\in[c-r,c+r]$, this never
occurs, so the improvement is strict for all clipped weights. This establishes the strict MSE advantage over independent quantization. Finally, writing $s=x_1+x_2$,
\eqref{eq:cpq_mse_opt} gives
$m^*=\frac{r^2\alpha^2}{2}\,|s|\bigl(1-\tfrac{|s|}{2}\bigr)$, which is zero if and only if
$s=0$, since $0<|s|\le 2$ makes both factors positive. Hence the paired noises cancel
exactly precisely when $w_1+w_2=2c$.
\hfill$\square$

\section{Proof of Theorem~\ref{thm:cpqlambda}}
\label{app:proof_cpqlambda}
Write $q_k=q(w_k)$, $p_2=1-q_2$,
$\alpha=\alpha(\epsilon_p)$, and
$\mu=\lambda_1\lambda_2$.
Define the endpoint probabilities
\[
  q_-=\frac{1}{e^{\epsilon_p}+1},
  \qquad
  q_+=\frac{e^{\epsilon_p}}{e^{\epsilon_p}+1}.
\]
Then $q(w)\in[q_-,q_+]$,
$q_-+q_+=1$, and $q_+/q_-=e^{\epsilon_p}$.

\noindent\textbf{(a) Marginal PLDP and unbiasedness.}
The mixture representation
in~\eqref{eq:cpqlambda_mixture}
follows by conditioning on the independent
mixing coins $(C_1,C_2)$.
Both mixture components have the same
marginals $\Pr(b_k=1)=q_k$. Hence each bit has exactly the LDP-FL output law, which satisfies $\epsilon_p$-PLDP~\citep{sun2020ldp}. The reconstruction expectation is therefore unchanged:
\[
 \mathbb{E}[\overline{W}_{g,j}]
 =c_j+\frac{r_j\alpha}{n}\sum_i(2q(w_{i,j})-1)
 =\frac{1}{n}\sum_iw_{i,j}.
\]

\noindent\textbf{(b) Pair-view privacy.}
The output alphabet is finite, so it is sufficient to bound the likelihood ratio at each atom. From the mixture representation
in~\eqref{eq:cpqlambda_mixture},
\begin{align}
 P_{11}(q_1,q_2)
 &=\mu\max(q_1+q_2-1,0)+(1-\mu)q_1q_2,
 \label{eq:cpqlambda_p11}\\
 P_{10}(q_1,q_2)
 &=\mu\min(q_1,1-q_2)+(1-\mu)q_1(1-q_2).
 \label{eq:cpqlambda_p10}
\end{align}
The remaining atoms follow from complementation. In particular, with the reflection $R(q_1,q_2)=(1-q_1,1-q_2)$,
\[
 P_{00}=P_{11}\circ R,\qquad P_{01}=P_{10}\circ R.
\]
Since $q_-+q_+=1$, $R$ maps $[q_-,q_+]^2$ onto itself. Therefore the worst ratios for $(0,0)$ and $(0,1)$ equal those for $(1,1)$ and $(1,0)$, respectively.

For $(1,1)$, $P_{11}$ is nondecreasing in $q_1$ and is strictly increasing because $1-\mu>0$. Thus the largest ratio over $q_1,q_1'$ is obtained at $q_1=q_+$ and $q_1'=q_-$. Since $q_-+q_2-1\le0$ and $q_++q_2-1=q_2-q_-\ge0$,
\begin{align*}
 \frac{P_{11}(q_+,q_2)}{P_{11}(q_-,q_2)}
 &=\frac{\mu(q_2-q_-)+(1-\mu)q_+q_2}
 {(1-\mu)q_-q_2}=e^{\epsilon_p}+\frac{\mu}{1-\mu}
 \left(\frac{1}{q_-}-\frac{1}{q_2}\right).
\end{align*}
The right-hand side is increasing in $q_2$, hence it is maximized at $q_2=q_+$. Using
\[
 \frac{1}{q_-}-\frac{1}{q_+}
 =(e^{\epsilon_p}+1)(1-e^{-\epsilon_p})
 =e^{\epsilon_p}-e^{-\epsilon_p},
\]
the maximum equals $e^{\epsilon_S}$ in~\eqref{eq:eps_server}.

For $(1,0)$, $P_{10}$ is also nondecreasing in $q_1$. Since $p_2\in[q_-,q_+]$, $\min(q_+,p_2)=p_2$ and $\min(q_-,p_2)=q_-$. Therefore
\[
 \frac{P_{10}(q_+,q_2)}{P_{10}(q_-,q_2)}
 =\frac{p_2[\mu+(1-\mu)q_+]}
 {q_-[\mu+(1-\mu)p_2]}.
\]
The map $p\mapsto p/[\mu+(1-\mu)p]$ is increasing, so the ratio is maximized at $p_2=q_+$, where it equals $q_+/q_-=e^{\epsilon_p}\le e^{\epsilon_S}$. Reflection handles the other two atoms. Hence every atom, and therefore every output set, satisfies~\eqref{eq:pairview_pldp}. The $(1,1)$ bound is attained at $(q_1,q_1',q_2)=(q_+,q_-,q_+)$, proving tightness.

\noindent\textbf{(b, continued) Privacy under side information.}
Condition on $Z_1=(V,C_2,R_2,w_2)$. The only remaining randomness in $b_1$ is $(C_1,R_1)$, which is independent of $Z_1$. For $q=q(w_1)$,
\begin{equation}
 g_V(q):=\Pr(b_1=1\mid w_1,Z_1)
 =\lambda_1\mathbf{1}_{\{V<q\}}+(1-\lambda_1)q.
 \label{eq:gV}
\end{equation}
The function is nondecreasing in $q$, so the largest ratio for the output $b_1=1$ is $g_V(q_+)/g_V(q_-)$. If $V\in[q_-,q_+)$, the indicator is zero at $q_-$ and one at $q_+$, giving
\begin{align*}
 \frac{g_V(q_+)}{g_V(q_-)}
 &=\frac{\lambda_1+(1-\lambda_1)q_+}{(1-\lambda_1)q_-}=\frac{e^{\epsilon_p}+\lambda_1}{1-\lambda_1}.
\end{align*}
When $V<q_-$ or $V\ge q_+$ the ratio is no larger. For the output $b_1=0$, the probability $1-g_V(q)$ is nonincreasing in $q$, and the corresponding endpoint ratio is maximized over the same interval $V\in[q_-,q_+)$, producing the identical value. Since $q_+-q_->0$, this event has positive probability, so the bound is tight. The follower result follows by replacing $V$ with $1-V$.

\noindent\textbf{(c) Pairwise MSE.}
For fixed inputs, squared pair-average error is a deterministic function of the output pair. Taking expectation under the mixture~\eqref{eq:cpqlambda_mixture} therefore gives
\[
 m_{\mu}=(1-\mu)m_{\mathrm{IND}}+\mu m^*.
\]
Using the independent cross-term $\mathbb{E}[U_1U_2]=x_1x_2$ in~\eqref{eq:mse_expanded} gives
\[
 m_{\mathrm{IND}}=\frac{r^2\alpha^2}{4}(2-x_1^2-x_2^2).
\]
The difference derived in Appendix~\ref{app:proof_cpq} is
\[
 m_{\mathrm{IND}}-m^*
 =\frac{r^2\alpha^2}{2}
 \left(1+x_1x_2-|x_1+x_2|\right).
\]
Substitution establishes~\eqref{eq:cpqlambda_mse}. As shown in the proof of Theorem~\ref{thm:cpq}, the bracket is strictly positive for all admissible clipped inputs because $|x_k|\le1/\alpha<1$. Hence every $\mu>0$ yields a strict improvement over independent quantization.

\noindent\textbf{(d) Population MSE.}
Let $S_\ell-s_\ell$ denote the pair-sum error for pair $\ell$. Pair errors are unbiased and independent across distinct pairs, so
\[
 \mathbb{E}[(\overline{W}_{g,j}-\bar w_j)^2]
 =\frac{1}{n^2}\sum_{\ell=1}^{n/2}
 \mathbb{E}[(S_\ell-s_\ell)^2].
\]
By part~(c), the pair-sum error is the mixture of the independent and CPQ pair-sum errors. The former satisfies
\[
 e_{\mathrm{IND}}
 =r_j^2\alpha^2(2-x_1^2-x_2^2)
 \le2r_j^2\alpha^2,
\]
while Corollary~\ref{cor:cpq_pop} gives $e_{\mathrm{CPQ}}\le2r_j^2\kappa$. Hence
\begin{align*}
 \mathbb{E}[(\overline{W}_{g,j}-\bar w_j)^2]
 &\le\frac{2r_j^2}{n^2}\sum_{\ell=1}^{n/2}
 \left[(1-\mu_\ell)\alpha^2+\mu_\ell\kappa\right]=\frac{r_j^2}{n}
 \left[(1-\bar\mu)\alpha^2+\bar\mu\kappa\right],
\end{align*}
which proves~\eqref{eq:cpqlambda_pop}. \hfill$\square$

\section{Proof of Theorem~\ref{thm:cpqlambda_vector}}
\label{app:cpqlambda_vector_proof}

\noindent\textbf{Joint output distribution.}
Consider one coordinate of a CPQ-$\lambda$ pair, with target
and partner bit-one probabilities $u,z\in[q_-,q_+]$,
respectively, where
\[
q_-=\frac{1}{1+e^{\epsilon_p}},
\qquad
q_+=1-q_-,
\qquad
\mu=\lambda^2.
\]
With probability $\mu$, both clients use complementary shared
randomness. Otherwise, their quantization outputs are
independent. The joint distribution is therefore
\begin{align}
P_{u,z}(1,1)
&=(1-\mu)uz+\mu(u+z-1)_+,\nonumber\\
P_{u,z}(1,0)
&=u-P_{u,z}(1,1),\nonumber\\
P_{u,z}(0,1)
&=z-P_{u,z}(1,1),\nonumber\\
P_{u,z}(0,0)
&=1-u-z+P_{u,z}(1,1),
\label{eq:vector_pair_distribution}
\end{align}
where $(x)_+=\max\{x,0\}$.

\noindent\textbf{Worst-case R\'enyi divergence.}
For a R\'enyi order $\nu>1$, define
\[
H_\nu(P,Q)
=\sum_x P(x)^\nu Q(x)^{1-\nu},
\]
so that
\[
D_\nu(P\|Q)
=\frac{\ln H_\nu(P,Q)}{\nu-1}.
\]
We maximize this divergence over alternative target
probabilities $u,v\in[q_-,q_+]$, while keeping the partner
probability $z$ fixed.

The partner bit follows $\mathrm{Bernoulli}(z)$ under both
target inputs. Conditional on the partner's bit, the target
bit is Bernoulli with success probabilities
\begin{equation}
\begin{aligned}
p_1(u)&=\frac{P_{u,z}(1,1)}{z},\\
p_0(u)&=\frac{P_{u,z}(1,0)}{1-z}.
\end{aligned}
\label{eq:vector_conditional_bits}
\end{equation}
Both probabilities are nondecreasing in $u$, as follows
from~\eqref{eq:vector_pair_distribution}.

For Bernoulli distributions with success probabilities
$p\leq q$, the R\'enyi power sum
\[
h_\nu(p,q)
=p^\nu q^{1-\nu}
+(1-p)^\nu(1-q)^{1-\nu}
\]
is nonincreasing in $p$ and nondecreasing in $q$.
Therefore, increasing the separation between the
alternative target probabilities cannot decrease the
corresponding divergence. Since the partner-bit
distribution remains unchanged, the maximum over
$u$ and $v$ is attained at
\[
(u,v)\in
\{(q_-,q_+),(q_+,q_-)\}.
\]

For either endpoint ordering, the joint probabilities
in~\eqref{eq:vector_pair_distribution} are affine
functions of $z\in[q_-,q_+]$. In particular,
\[
(q_-+z-1)_+=0,
\qquad
(q_++z-1)_+=z-q_-.
\]
The R\'enyi power sum $H_\nu(P,Q)$ is jointly convex
in its distribution arguments for $\nu>1$.
Consequently, its maximum over the resulting affine
interval occurs at $z\in\{q_-,q_+\}$.

Thus,
\begin{equation}
\begin{aligned}
&\sup_{u,v,z\in[q_-,q_+]}
D_\nu(P_{u,z}\|P_{v,z})=
\max_{\substack{
(u,v)\in\{(q_-,q_+),(q_+,q_-)\}\\
z\in\{q_-,q_+\}}}
D_\nu(P_{u,z}\|P_{v,z}),
\end{aligned}
\label{eq:vector_rdp_endpoints}
\end{equation}
which establishes the exact worst-case per-coordinate
divergence in~\eqref{eq:cpqlambda_rdp}.

\noindent\textbf{Vector-level composition.}
Fix the public round state and all other clients'
clipped inputs. Under independent ideal randomization
across coordinates, the joint observation distributions
factorize over the $m$ model coordinates.

Outputs of pairs not containing the target client have
identical distributions under the two alternative
target inputs and contribute zero R\'enyi divergence.
Therefore,
\begin{equation}
D_\nu\!\left(
\mathcal{M}_{\mathrm{obs}}(\mathbf{w}_i)
\middle\|
\mathcal{M}_{\mathrm{obs}}(\mathbf{w}'_i)
\right)
\leq m\,d_\nu.
\label{eq:vector_rdp_composition}
\end{equation}

The conversion from $(\nu,\rho)$-RDP to approximate
differential privacy gives the sufficient vector-level
privacy certificate
\[
\epsilon_v
=
\inf_{\nu>1}
\left[
m\,d_\nu+
\frac{\ln(1/\delta_v)}{\nu-1}
\right],
\]
which establishes
\eqref{eq:cpqlambda_vector_privacy}.

Because the pairing assignments are independent of the
protected inputs, hiding them is a post-processing of
the stronger pairing-revealed observation. The same
privacy guarantee therefore applies to the RSU observer
when the pairings remain hidden.
\hfill$\square$

\noindent\textbf{Calibration of comparison mechanisms.}
For the matched vector-level privacy comparison in
Table~\ref{tab:vector_privacy_comparison}, we use
the exact per-coordinate R\'enyi divergences of the
independent baseline mechanisms.

For LDP-FL with parameter-level privacy budget $s$,
\begin{equation}
d_\nu^{\mathrm{LDP}}(s)
=
\frac{1}{\nu-1}
\ln\frac{
e^{\nu s}+e^{(1-\nu)s}
}{1+e^s}.
\label{eq:vector_ldp_rdp}
\end{equation}

For coordinatewise Laplace noise with scale $\beta$,
let $t=2r/\beta$. The corresponding divergence is
\begin{equation}
d_\nu^{\mathrm{Lap}}(t)
=
\frac{1}{\nu-1}
\ln\left(
\frac{
\nu e^{(\nu-1)t}
+(\nu-1)e^{-\nu t}
}{2\nu-1}
\right).
\label{eq:vector_laplace_rdp}
\end{equation}

For independent Gaussian noise with standard deviation
$\sigma$,
\begin{equation}
d_\nu^{\mathrm{Gau}}
=\frac{\nu(2r)^2}{2\sigma^2}.
\label{eq:vector_gaussian_rdp}
\end{equation}

Each per-coordinate divergence is composed across
the $m$ coordinates and converted to an
$(\epsilon_v,\delta_v)$ guarantee using the same
RDP accounting procedure as in
\eqref{eq:cpqlambda_vector_privacy}.
The mechanisms are calibrated to the common target
$(\epsilon_v,\delta_v)=(1,10^{-5})$.

\noindent\textbf{Aggregation-MSE comparison.}
Define the binary reconstruction scale
\[
A(s)=r\frac{e^s+1}{e^s-1},
\]
and consider an even number $K$ of learning vehicles
forming independent pairs.

At centered clipped inputs $w_{i,j}=c_j$, each binary
output has probability $1/2$. For CPQ-$\lambda$,
the covariance between the reconstructed parameters
of a pair is $-\mu A(s)^2$. Consequently, the
per-coordinate FedAvg MSEs are
\begin{equation}
\begin{aligned}
M_{\mathrm{CPQ}\text{-}\lambda}
&=\frac{(1-\mu)A(s)^2}{K},\\
M_{\mathrm{LDP}}
&=\frac{A(s_0)^2}{K},\\
M_{\mathrm{Gau}}
&=\frac{\sigma^2}{K},\\
M_{\mathrm{Lap}}
&=\frac{2\beta^2}{K},
\end{aligned}
\label{eq:vector_matched_mse}
\end{equation}
where $s$ and $s_0$ are the respective
RDP-calibrated parameter-level privacy budgets
of CPQ-$\lambda$ and LDP-FL.

The MSE advantage can also be bounded for arbitrary
admissible clipped inputs. Writing
\[
q_-(s)=\frac{1}{1+e^s},
\]
the CPQ-$\lambda$ pair-average MSE satisfies
\[
M_{\mathrm{pair}}^{\mathrm{CPQ}\text{-}\lambda}
\leq
\frac{A(s)^2}{2}
\left(1-4\mu q_-^2(s)\right),
\]
while the independent LDP-FL pair-average MSE satisfies
\[
M_{\mathrm{pair}}^{\mathrm{LDP}}
\geq
\frac{A(s_0)^2-r^2}{2}.
\]
Since the errors of distinct pairs are independent,
these inequalities imply
\begin{equation}
\frac{M_{\mathrm{CPQ}\text{-}\lambda}}
{M_{\mathrm{LDP}}}
\leq
\frac{
A(s)^2\left(1-4\mu q_-^2(s)\right)
}{
A(s_0)^2-r^2
}.
\label{eq:vector_mse_ratio_bound}
\end{equation}

For the matched-privacy configuration in
Table~\ref{tab:vector_privacy_comparison},
the right-hand side is approximately $0.836$
at $\lambda=0.627$. Thus, the considered
CPQ-$\lambda$ configuration achieves lower
theoretical aggregation MSE than the independent
LDP-FL baseline throughout the admissible
clipping range.

The comparison applies a common RDP-based privacy
certification procedure. Tighter mechanism-specific
approximate-DP accounting may be available for
independent LDP-FL and Gaussian mechanisms and
could permit lower aggregation error at the same
privacy target. These results concern one model
upload under ideal independent randomization and
do not establish multi-round or trajectory-level
privacy.

\section{Finite-Precision Analysis}
\label{app:finite_precision}

To analyze finite-precision effects, consider a
$d$-bit integer $K$ uniformly distributed over
$\{0,\dots,2^d-1\}$. The leader uses $K_1=K$
and the follower uses $K_2=2^d-1-K$, providing
a discrete counterpart of $V$ and $1-V$. With $q_k=q(w_k)$ and
$k_k=\lfloor 2^d q_k\rfloor$, client $k\in\{1,2\}$ transmits
\begin{equation}
  b_k=
  \begin{cases}
    1, & K_k<k_k,\\
    0, & K_k>k_k,\\
    \mathrm{Bernoulli}\bigl(2^d q_k-k_k\bigr), & K_k=k_k,
  \end{cases}
  \label{eq:cpq_bit_finite}
\end{equation}
where the Bernoulli draw uses local randomness and is needed only when $K_k$ falls in
the cell containing $q_k$. As is standard for local randomized response, this local draw
is modeled as exact; only the shared randomness must be finite, since both clients must
reproduce it identically.

\begin{proposition}[Finite-precision CPQ]
\label{prop:cpq_finite}
For every $d\in\mathbb{N}$, the pair $(b_1,b_2)$ generated by~\eqref{eq:cpq_bit_finite}
satisfies:
\begin{enumerate}[label=(\alph*)]
  \item $P(b_k=1)=q(w_k)$ exactly for $k\in\{1,2\}$; hence part (i) of
    Theorem~\ref{thm:cpq} holds exactly.
  \item The pairwise MSE $m_d(w_1,w_2)$ satisfies
    $0\le m_d(w_1,w_2)-m^*(w_1,w_2)\le r^2\alpha^2(\epsilon_p)/2^{d-2}$.
\end{enumerate}
\end{proposition}

\begin{proof}
\textit{(a)} Since $q_k<1$, $k_k\le 2^d-1$. Both $K_1$ and $K_2$ are uniform on
$\{0,\dots,2^d-1\}$, so
\begin{align*}
   P(b_k=1)&=P(K_k<k_k)+P(K_k=k_k)\,(2^dq_k-k_k)
=\frac{k_k}{2^d}+\frac{2^dq_k-k_k}{2^d}=q_k . 
\end{align*}

This output distribution preserves both unbiasedness and $\epsilon_p$-PLDP, so the arguments of part (i) of Theorem~\ref{thm:cpq} apply unchanged.

\textit{(b)} Let $W_1,W_2$ be i.i.d.\ $\mathrm{Uniform}[0,1)$, independent of $K$, and
set $V_k=(K_k+W_k)/2^d$. Conditional on $K_k$, the bit $\mathbf{1}\{V_k<q_k\}$ has the
same law as $b_k$: if $K_k<k_k$ then $V_k<(K_k+1)/2^d\le q_k$; if $K_k>k_k$ then
$V_k\ge (k_k+1)/2^d>q_k$; and if $K_k=k_k$ then
$P(V_k<q_k)=P(W_k<2^dq_k-k_k)=2^dq_k-k_k$. Hence we may take
$b_k=\mathbf{1}\{V_k<q_k\}$. Here $V_1$ is uniform on $[0,1)$, and the idealized follower
of Theorem~\ref{thm:cpq} uses $1-V_1=(K_2+1-W_1)/2^d$; denote its bit by $b_2^{\circ}$,
so that $(b_1,b_2^{\circ})$ is the idealized CPQ pair with MSE $m^*$. The bits $b_2$ and
$b_2^{\circ}$ are both determined by $K_2$ unless $K_2=k_2$, so they differ only on an
event of probability $2^{-d}$. The pair-average error satisfies
$|e|\le r\alpha+r\le 2r\alpha$, so $e^2\in[0,4r^2\alpha^2]$, and therefore
$m_d-m^*\le 4r^2\alpha^2\cdot 2^{-d}=r^2\alpha^2/2^{d-2}$. Finally, $m_d\ge m^*$ since,
by (a), the pair has the prescribed marginals, and the Fr\'echet--Hoeffding bound in
the proof of Theorem~\ref{thm:cpq} applies to every coupling with these marginals.
\end{proof}

\begin{proposition}[Finite-precision stronger-view privacy of CPQ-$\lambda$]
\label{prop:cpqlambda_finite}
Let the shared branch of CPQ-$\lambda$ use the $d$-bit rule in~\eqref{eq:cpq_bit_finite}, while the private branch uses local randomness. Then the marginal PLDP and unbiasedness result of Theorem~\ref{thm:cpqlambda}(a) holds exactly, and the side-information bound $\epsilon_C$ in~\eqref{eq:eps_collusion} remains valid without a finite-precision correction. For $0<\mu<1$, define $\epsilon_S^{(d)}$ as the pair-view privacy parameter of the finite-precision mechanism. If
\begin{equation}
  2^{-d}\le\frac{1}{4\mu(e^{\epsilon_p}+1)^2},
  \label{eq:finite_condition}
\end{equation}
then
\begin{equation}
 e^{\epsilon_S^{(d)}}
 \le e^{\epsilon_S}
 +\frac{4\mu e^{\epsilon_p}(e^{\epsilon_p}+1)^2}{1-\mu}\,2^{-d}.
 \label{eq:finite_eps_server}
\end{equation}
For $\mu=0$, the mechanism is the independent LDP-FL coupling and $\epsilon_S^{(d)}=\epsilon_p$ exactly.
\end{proposition}

\begin{proof}
Part~(a) of Proposition~\ref{prop:cpq_finite} gives the exact marginal probabilities, so marginal PLDP and unbiasedness are unchanged. Conditional on a revealed $d$-bit shared value, the shared-branch probability of $b_1=1$ is a nondecreasing function $h_d(q)\in[0,1]$. Hence
\[
 \Pr(b_1=1\mid q,Z)
 =\lambda_1 h_d(q)+(1-\lambda_1)q.
\]
The largest possible endpoint ratio is obtained by the extreme values $h_d(q_-)=0$ and $h_d(q_+)=1$, which gives exactly the upper bound in~\eqref{eq:eps_collusion}; the same argument applied to the complementary output handles $b_1=0$. Thus no additional correction is needed for $\epsilon_C$ (although exact attainment of the continuous bound need not occur on a very coarse grid).

For the pair view, use the coupling constructed in the proof of Proposition~\ref{prop:cpq_finite}. The finite-precision CPQ pair and the ideal continuous CPQ pair differ only when the follower lies in its boundary cell, an event of probability at most $\delta_d=2^{-d}$. Therefore, for every output atom $a\in\{0,1\}^2$,
\begin{equation}
 \left|P_{\mathrm{CPQ}}^{(d)}(a)-P_{\mathrm{CPQ}}(a)\right|\le\delta_d,
 \label{eq:finite_atom_tv}
\end{equation}
and the CPQ-$\lambda$ mixture differs from its continuous counterpart at any atom by at most $x=\mu\delta_d$.

Consider first the concordant atom $(1,1)$ and the extremal change $q_1:q_-\to q_+$. For $q_1=q_-$, the ideal CPQ contribution to this atom is zero; the finite CPQ contribution is nonnegative. Hence the finite denominator is at least the continuous denominator $D=(1-\mu)q_-q_2$, while the numerator increases by at most $x$. Using the continuous bound $N/D\le e^{\epsilon_S}$ gives
\[
 \frac{N_d}{D_d}
 \le e^{\epsilon_S}+\frac{x}{(1-\mu)q_-^2}.
\]
This additive term is no larger than the correction in~\eqref{eq:finite_eps_server}.

For the discordant atom $(1,0)$, let $N,D$ denote the continuous probabilities at $q_+$ and $q_-$, respectively. The continuous proof gives $N/D\le e^{\epsilon_p}$. Moreover
\[
 D=q_-[\mu+(1-\mu)p_2]\ge q_-^2.
\]
Condition~\eqref{eq:finite_condition} implies $x\le q_-^2/4\le D/4$, so
\begin{align*}
 \frac{N_d}{D_d}
 &\le\frac{N+x}{D-x}\le e^{\epsilon_p}
 +\frac{x(e^{\epsilon_p}+1)}{D-x}\le e^{\epsilon_p}
 +2\mu\delta_d(e^{\epsilon_p}+1)^3.
\end{align*}
Because $e^{\epsilon_p}>1$ and $0<\mu<1$,
\[
 2\mu(e^{\epsilon_p}+1)^3
 \le\frac{4\mu e^{\epsilon_p}(e^{\epsilon_p}+1)^2}{1-\mu},
\]
and $e^{\epsilon_p}\le e^{\epsilon_S}$. Reflection gives the same bounds for $(0,0)$ and $(0,1)$. Taking the largest atom-wise ratio proves~\eqref{eq:finite_eps_server}.
\end{proof}

\noindent\textit{PRG-generated randomness.} Proposition~\ref{prop:cpq_finite} assumes
$K$ is uniformly random. When $K$ is produced by a cryptographically secure PRG, each
transmitted bit is a polynomial-time function of the PRG output and local randomness,
so any efficient adversary that distinguishes the PRG-based outputs from those of the
truly random mechanism would also distinguish the PRG from uniform. The guarantees
therefore hold against computationally bounded adversaries, up to a negligible loss.
With $d=32$, the MSE gap in (b) is below $10^{-9}\,r^2\alpha^2(\epsilon_p)$. 

The experimental implementation uses ChaCha20-derived
53-bit floating-point values with direct threshold
comparisons. The boundary-randomized construction
analyzed above provides an idealized finite-precision
model with exact marginal guarantees; the direct
comparison used in the experiments is a numerically
close approximation.

\section{SUMO Network Topology Characteristics}
\label{app:simulation_details}
The simulated environment represents a custom 10-junction road network mapped out over a 723\,m $\times$ 200\,m coordinate plane (see Figure \ref{fig:sumo_map}). The topology consists of primary arterial roads connecting major junctions (J0 through J12). To simulate diverse urban driving conditions, the edges have distinct speed constraints. The majority of the 30 directional edges permit a maximum velocity of 15.00\,m/s ($\approx$ 54\,km/h). However, the central connecting route between junctions J3 and J5 (edges \texttt{E7} and \texttt{-E7}) enforces a stricter speed limit of 10.00\,m/s ($\approx$ 36\,km/h), acting as a systematic bottleneck designed to force the RL agents to balance spatial distance against congestion and speed constraints.

For the privacy--utility evaluation, the road network
contains $M=30$ directed edges. Accordingly, the
state representation defined in~\eqref{eq:state}
has dimension $3M+4=94$. The maximum number of
turning actions is $|\mathcal{A}_{\max}|=3$,
with unavailable turns excluded through
action masking.

\begin{wrapfigure}{r}{0.48\textwidth}
\centering
\includegraphics[width=\linewidth]{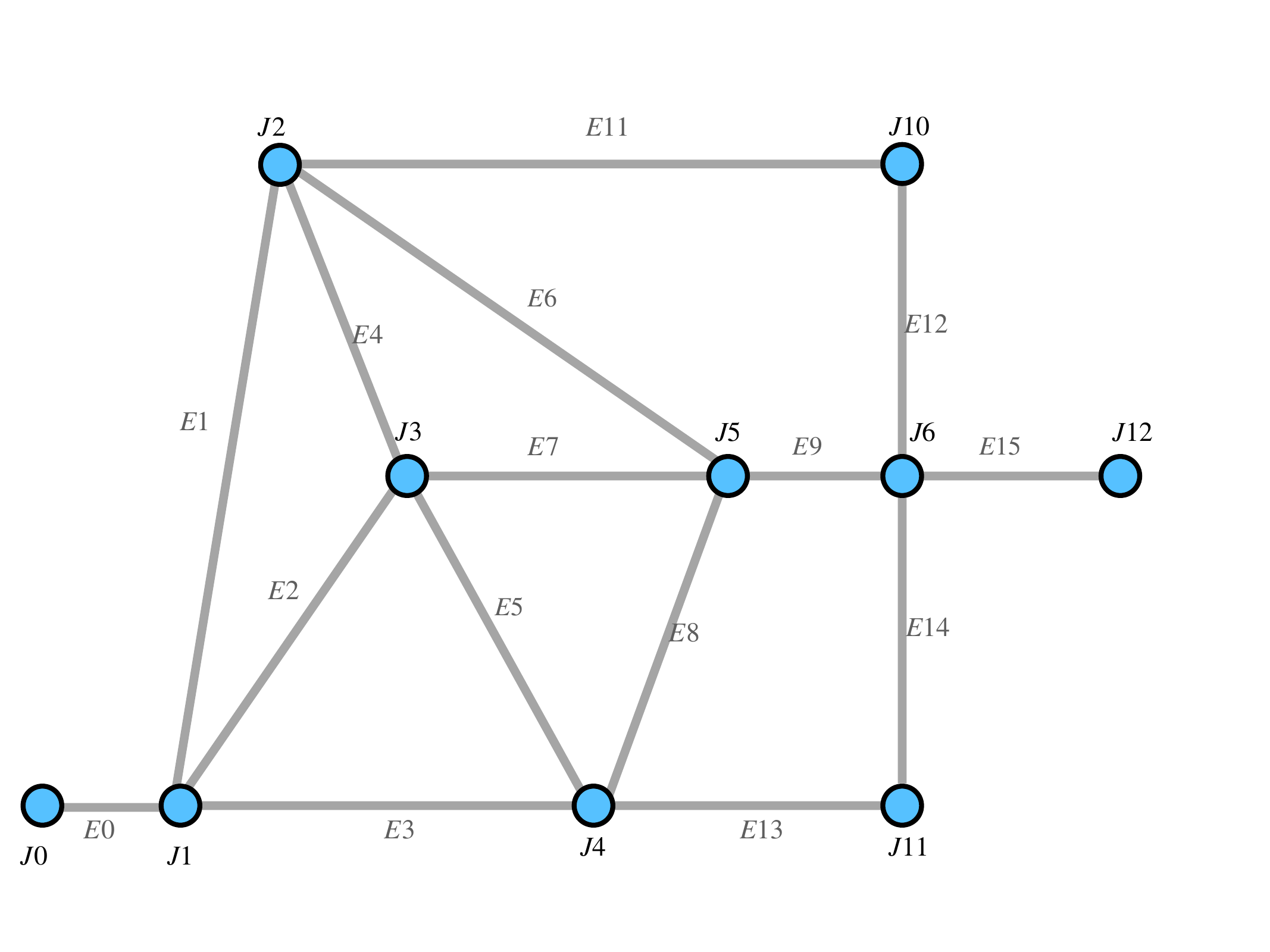}
\caption{Topology of the synthetic road network implemented
in the SUMO environment. The network comprises 10 junctions
(J0--J6 and J10--J12) and 30 directed edges
(15 bidirectional corridors) covering a total distance of
5.37\,km.}
\label{fig:sumo_map}
\end{wrapfigure}

\section{Experimental Configuration}
\label{app:exp_config}

This appendix consolidates the training and communication
parameters used in the experimental evaluation. The FLDQN
configuration specifies the learning architecture and training
schedule, while the SR-ARQ configuration documents the
fragmentation and retransmission settings used in ns-3.

\noindent\textbf{FLDQN training configuration.}
The FLDQN agents use the network architecture and training
procedure described in Section~\ref{subsec:fldqn_training}
of the main manuscript.
Table~\ref{tab:hyperparams} summarizes the corresponding
hyperparameters, including optimization, exploration,
episode termination, and federated training settings.

\noindent\textbf{SUMO traffic configuration.}
The privacy--utility experiment uses SUMO~1.26
and the road network described in
Appendix~\ref{app:simulation_details}.
Background traffic is generated using
\texttt{randomTrips.py}, with an average
departure interval of 1.75~s. A warm-up period
of 200~s precedes the departure of the
learning vehicles.

The $K=50$ learning vehicles are distributed
equally across starting edges
\texttt{E0}, \texttt{E1}, \texttt{E3},
\texttt{E11}, and \texttt{E8}, with a common
destination at edge \texttt{E15}.
Their departures are sampled uniformly
within a 100~s window following the warm-up
period and regenerated for each training
episode.

Individual trips are subject to a maximum
duration of $T_{\mathrm{trip}}=600$~s,
measured from the scheduled departure.
Vehicles that exceed this limit are
classified as not arrived. Such time-limit
terminations are stored as non-terminal
transitions in the local replay buffers.

For evaluation, three fixed traffic
scenarios are used, with background
traffic seeds 1000, 1001, and 1002,
and corresponding departure seeds
5000, 5001, and 5002. The same evaluation
scenarios are used for all privacy
mechanisms, privacy budgets, and training
seeds.

\noindent\textbf{Privacy clipping configuration.}
For privacy-preserving model exchange, each
locally trained parameter is clipped around
the corresponding parameter of the current
global model using a fixed radius
$r=0.02$ across all parameter tensors.
The resulting coordinate-wise interval is
shared by clients trained from the same
global model and is used for privacy
calibration and parameter reconstruction.

\noindent\textbf{Agent--SUMO interaction.}
The Python training process interacts with a
shared SUMO simulation through TraCI. Each
learning vehicle maintains an individual DQN
policy and experience buffer while observing
traffic conditions within the common environment.
At routing decision points, the vehicle obtains
the traffic and positional information required
to construct the state in~\eqref{eq:state}.

The base reward is determined from the elapsed
SUMO simulation time between consecutive
edge-entry events, as defined in~\eqref{eq:reward},
with repeated-edge penalties applied according
to~\eqref{eq:loop_penalty}. Individual trip termination follows the conditions
described in Section~\ref{subsec:routing_formulation}.

\noindent\textbf{ns-3 channel and radio configuration.}
The V2I simulation uses ns-3 version~3.46,
a 5.9\,GHz carrier frequency, and a 10\,MHz channel.
The RSU is located at junction J3,
$(300,100,0)$\,m, of the SUMO road network.
Both data and control transmissions use
\texttt{OfdmRate6MbpsBW10MHz} through
\texttt{ConstantRateWifiManager}, with
transmission power set to 20\,dBm.
Propagation is modeled using Log-Distance
path loss with exponent $n=2.7$ and
reference loss $L_0=46.68$\,dB,
followed by Nakagami-$m$ fading and
constant-speed propagation delay.
The MAC uses \texttt{AdhocWifiMac}
without QoS extensions. The receiver noise figure is 7\,dB,
the CCA energy-detection threshold is
$-85$\,dBm, the minimum preamble-detection
RSSI is $-82$\,dBm, and the receiver
sensitivity is $-101$\,dBm.

\begin{wraptable}{r}{0.50\textwidth}
\centering
\caption{FLDQN configuration for the privacy--utility experiment.}
\label{tab:hyperparams}
\footnotesize
\setlength{\tabcolsep}{2pt}
\renewcommand{\arraystretch}{1.10}
\begin{tabular}{@{}p{0.43\linewidth}
                    p{0.23\linewidth}
                    p{0.26\linewidth}@{}}
\toprule
\textbf{Parameter} & \textbf{Symbol} & \textbf{Value} \\
\midrule
State dimension
    & $3M+4$ & 94 \\
Action space size
    & $|\mathcal{A}_{\max}|$ & 3 \\
Network architecture
    & --- & 94--150--100--3 \\
Optimizer
    & --- & Adam \\
Learning rate
    & $\eta$ & $10^{-3}$ \\
Discount factor
    & $\gamma$ & 0.99 \\
Replay capacity
    & $|\mathcal{B}_i|$ & 20\,000 \\
Mini-batch size
    & $B$ & 32 \\
Training start
    & $N_{\mathrm{start}}$ & 200 transitions \\
Target update period
    & $C_{\mathrm{target}}$ & 1\,000 steps \\
Huber loss threshold
    & $\delta_{\mathrm{H}}$ & 1.0 \\
Exploration start
    & $\epsilon_{\mathrm{start}}$ & 1.0 \\
Exploration minimum
    & $\epsilon_{\mathrm{end}}$ & 0.05 \\
Exploration decay
    & --- & Exp., 80 rounds \\
Second-visit penalty
    & $\kappa_2$ & 10 \\
Subsequent-visit penalty
    & $\kappa_{3+}$ & 50 \\
Local episodes per round
    & $E$ & 10 \\
Participating vehicles
    & $K$ & 50 \\
Federated rounds
    & $T$ & 100 \\
\bottomrule
\end{tabular}
\end{wraptable}

\noindent\textbf{Network address resolution.}
The ns-3 implementation prepopulates ARP caches
for vehicle--RSU communication to prevent
address-resolution state from introducing
prolonged delivery failures when vehicles
temporarily leave radio coverage.

\noindent\textbf{Beacon-based contact configuration.}
The RSU transmits discovery beacons of size
$L_{\mathrm{b}}=100$\,B at intervals of
$T_{\mathrm{b}}=0.1$\,s over the same simulated
IEEE~802.11p channel used for model transfers.
Contact eligibility is evaluated over a sliding
observation window of
$T_{\mathrm{obs}}=0.5$\,s, requiring at least
$N_{\mathrm{beacon}}=3$ successful beacon
receptions. This corresponds to the nominal
three-of-five beacon reception criterion used
in the experiments.

\noindent\textbf{RSU transfer scheduling configuration.}
The RSU evaluates the transfer queue at intervals
of $\Delta t_{\mathrm{sched}}=0.1$\,s of
simulation time. At each scheduling step, the
controller processes completed transfers,
updates vehicle contact eligibility using the
beacon-reception criterion, and starts the next
queued transfer when no other transfer is active.

Waiting vehicles that lose contact are removed
from the queue. Following an unsuccessful
transfer, the corresponding model object remains
pending and may be scheduled for another attempt
when the vehicle is eligible again. Completed
attempts are placed at the end of the queue
when additional transfer work remains.
No new attempts are initiated after the end
of the corresponding communication drive.

\noindent\textbf{Communication-drive configuration.}
Each downlink or uplink communication phase is
executed during a non-learning SUMO drive.
Following a 200\,s traffic warm-up, learning
vehicles depart at randomized times within a
100\,s departure window and follow their current
routing policies greedily. Individual trips
terminate upon arrival, route failure, or a
maximum trip duration of 600\,s.

The RSU schedules transfers during the available
vehicle contacts, with the end of the drive
serving as the deadline for outstanding requests.
Each federated round includes separate downlink
and uplink drives, followed by global-model
evaluation on three fixed scenarios without
additional V2I model communication.

\noindent\textbf{SUMO-ns-3 control interface.}
The Python co-simulation driver maintains one
persistent TCP connection to a control endpoint
in the ns-3 process. Supported operations include
updating vehicle positions and velocities,
moving inactive vehicles outside communication
range, querying beacon-reception status,
initiating uplink and downlink transfers,
advancing the ns-3 simulation clock, and
resetting the simulation state.

Transfer initiation is asynchronous. When the
driver advances ns-3 to a specified simulation
time, the interface reports transfers completed
during the elapsed interval, including their
delivery outcomes, simulated completion times,
and received bytes. Successfully delivered
objects are then passed to the corresponding
receiving endpoint for signature and
communication-context verification. Model
traffic itself is carried through the simulated
UDP/IEEE~802.11p data plane rather than the
TCP control connection.

\noindent\textbf{SR-ARQ transport configuration.}
The ns-3 communication model employs a custom Selective
Repeat ARQ layer for packet fragmentation, retransmission,
and reassembly. Table~\ref{tab:sr_params} summarizes its
transport parameters, including chunk sizes, window limits,
retransmission settings, and UDP port assignments.

Each SR-ARQ chunk additionally carries a 44-byte binary
header containing message identification, sequencing,
transmission timing, and integrity metadata. The header
layout is specified as follows:
\begin{equation}
  \text{Header} = \underbrace{\texttt{u8}}_{\rm type}
  \underbrace{\texttt{u8}}_{\rm rsv}
  \underbrace{\texttt{u16}}_{\rm flags}
  \underbrace{\texttt{u64}}_{\rm msg\_id}
  \underbrace{\texttt{u32}}_{\rm round}
  \underbrace{\texttt{u16}}_{\rm veh}
  \underbrace{\texttt{u16}}_{\rm idx} 
  \underbrace{\texttt{u16}}_{\rm total}
  \underbrace{\texttt{u16}}_{\rm wnd}
  \underbrace{\texttt{u64}}_{\rm tx\_ns}
  \underbrace{\texttt{u32}}_{\rm tot\_len}
  \underbrace{\texttt{u32}}_{\rm chk\_len}
  \underbrace{\texttt{u32}}_{\rm crc32}
  \;\Big|\;
  \underbrace{\phantom{xxxx}}_{\text{payload}}
  \label{eq:chunk_hdr}
\end{equation}

The \texttt{flags} field identifies the transmission
direction, while \texttt{idx} and \texttt{total} specify
the chunk index and the total number of chunks in the
message. The \texttt{crc32} field contains the CRC-32
checksum used to validate the chunk payload.

\begin{wraptable}{r}{0.54\textwidth}
\centering
\caption{SR-ARQ transport parameters used in ns-3.}
\label{tab:sr_params}
\footnotesize
\setlength{\tabcolsep}{3pt}
\renewcommand{\arraystretch}{1.10}
\begin{tabular}{p{0.45\linewidth}p{0.20\linewidth}p{0.25\linewidth}}
\toprule
\textbf{Parameter} & \textbf{Symbol} & \textbf{Value} \\
\midrule
Chunk payload size      & $L_{\rm chunk}$     & 1200\,B \\
Chunk header size       &                     & 44\,B \\
Initial window          & $W_0$               & 32 chunks \\
Maximum window          & $W_{\max}$          & 128 chunks \\
Retransmission timeout  & $\tau_{\rm rto}$    & 120\,ms \\
Per-chunk retry cap     & $R_{\max}$          & 8 \\
Inter-chunk pacing      & $\Delta_{\rm pace}$ & 200\,\textmu s \\
Window increase (ACK)   &                     & $+1$ (additive) \\
Uplink UDP port         &                     & 40000 \\
Downlink UDP base port  &                     & 41000\,$+$\,veh\,index \\
ACK UDP port            &                     & 40001 \\
\bottomrule
\end{tabular}
\end{wraptable}

\noindent\textbf{Software and execution environment.}
The experimental framework integrates SUMO, ns-3,
and the Python-based learning process. SUMO
communicates with the training process through
TraCI, while ns-3 operates as a separate process
and advances in lockstep with the SUMO simulation
clock. Local DQN training uses TensorFlow/Keras,
whereas privacy processing and model serialization
use NumPy.

Federated aggregation is coordinated within the
Python training process using Flower's strategy
classes, without its client--server gRPC transport.
Post-quantum signature operations and ML-KEM-512
key establishment use liboqs through its Python
\texttt{oqs} interface. The CPQ group setup
additionally employs ChaCha20, ChaCha20-Poly1305,
HKDF-SHA256, and HMAC-SHA256 through the
\texttt{cryptography} package.
Table~\ref{tab:software_stack} summarizes the
software components and their versions.

\noindent\textbf{Execution hardware.}
The controlled cryptographic microbenchmarks
were performed on an x86\_64 Linux workstation
equipped with an Intel Core i7-8700 CPU at
3.20\,GHz, 12 logical CPUs, and 15.42\,GiB
of RAM. The software environment used
Linux 5.4 and Python~3.10.20.
Signature timing measurements were CPU-based.

\noindent\textbf{PQC key persistence.}
Signing key pairs are generated through the Python
\texttt{oqs} interface to liboqs during system
initialization. The resulting keys are stored as
\texttt{\char`\{identity\char`\}\_\char`\{scheme\char`\}.pk/.sk} files
and reloaded when available to maintain consistency
across restarts within a training run.

\begin{table}[!t]
\centering
\caption{Software components and reported versions.}
\label{tab:software_stack}
\footnotesize
\setlength{\tabcolsep}{2pt}
\renewcommand{\arraystretch}{1.16}

\begin{tabular}{@{}
>{\raggedright\arraybackslash}p{0.29\columnwidth}
>{\raggedright\arraybackslash}p{0.40\columnwidth}
>{\raggedright\arraybackslash}p{0.25\columnwidth}
@{}}
\toprule
\textbf{Component} & \textbf{Tool} & \textbf{Version} \\
\midrule

Traffic simulator
  & SUMO & 1.26 \\

Network simulator
  & ns-3 & 3.46 \\

DQN training
  & TensorFlow (Keras) & 2.x \\

Privacy processing
  & NumPy & 1.x \\

FL coordination
  & Flower (\texttt{flwr}) & 1.x \\

PQC signatures and ML-KEM
  & liboqs / liboqs-python & 0.15.0 / 0.14.1 \\

Classical signature benchmark
  & \texttt{cryptography}/OpenSSL & 46.0.7 \\

Payload serialization
  & NumPy (\texttt{savez\_compressed}) & 1.x \\

CRC integrity
  & zlib & system \\

Python integration
  & Python & 3.10.20 \\

\bottomrule
\end{tabular}
\end{table}

\section{Controlled Signature Microbenchmark}
\label{app:signature_microbenchmark}

This section presents a controlled microbenchmark of the digital
signature schemes considered in this study. Table~\ref{tab:pqc_comparison}
compares public-key sizes, signature sizes, and the computational costs
of key generation, signing, and verification across four post-quantum
configurations and two classical reference schemes. These measurements
isolate cryptographic operations from model serialization and network
transmission, providing a reference for interpreting the integrated
V2X communication results in Section~\ref{subsec:exp_pqc}.

\noindent\textbf{Measurement protocol.}
The microbenchmark was conducted on the evaluation workstation
described in Appendix~\ref{app:exp_config}. Post-quantum signature
operations were evaluated using liboqs~0.15.0 through
liboqs-python~0.14.1. The classical RSA and ECDSA reference
schemes were evaluated using the \texttt{cryptography}~46.0.7
package with its OpenSSL backend.

For each signature scheme, five warm-up trials were executed
and excluded from measurement, followed by $N=200$ recorded
trials. Each recorded trial generated a fresh key pair, signed
a 4096-byte message, and verified the resulting signature.
Key generation, signing, and verification were timed
independently using a monotonic high-resolution clock.

\begin{table}[t]
\centering
\footnotesize
\caption{Controlled signature microbenchmark on the evaluation
workstation. Timing values are mean~$\pm$~standard deviation in
milliseconds over $N=200$ measured trials using 4096-byte messages.}
\label{tab:pqc_comparison}
\begin{tabular*}{0.95\textwidth}{@{\extracolsep{\fill}} lcccccc}
\toprule
\textbf{Algorithm}
& \textbf{Basis}
& \textbf{PK (B)}
& \textbf{Sig. (B)}
& \textbf{Keygen (ms)}
& \textbf{Sign (ms)}
& \textbf{Verify (ms)} \\
\midrule
ML-DSA-44
& MLWE/MSIS
& 1312
& 2420
& $0.037 \pm 0.009$
& $0.105 \pm 0.063$
& $0.045 \pm 0.009$ \\

Falcon-512
& NTRU/SIS
& 897
& 655.1
& $5.771 \pm 1.285$
& $0.210 \pm 0.027$
& $0.059 \pm 0.012$ \\

SLH-DSA-SHA2-128f
& Hash
& 32
& 17088
& $0.989 \pm 0.086$
& $23.078 \pm 1.377$
& $1.412 \pm 0.110$ \\

SLH-DSA-SHA2-128s
& Hash
& 32
& 7856
& $61.707 \pm 2.624$
& $469.949 \pm 4.877$
& $0.491 \pm 0.030$ \\
\midrule

RSA-2048-PSS-SHA256
& IFP
& 256
& 256
& $38.741 \pm 24.430$
& $0.799 \pm 0.025$
& $0.048 \pm 0.003$ \\

ECDSA-P256-SHA256
& ECDLP
& 65
& 71.0
& $0.022 \pm 0.004$
& $0.053 \pm 0.006$
& $0.084 \pm 0.008$ \\
\bottomrule
\end{tabular*}
\end{table}

Because Falcon-512 and ECDSA P-256 produce variable-length
signatures, Table~\ref{tab:pqc_comparison} reports their mean
signature sizes across the measured trials. For Falcon-512,
signature sizes ranged from 648 to 661 bytes, with a mean of
655.1 bytes. ECDSA P-256 signatures ranged from 69 to 72 bytes,
with a mean of 71.0 bytes. The evaluated Falcon-512
implementation supports a maximum signature size of 752 bytes,
which is distinct from the sizes observed during the 200 trials.

The microbenchmark includes both SLH-DSA-SHA2-128f and SLH-DSA-SHA2-128s to characterize the trade-offs between hash-based signature parameter sets. The primary mobility-coupled vehicular evaluation
employs ML-DSA-44 for authenticated model exchange,
while the additional signature-sensitivity experiment
evaluates ML-DSA-44, Falcon-512, and
SLH-DSA-SHA2-128f under mobility-coupled V2I
communication. RSA-2048-PSS-SHA256
and ECDSA-P256-SHA256 are included as conventional reference
baselines for computational comparison. These classical
signature schemes are not quantum-resistant and are not
used in the integrated vehicular experiments.

For the notation in Table~\ref{tab:pqc_comparison}, PK refers
to the raw public-key representation used for comparison.
The underlying mathematical bases are identified using IFP
for the Integer Factorization Problem, ECDLP for the Elliptic
Curve Discrete Logarithm Problem, MLWE for Module Learning
With Errors, and MSIS for Module Short Integer Solution.

\section{Signature-Scheme Sensitivity Under
Mobility-Coupled V2I Communication}
\label{app:e3_signature_network}

This appendix provides an additional analysis of
the signature-sensitivity experiment presented in
Section~\ref{subsec:exp_pqc}. We examine how the
different authenticated message sizes influence
uplink transfer reliability as vehicle-to-RSU
distance increases. The comparison includes
ML-DSA-44, Falcon-512, and SLH-DSA-SHA2-128f,
each evaluated with full-precision and one-bit
CPQ model updates under the same IEEE~802.11p
communication framework.

\noindent\textbf{Distance-dependent uplink reliability.}
Figure~\ref{fig:e3_distance} reports the
uplink transfer-attempt success rates across
vehicle-to-RSU distance intervals. With
full-precision updates, all three signature
schemes exhibit a similar decline in transfer
reliability as the vehicle moves farther
from the RSU. At distances of $125$--$150$\,m,
the attempt success rates are $3\%$, $4\%$,
and $3\%$ for ML-DSA-44, Falcon-512,
and SLH-DSA-SHA2-128f, respectively.

In contrast, CPQ substantially improves
distance-dependent communication reliability
by reducing the authenticated uplink size.
At $125$--$150$\,m, CPQ achieves $100\%$
transfer-attempt success with ML-DSA-44
and Falcon-512, compared with $25\%$
for SLH-DSA-SHA2-128f. At $150$--$175$\,m,
the corresponding success rates are
$77\%$, $94\%$, and $0\%$.
These results illustrate how signature
size becomes increasingly influential
when the model-update payload is highly
compressed and wireless connectivity
is constrained by vehicle mobility.

The overall uplink delivery rate remains
distinct from the per-attempt success rate,
since the communication protocol can
retransmit unsuccessful transfers. For CPQ,
the overall uplink delivery rates are
$99.8\%$ with ML-DSA-44, $99.8\%$
with Falcon-512, and $93.1\%$ with
SLH-DSA-SHA2-128f. Thus, retransmissions
support eventual delivery, but larger
authenticated messages impose additional
wireless communication costs.

\begin{figure}[!t]
\centering
\includegraphics[width=0.85\textwidth]
{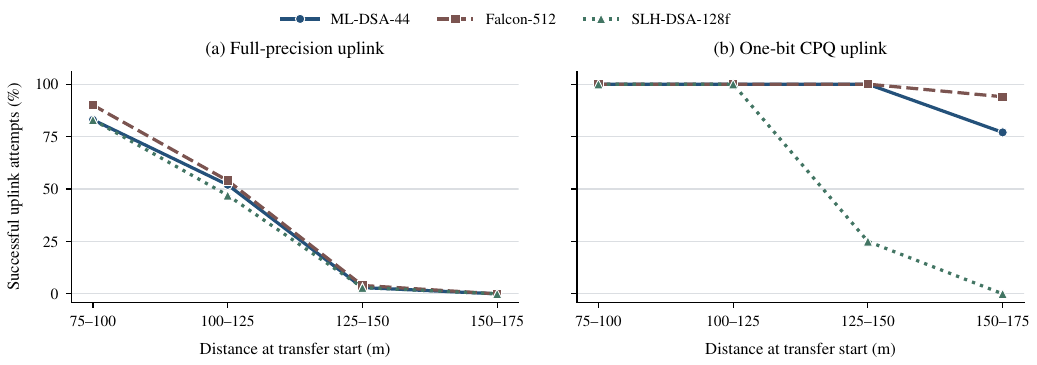}
\caption{Distance-dependent uplink transfer-attempt
success under IEEE~802.11p communication for
three post-quantum signature schemes:
(a) full-precision updates and (b) one-bit
CPQ updates. Each point represents the
percentage of successful uplink transfer
attempts within the corresponding
vehicle-to-RSU distance interval at
transfer initiation ($\epsilon_p=1$).}
\label{fig:e3_distance}
\end{figure}

\section{Detailed Privacy--Utility Experimental Results}
\label{app:detailed_privacy_results}

This appendix provides additional analyses of the
privacy--utility results presented in
Section~\ref{subsec:exp_privacy_utility}.
We examine seed-paired differences between
CPQ and LDP-FL, the effect of CPQ-$\lambda$
correlation mixing, run-to-run variability,
and secondary traffic-delay measures.
The primary comparisons
consider privacy budgets
$\epsilon_p\in\{0.3,0.5,1,2,5\}$,
while the CPQ-$\lambda$ analysis is conducted
at a fixed $\epsilon_p=1$.

\subsection{Paired Comparison with Independent Quantization}

To isolate the empirical effect of correlated quantization,
Fig.~\ref{fig:cpq_paired_advantage} compares CPQ with LDP-FL
at identical privacy budgets using matched random seeds.
Both mechanisms employ the same one-bit representation and
clipping configuration, differing in the correlation of
their quantization randomness. For each seed, performance
is summarized over rounds 91--100 before computing the
paired difference between the two mechanisms.

The travel-time benefit is defined as the LDP-FL travel
time minus the corresponding CPQ travel time, while the
arrival-rate benefit is defined as the CPQ arrival rate
minus that of LDP-FL. Positive values therefore favor
CPQ in both metrics. The figure reports the individual
seed-paired differences, their mean, and the corresponding
95\% Student-$t$ confidence intervals across five seeds.

CPQ exhibits a positive mean travel-time benefit at
all five evaluated privacy budgets, ranging from
8.2 to 61.7~s. However, the uncertainty varies
considerably across configurations. The nominal 95\%
confidence interval excludes zero at $\epsilon_p=1$,
where the mean paired benefit is
$43.1 \pm 26.5$~s, but includes zero at the other
evaluated privacy budgets. For arrival rate, the
mean paired differences are generally smaller,
and all corresponding confidence intervals include zero.

These observations support a favorable average
routing-efficiency trend for CPQ relative to
independent quantization, particularly at tighter
privacy budgets. Nevertheless, the variability
across seeds indicates that the observed benefit
is not uniformly established across all settings.

\begin{figure}[t]
\centering
\includegraphics[width=0.85\textwidth]
{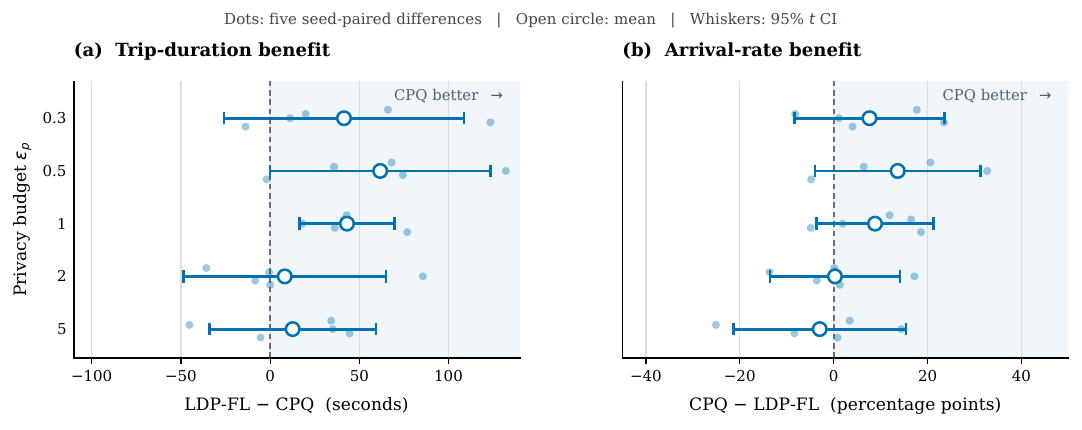}
\caption{Seed-paired routing-performance differences
between CPQ and LDP-FL.}
\label{fig:cpq_paired_advantage}
\end{figure}

\subsection{Learning-Vehicle Population Sensitivity}
\label{app:population_sensitivity}

We further examine the effect of the learning-vehicle population under ideal communication, comparing CPQ and LDP-FL at $\epsilon_p=1$ with $K\in\{10,25,50\}$. Each configuration uses 100 federated rounds, 10 local training episodes per round, and five matched seeds. Results are averaged over rounds 91--100 before computing the seed-paired differences. The $K=50$ results are taken from the primary privacy--utility evaluation. For $K=25$, CPQ forms 12 pairs per coordinate, with the remaining vehicle using independent quantization.

Table~\ref{tab:population_sensitivity} reports travel-time gains, defined as LDP-FL minus CPQ, and arrival-rate gains, defined as CPQ minus LDP-FL. Positive values favor CPQ in both metrics. Results are reported as paired means with 95\% Student-$t$ confidence half-widths, while the final column indicates the number of seeds in which CPQ achieves lower travel time.

\begin{wraptable}{r}{0.50\textwidth}
\vspace{-10pt}
\centering
\caption{Population sensitivity at $\epsilon_p=1$.}
\label{tab:population_sensitivity}
\setlength{\tabcolsep}{3pt}
\renewcommand{\arraystretch}{1.10}
\begin{tabular*}{\linewidth}{@{\extracolsep{\fill}}cccc@{}}
\toprule
$K$ &
\shortstack{Travel-time gain\\(s) $\uparrow$} &
\shortstack{Arrival-rate gain\\(pp) $\uparrow$} &
\shortstack{Faster\\seeds} \\
\midrule
$10$ & $42.0\pm142.0$ & $8.6\pm31.1$ & $3/5$ \\
$25$ & $61.8\pm43.5$ & $10.8\pm10.0$ & $5/5$ \\
$50$ & $43.1\pm26.5$ & $8.8\pm12.4$ & $5/5$ \\
\bottomrule
\end{tabular*}
\end{wraptable}

CPQ achieves positive mean travel-time and arrival-rate gains across all three populations. The travel-time confidence intervals exclude zero at $K=25$ and $K=50$, whereas the $K=10$ results show substantially greater variability. The gains do not increase monotonically with $K$. Since all learning vehicles share the SUMO environment, changing the population also changes traffic demand, rather than isolating the effect of aggregation population size.

\subsection{Correlation Mixing and Seed-Paired Analysis}
\label{app:cpqlambda_paired}

The CPQ-$\lambda$ ablation in
Table~\ref{tab:cpqlambda_exp} includes empirical
aggregation-MSE measurements and seed-paired
routing comparisons. For each round, aggregation
MSE is measured across all model coordinates
relative to the average of the clients' clipped
models, excluding clipping error. The measurements
are averaged over 100 rounds within each run
and then across ten seeds. The observed
MSE decreases monotonically with $\lambda$,
consistent with the theoretical correlation-mixture
relationship.

To examine routing variability, Fig.~\ref{fig:cpqlambda_paired}
reports seed-paired travel-time differences
relative to independent LDP-FL and fully correlated
CPQ. For each of ten matched seeds,
travel time is first averaged over rounds 91--100
before computing the paired difference. The figure
shows individual differences and their mean with
95\% Student-$t$ confidence intervals. Negative
values indicate lower travel time than the
reference mechanism.

Relative to LDP-FL, the intermediate mixing configurations
exhibit mean travel-time differences ranging from $-41.3$
to $-30.7$~s, with only the confidence interval for
$\lambda=0.75$ excluding zero.
Fully correlated CPQ yields a paired difference
of $-45.8\pm38.0$~s. Relative to fully correlated
CPQ, all intermediate configurations have
confidence intervals containing zero. These
results indicate favorable average routing
performance with partial correlation but do
not establish a consistent ranking among the
mixing configurations.

\begin{wrapfigure}{r}{0.45\textwidth}
\centering
\vspace{-20pt}
\includegraphics[width=\linewidth]
{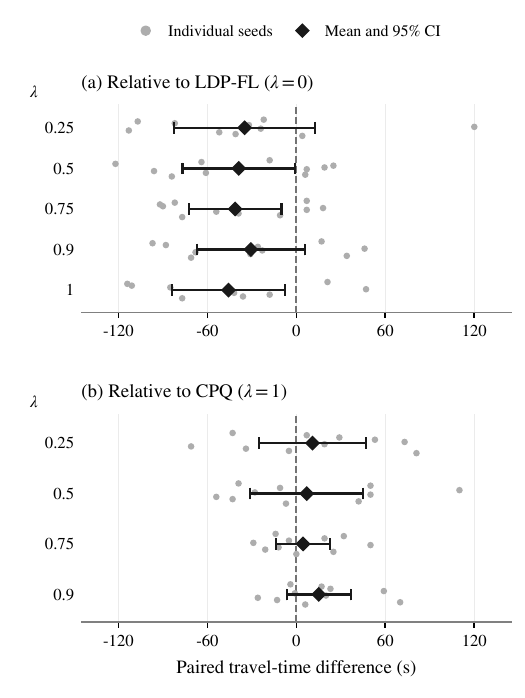}
\caption{Seed-paired travel-time differences for CPQ-$\lambda$.}
\label{fig:cpqlambda_paired}
\end{wrapfigure}

\subsection{Run-to-Run Variability}

Figure~\ref{fig:e1_seed_variability} provides a
seed-level view of the late-round routing performance
reported in Table~\ref{tab:privacy_utility}.
Each individual marker represents the mean evaluation
result over rounds 91--100 for one independent seed.
The larger markers and error bars indicate the
across-seed mean and sample standard deviation,
respectively. The no-privacy FLDQN reference is
shown separately using a dashed line and a shaded
standard-deviation band.

\begin{figure}[t]
\centering
\includegraphics[width=0.85\textwidth]
{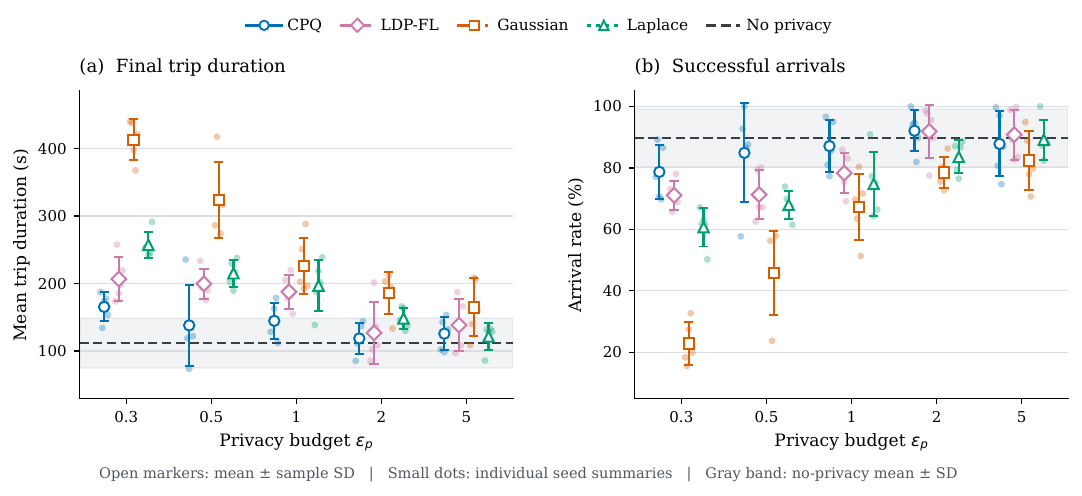}
\caption{Seed-level variability of late-round
travel time and arrival rate across privacy budgets.}
\label{fig:e1_seed_variability}
\end{figure}

\subsection{Secondary Routing Delay Measures}

To complement the travel-time and arrival-rate results,
Fig.~\ref{fig:e1_secondary_delays} reports two additional
SUMO metrics: time loss and waiting time. Time loss
represents the difference between experienced travel time
and free-flow travel time along the driven route, whereas
waiting time measures the accumulated duration during
which vehicle speed falls below $0.1$~m/s. Both metrics
are summarized over rounds 91--100 and averaged across
five independent runs.

At $\epsilon_p=0.3$, CPQ achieves a mean time loss
of $51.4$~s and waiting time of $20.2$~s, compared
with $86.7$~s and $42.2$~s for LDP-FL, respectively.
The Laplace and Gaussian mechanisms exhibit higher
values for both metrics under the same privacy budget.

As the privacy budget increases, the differences
among the mechanisms generally become smaller,
although their relative ordering is not uniform.
At $\epsilon_p=5$, Laplace achieves slightly lower
time loss and waiting time than CPQ. These results
provide additional context for the routing-efficiency
differences observed in the main evaluation.

\begin{figure}[t]
\centering
\includegraphics[width=0.85\textwidth]
{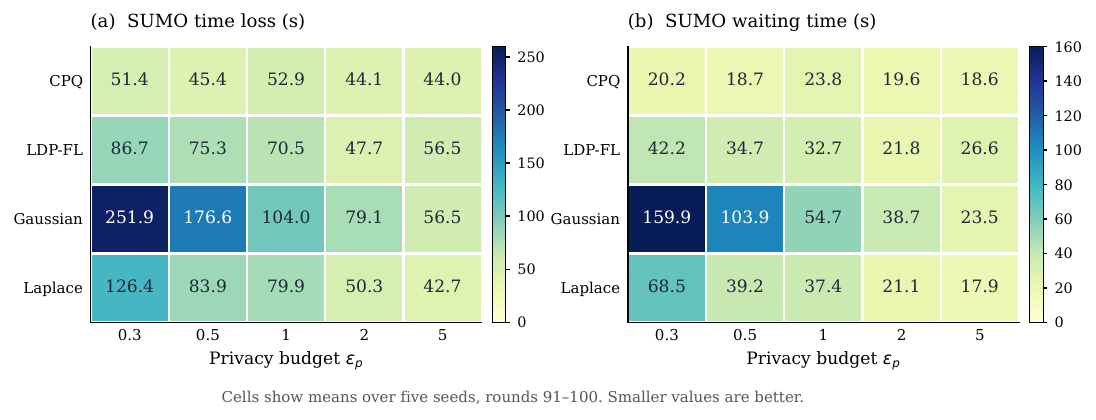}
\caption{Secondary routing-delay measures across
privacy mechanisms and privacy budgets.}
\label{fig:e1_secondary_delays}
\end{figure}

\section{Server-Hidden Pairing for CPQ and CPQ-$\lambda$}
\label{app:hidden-pairing}

This appendix details the authenticated group establishment and server-hidden pairing procedures underlying CPQ and CPQ-$\lambda$. An agreed vehicle group establishes a shared secret before federated training and subsequently derives coordinate- and round-dependent pair assignments and complementary pseudorandom values without exchanging pairing information during training. We first describe group establishment over the simulated V2I channel, followed by the local pairing construction and its security assumptions.

\noindent\textbf{Setup and trust assumptions.}
Before federated training, the vehicles establish
an agreed ordered participant list $\mathcal R$
and a group epoch identifier $e$. The first
vehicle in $\mathcal R$ serves as the designated
group coordinator $D$. Each vehicle possesses
registered post-quantum signing and ML-KEM-512
public keys, established through an authenticated
bootstrap channel.

The coordinator distributes a group secret $K_G$
to the other vehicles using authenticated,
ML-KEM-protected messages relayed by the RSU.
The RSU acts as a store-and-forward intermediary
and does not require access to the group secret
or the vehicles' KEM private keys. Each
recipient verifies the sender, intended
recipient, group epoch, and ordered participant
list before accepting the setup information.
Training begins only after the required group
confirmations have completed successfully.

The server-hidden pairing model assumes an
honest-but-curious RSU without access to $K_G$
or the information required to derive it.
Participating vehicles are assumed not to
disclose the group secret to the RSU;
compromised or colluding group members fall
outside this secrecy assumption. Group
membership remains fixed within an epoch,
and a changed participant group requires
a new group setup.

\noindent\textbf{Authenticated group establishment.}
The coordinator $D$ executes a three-stage
protocol using signed messages delivered
through the RSU:

\begin{enumerate}
\item \textbf{KEY.}
For each other participant $u\in\mathcal R$,
the coordinator encapsulates a shared secret
using the recipient's ML-KEM-512 public key.
A key derived from the encapsulated secret
through HKDF-SHA256 protects $K_G$ using
ChaCha20-Poly1305. The authenticated context
binds the recipient, epoch, and hash of the
ordered participant list.

\item \textbf{CONFIRM.}
After verifying and recovering $K_G$, each
recipient returns a signed key-confirmation
message containing an HMAC-SHA256 tag derived
from $K_{\rm check}$ and bound to its identity,
the epoch, and the participant-list hash.

\item \textbf{SETUP\_OK.}
Once all required confirmations have been
validated, the coordinator sends a signed
completion message to the group. Any failed
authentication, decryption, context validation,
or incomplete confirmation prevents successful
completion of the group setup.
\end{enumerate}

Because the vehicles may contact the RSU at
different times, setup messages are stored
and forwarded as recipients become available.
They traverse the same beacon-controlled,
RSU-scheduled, SR-ARQ V2I communication system
used for model exchange, during dedicated
setup drives preceding federated training.
No additional pairing messages are transmitted
during subsequent training rounds.

\begin{wrapfigure}{r}{0.50\textwidth}
\centering
\includegraphics[width=\linewidth]
{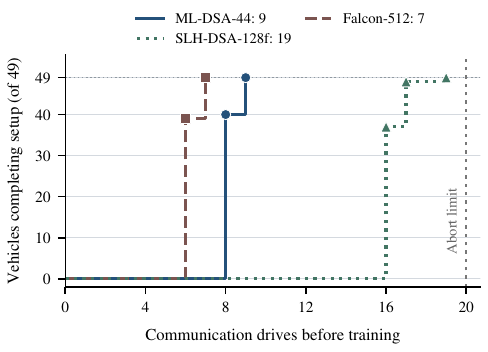}
\caption{Authenticated CPQ group-establishment progress
across post-quantum signature schemes. All 49 recipients
verified \texttt{SETUP\_OK} after 7 (Falcon-512),
9 (ML-DSA-44), or 19 (SLH-DSA-SHA2-128f)
communication drives. The vertical line marks
the 20-drive setup limit.}
\label{fig:e3_setup}
\end{wrapfigure}

\noindent\textbf{Group-establishment overhead.}
In the primary ML-DSA-44 mobility-coupled
evaluation, authenticated group establishment
completed in $3$--$9$ dedicated setup drives
across five independent CPQ runs, generating
$714.6$--$717.2$\,KB of total wire traffic
per run. The protocol required $246$--$247$
transfer attempts, including retransmissions
when necessary.

The signature-sensitivity experiment further
characterizes the communication cost of CPQ
initialization under three post-quantum
authentication schemes. All three configurations
successfully completed group establishment
for the same $49$ recipient vehicles.
Falcon-512 completed setup in $7$ drives
with $246$ successful transfers and
$280.4$\,KB of total wire traffic.
ML-DSA-44 completed setup in $9$ drives
with $246$ successful transfers and
$714.6$\,KB of traffic. In comparison,
SLH-DSA-SHA2-128f required $19$ setup
drives and $426$ transfer attempts to
complete the same $246$ successful message
deliveries, generating $7{,}462.2$\,KB
of wire traffic.

Figure~\ref{fig:e3_setup} illustrates
the progression of group-establishment
completion across the dedicated setup
drives. The shorter Falcon-512 signatures
permit more compact authenticated setup
messages, while the larger SLH-DSA
signatures increase wireless transmission
and retransmission costs. Consequently,
signature selection influences both the
communication overhead of model exchange
and the time required to establish CPQ's
shared randomness under intermittent
V2I connectivity.

Group establishment is performed once
per participant-group epoch, before
federated training. Its communication
cost is therefore distinct from the
per-round uplink and downlink traffic
reported in the main evaluation.

\begin{algorithm}[t]
\caption{Authenticated setup and server-hidden
CPQ shared-randomness generation}
\label{alg:hidden-setup}
\begin{algorithmic}[1]

\State \textbf{Group setup (once per epoch)}
\State Agree on ordered group $\mathcal R$
and epoch $e$; let $D$ be the first vehicle.
\State $h_{\mathcal R}\gets
\operatorname{SHA256}(\operatorname{Encode}(\mathcal R))$
\State $D$ obtains $K_G$ and distributes
signed, ML-KEM-protected \texttt{KEY} messages
through the RSU.
\State Recipients authenticate and recover $K_G$.
\State Derive $K_{\rm pair}$, $K_{\rm coin}$,
and $K_{\rm check}$ using HKDF-SHA256 with
purpose labels, $e$, and $h_{\mathcal R}$.
\State Recipients return signed
\texttt{CONFIRM} messages using $K_{\rm check}$.
\State $D$ verifies all confirmations and
distributes signed \texttt{SETUP\_OK};
abort on incomplete or invalid setup.

\State
\State \textbf{Local procedure}
\State \textsc{HiddenPairValue}
$(K_G,\mathcal R,e,i,j,t)$:
\State Use the epoch keys $K_{\rm pair}$
and $K_{\rm coin}$ derived during setup.
\For{each $u\in\mathcal R$}
  \State $s_u\gets
  \operatorname{ChaCha20Stream}
  (K_{\rm pair};t,u)[j]$
\EndFor
\State Sort $\mathcal R$ by increasing $s_u$,
breaking ties by original list position.
\State Let $k$ be the zero-based position of
vehicle $i$ in the sorted list.
\If{$|\mathcal R|$ is odd and
$k=|\mathcal R|-1$}
  \State \textbf{return} a private uniform value.
\EndIf
\State $\ell\gets\lfloor k/2\rfloor$
\State $V\gets
\operatorname{Float53}\!\left(
\operatorname{ChaCha20Stream}
(K_{\rm coin};t,\ell)[j]\right)$
\If{$k\bmod 2=0$}
  \State \textbf{return} $V$
\Else
  \State \textbf{return} $1-V$
\EndIf

\end{algorithmic}
\end{algorithm}

\noindent\textbf{Key and randomness derivation.}
The group secret $K_G$ is expanded using
HKDF-SHA256 into three purpose-separated,
32-byte keys: $K_{\rm pair}$ for pair
assignment, $K_{\rm coin}$ for shared-value
generation, and $K_{\rm check}$ for group
confirmation. The derivation incorporates
the group epoch $e$ and the SHA-256 hash
of the consistently encoded ordered
participant list $\mathcal R$.

For each round $t$, ChaCha20 streams indexed
by vehicle identity generate the sorting
values for all model coordinates, while
separate streams indexed by pair position
generate the corresponding shared values.
At each coordinate, the vehicles sort the
entire agreed group by the derived keys,
using original list positions to break ties.
Adjacent sorted vehicles form pairs and
use complementary values $V$ and $1-V$.
No pairing assignments or shared random
values are transmitted during training.

\noindent\textbf{Quantization properties and partial participation.}
Under the ideal model of independent,
continuous uniform shared values across
distinct pairs, the construction realizes
the complementary coupling analyzed in
Theorem~\ref{thm:cpq}. The population MSE
bound in Corollary~\ref{cor:cpq_pop}
additionally assumes that the aggregate
contains complete pairs.

In the contact-dependent implementation,
pair assignments are determined over
$\mathcal R$, regardless of which vehicles
successfully deliver their updates. If
only one member of a pair is included,
the pairwise error cancellation is not
realized. Consequently, the complete-pair
population bound does not directly apply
to arbitrary accepted-client sets.
The finite-precision implementation
requires separate consideration, as
discussed in Appendix~\ref{app:finite_precision}.


\section{Vehicle-Origin Inference Attack: Design and Evaluation}
\label{app:origin_attack}

\noindent\textbf{Experimental configuration.}
To complement the empirical privacy evaluation in
Section~\ref{subsec:exp_origin_attack}, we investigate
vehicle-origin inference using the mobility-coupled
SUMO and ns-3 framework. The road network contains
30 directed edges, with 50 learning vehicles distributed
equally across 10 possible starting edges:
E0, E1, E3, E11, E8, -E1, -E11, -E8, -E13, and E14.
Each vehicle retains its assigned origin throughout
training, while all vehicles share the publicly known
destination E15.

The DQN architecture consists of 94 state inputs,
two hidden layers with 150 and 100 neurons, and
three output actions, resulting in 29,653 trainable
parameters. Training proceeds for 100 federated rounds
with 10 local episodes per round. The learning rate
is $10^{-3}$, the mini-batch size is 32, and the
exploration rate decreases from 1.0 to 0.05 over
80 rounds. Three independent random seeds are
evaluated for each method.

The communication environment employs IEEE~802.11p
at 6 Mb/s over a 10 MHz channel, with 20 dBm
transmission power, log-distance path loss
(exponent 2.7), and Nakagami fading. The RSU is
located at junction J3. Vehicles discover the RSU
through beacon transmissions and upload their
updates only during available communication
opportunities. The origins are selected such that
their shortest routes pass the RSU at comparable
distances of 120--143 m, reducing differences
in nominal RSU proximity across origin classes.
Background traffic is generated at an average
interval of 1.75 s following a 200 s warm-up,
while learning-vehicle departures are randomized
within a 100 s window in each episode.

Model uploads are authenticated using ML-DSA-44,
and the CPQ group secret is established using
ML-KEM. The evaluation compares full-precision
model exchange without privacy against CPQ and
independent LDP-FL quantization at
$\epsilon_p=0.3$. Both private mechanisms clip
local model updates relative to the global model
using $r=0.02$ and encode each parameter with
one bit. The server retains only successfully
received and authenticated uploads.

\noindent\textbf{Adversary and observation model.}
We consider an honest-but-curious RSU that
follows the communication and aggregation
protocol while attempting to infer each vehicle's
private starting edge from received model updates.
The attacker knows the road network, learning
architecture, privacy mechanism, privacy budget,
clipping radius, and global model at each round.
It also observes individual successfully received
uploads and can associate observations belonging
to the same pseudonymous vehicle across rounds.

The attacker does not observe private vehicle
trajectories or possess the CPQ group secret.
Vehicle identifiers carry no origin information,
and vehicle order is randomized before evaluating
each attack. The attack objective is to identify
one of 10 possible origins for each target vehicle,
corresponding to a random-guessing accuracy of
$10\%$. The attacker uses only the updates
actually available at the RSU and does not
assume complete federated participation.

\noindent\textbf{Attack feature construction.}
The attacker extracts information from received
updates by evaluating their influence on the
learned Q-function. Specifically, 60 fixed probe
states are constructed from the public road
network, covering each of the 30 directed edges
under two traffic conditions representing empty
and busy traffic. These states do not depend on
the private trajectories of target vehicles.

For each received update, the attacker evaluates
the difference between the uploaded model's
Q-values and those of the corresponding global
model at every probe state:

\begin{equation}
\Delta Q_i^{(t)}(s,a)
=
Q_i^{(t)}(s,a)-Q_g^{(t)}(s,a),
\label{eq:origin_attack_features}
\end{equation}

where $Q_i^{(t)}$ and $Q_g^{(t)}$ denote the
Q-functions associated with the received client
model and the global model at round $t$,
respectively. With 60 probe states and three
actions, each update produces a 180-dimensional
feature vector.

This representation captures the changes in
routing decisions induced by local training
while providing comparable features across
training runs. Alternative representations,
including selected first-layer weights and
principal-component features of complete
updates, are also evaluated.

\noindent\textbf{Temporal evidence accumulation.}
Because each vehicle retains the same origin
throughout training, the attacker combines
evidence across multiple federated rounds.
We evaluate two inference strategies.

In the first strategy, an L2-regularized logistic
regression classifier estimates the probability
of each origin from the features of an individual
received update. Log-probabilities are accumulated
over the observed rounds, and the origin with the
largest cumulative score is selected:

\begin{equation}
\hat{o}_i^{(k)}
=
\arg\max_{o\in\mathcal{O}}
\sum_{t\in\mathcal{T}_i(k)}
\log P(o\mid \mathbf{x}_i^{(t)}),
\label{eq:origin_attack_prediction}
\end{equation}

where $\mathcal{O}$ denotes the set of candidate
origins, $\mathbf{x}_i^{(t)}$ is the feature
vector extracted from vehicle $i$ at round $t$,
and $\mathcal{T}_i(k)$ contains the rounds
up to $k$ in which the RSU successfully
received an update from that vehicle.

In the second strategy, the received updates
of each vehicle are averaged across the
observation window, and a classifier trained
on averaged updates predicts its origin.
Both strategies account for missing uploads
by using only successfully received observations.
Vehicles without any received updates in the
observation window are assigned random guesses.

Attack accuracy is evaluated after
$k\in\{1,2,5,10,20,50,100\}$ observed rounds.

\noindent\textbf{Attack 1: Known-origin vehicles.}
In the first scenario, the RSU knows the
starting edge of one vehicle from each
origin class, corresponding to 10 known
vehicles among the 50 participants.
These vehicles are ordinary honest clients,
and their cryptographic secrets are not
available to the attacker.

The attacker trains on the labeled updates
from these vehicles and evaluates origin
inference on the remaining 40 vehicles.
It uses the Q-function probe features and
the per-round probability accumulation
strategy in~\eqref{eq:origin_attack_prediction}.
The classifier employs L2-regularized
logistic regression with $C=0.1$,
fixed using separate pilot experiments.
Five random selections of the known-origin
vehicles are evaluated for each simulation run.

\noindent\textbf{Attack 2: Shadow-run inference.}
In the second scenario, the RSU has no
labeled vehicles in the target deployment.
Instead, it generates independent simulations
of the same road network and learning
configuration, using their known origin
labels to train an inference classifier.

We employ leave-one-seed-out evaluation.
For each method, two simulation seeds are
used for attacker training, while the third
seed is held out for evaluation. This process
is repeated with each seed serving as the
target once. Feature representation,
classifier regularization, and temporal
aggregation strategy are selected using
only training-seed data. The target seed
is never used for classifier training or
model selection.

The attacker then predicts the origins
of all 50 vehicles in the held-out
deployment, providing an evaluation
without labeled vehicles in the target run.

\noindent\textbf{Evaluation and validation.}
Attack effectiveness is measured as the
percentage of target vehicles whose
origins are correctly identified after
observing $k$ federated rounds. The
evaluation includes three independent
simulation seeds for each privacy method.

In addition to reporting seed-level
variability, 95\% confidence intervals
are estimated through a cluster bootstrap
that resamples simulation seeds and vehicles.
Shuffled-label controls are also evaluated,
using five random permutations of origin
labels to verify that the attack does not
exploit artificial label associations.

To verify consistency between the attack
observations and federated training, the
server records its global model and all
successfully received individual uploads.
Reconstructing and aggregating these
logged updates reproduces the server's
next global model with a maximum difference
below $10^{-6}$ across all nine runs.
The shuffled-label controls achieve
approximately $8$--$11\%$ accuracy,
consistent with random guessing.

The mean percentages of vehicle-round
updates successfully received by the
RSU are $47\%$ for the no-privacy
baseline, $61\%$ for CPQ, and $62\%$
for LDP-FL. Thus, the attack results
reflect the actual observations available
under mobility-dependent communication.

\noindent\textbf{Origin inference results.}
Table~\ref{tab:origin_attack_results}
summarizes the inference accuracy of
the two attacks at three representative
observation windows. Values are reported
as mean $\pm$ sample standard deviation
across three seeds, consistent with the
reporting convention used in the main paper.
The statistics are computed from the
reported rounded seed-level accuracies.

\begin{wraptable}{r}{0.54\textwidth}
\centering
\caption{Vehicle-origin inference accuracy (mean $\pm$ SD).}
\label{tab:origin_attack_results}
\setlength{\tabcolsep}{3pt}
\renewcommand{\arraystretch}{1.10}
\begin{tabular*}{\linewidth}{@{\extracolsep{\fill}}lccc@{}}
\toprule
\textbf{Rounds}
& \textbf{No privacy}
& \textbf{CPQ}
& \textbf{LDP-FL} \\
\midrule
\multicolumn{4}{c}{\textbf{A. Known-origin attack}} \\
\cmidrule(lr){1-4}
10  & $33.7 \pm 4.5$ & $11.3 \pm 0.6$ & $9.3 \pm 2.5$ \\
50  & $63.7 \pm 4.5$ & $11.7 \pm 1.5$ & $10.0 \pm 1.0$ \\
100 & $68.7 \pm 6.8$ & $13.0 \pm 2.6$ & $12.0 \pm 3.5$ \\
\midrule
\multicolumn{4}{c}{\textbf{B. Shadow-run attack}} \\
\cmidrule(lr){1-4}
10  & $46.0 \pm 7.2$ & $13.3 \pm 4.2$ & $11.3 \pm 2.3$ \\
50  & $76.7 \pm 6.1$ & $14.7 \pm 4.2$ & $10.7 \pm 5.0$ \\
100 & $91.3 \pm 1.2$ & $30.0 \pm 10.6$ & $22.0 \pm 4.0$ \\
\bottomrule
\end{tabular*}
\end{wraptable}

The known-origin attack reaches
$68.7\%$ mean accuracy after 100
rounds without privacy, compared
with $13.0\%$ for CPQ and $12.0\%$
for LDP-FL. Both private mechanisms
substantially reduce inference
accuracy across the evaluated
observation windows.

Under the shadow-run attack,
which requires no known-origin
vehicles in the target deployment,
the no-privacy baseline reaches
$91.3\%$ mean accuracy after
100 rounds. In comparison,
CPQ and LDP-FL achieve
$30.0\%$ and $22.0\%$,
respectively. The stronger
late-round inference results for
the private mechanisms arise
from the strategy that averages
received updates across rounds.

These findings show that the
unprotected federated updates
contain information that can be
exploited to infer private
vehicle origins. Both CPQ and
LDP-FL substantially reduce
the effectiveness of the evaluated
attacks, providing application-level
empirical evidence that complements
the theoretical privacy analysis.
Together with the routing and
communication evaluations, the
results further characterize the
privacy--utility--communication
trade-offs of private quantization
in vehicular federated learning.

\textbf{Ethical Considerations:} The origin-inference attacks were conducted entirely within a simulated vehicular environment using synthetically generated mobility and federated learning data. No real driver information, personal mobility records, or operational vehicular infrastructure were involved. The attacks were implemented solely to evaluate potential privacy vulnerabilities and assess the effectiveness of the proposed privacy-preserving mechanism.

\end{document}